\documentclass[twoside,11pt,dvipsnames]{article}
\usepackage[latin1]{inputenc}
\usepackage{array}
\usepackage{longtable}
\usepackage{varioref}
\usepackage{prettyref}
\usepackage{booktabs}
\usepackage{amsmath}
\usepackage{graphicx}

\makeatletter

\providecommand{\tabularnewline}{\\}

\usepackage{amsmath}
\usepackage{amsfonts}
\usepackage{amsthm}
\usepackage{jair}
\usepackage{rawfonts}

\PassOptionsToPackage{dvipsnames}{xcolor}

\providecommand{\argsA}[2]{ {#1}_{#2} } % 'agent' argument
\providecommand{\argsG}[2]{ {\joint{#1}}_{#2} } % 'group' argument - NOTE this handles making it 'joint' (since we might not always want that...
\providecommand{\argsT}[2]{ {#1}^{#2} } % 'time' argument
\providecommand{\argsI}[2]{ {#1}^{#2} } % 'index' argument
\providecommand{\argsAT}[3]{ {#1}_{#2}^{#3} }
\providecommand{\argsGT}[3]{ {\joint{#1}}_{#2}^{#3} } 
 
\providecommand{\argsIT}[3]{ {#1}^{#2,#3} }
\providecommand{\argsAIT}[4]{ {#1}_{#2}^{#3,#4} }

\providecommand{\V}{V}
\providecommand{\ed}{e}
\providecommand{\edS}{\set{E}}

\providecommand{\RA}[1]     {\argsA{R}{#1}}
\providecommand{\QT}[1]     {\argsT{Q}{#1}}
\providecommand{\QAT}[2]    {\argsAT{Q}{#1}{#2}}
\providecommand{\VA}[1]     {\argsA{V}{#1}}
\providecommand{\VAT}[2]    {\argsAT{V}{#1}{#2}}

\providecommand{\KroD}[2]   { \IndF_{\{ #1, #2 \}} }

\providecommand{\IndF}      { \boldsymbol{1} }
\providecommand{\lbA}[1]    {\argsAT{b}{#1}{l}}
\providecommand{\lbAT}[2]   {\argsAT{b}{#1}{l,{#2}}}
\providecommand{\gbA}[1]    {\argsAT{b}{#1}{g}}
\providecommand{\gbAT}[2]   {\argsAT{b}{#1}{g,{#2}}}

\providecommand{\NMF}        {y}                %non-modeled factor
\providecommand{\MF}         {x}                %modeled factor
\providecommand{\MFn}        {\tilde{x}}               %non-locally-affected MF (NLAF)
\providecommand{\MFl}        {\mathring{x}}               %only-locally-affected MF (LAF)
\providecommand{\PF}         {xPRIVATE}         %private factor = only locally affected (LAF)
\providecommand{\mfI}[1]    {\argsI {\MF}{#1}}
\providecommand{\mfIT}[2]   {\argsIT{\MF}{#1}{#2}}

\providecommand{\mfA}[1]    {\argsA {\joint{\MF}}{#1}}       %<--- AKA: a local state!
\providecommand{\mfAT}[2]   {\argsAT{\joint{\MF}}{#1}{#2}}

\providecommand{\LSF}{S}        %the local state function.
\providecommand{\slfmAS}[1]    {\argsA {\joint{\set{X}}} {#1}}    
\providecommand{\mflI}[1]   {\argsI {\MFl}{#1}}

\providecommand{\mflIT}[2]  {\argsIT{\MFl}{#1}{#2}}
\providecommand{\mflAT}[2]  {\argsAT{\joint{\MFl}}{#1}{#2}}

\providecommand{\mfnI}[1]   {\argsI {\MFn}{#1}}
\providecommand{\mfnT}[1]   {\argsT {\MFn}{#1}}
\providecommand{\mfnIT}[2]  {\argsIT{\MFn}{#1}{#2}}
\providecommand{\mfnAT}[2]   {\argsAT{\joint{\MFn}}{#1}{#2}}
\providecommand{\mfnAIT}[3]  {\argsAIT{\joint{\MFn}}{#1}{#2}{#3}}
\providecommand{\pfI}[1]    {\argsI {\PF}{#1}}

\providecommand{\nmfA}[1]   {\argsA{\joint{\NMF}}{#1}}
\providecommand{\nmfAT}[2]  {\argsAT{\joint{\NMF}}{#1}{#2}}
\providecommand{\nmfI}[1]   {\argsI{\NMF}{#1}}
\providecommand{\nmfT}[1]   {\argsT{\NMF}{#1}}
\providecommand{\nmfIT}[2]  {\argsIT{\NMF}{#1}{#2}}

\providecommand{\INFL}        {I} %the influence base symbol
 
\providecommand{\ifSD}[2]   {\INFL_{#1\rightarrow#2}}

\providecommand{\iffunc}    {\INFL}                               %influence function (used for I(..|..) probability style notation

\providecommand{\ifpiA}[1]  {\argsA{\INFL}{\rightarrow#1}}        
\providecommand{\ifpiAT}[2] {\argsAT{\INFL}{\rightarrow#1}{#2}}  
   
\providecommand{\ifpoAT}[2] {\argsAT{\INFL}{#1\rightarrow}{#2}} 
 
\providecommand{\jifp}      {{\INFL}}                 %a joint influence point.
\providecommand{\IFSOURCE}  {u}
\providecommand{\ifsT}[1]   {\argsT     {\IFSOURCE}{#1}}
\providecommand{\ifsIT}[2]  {\argsIT    {\IFSOURCE}{#1}{#2}}
\providecommand{\ifsAT}[2]  {\argsAT    {\joint{\IFSOURCE}}{\rightarrow #1}{#2}} 

\providecommand{\IFDEST}    {v}
\providecommand{\ifdT}[1]   {\argsT     {\IFDEST}{#1}}

\providecommand{\DSET}          {\joint{D}}
\providecommand{\dsetA}[1]      {\argsA{\DSET}{#1}}
\providecommand{\dsetAT}[2]     {\argsAT{\DSET}{#1}{#2}}
\providecommand{\dHistAT}[2]    {\vec{\DSET}\argsAT{{}}{#1}{#2}}

\providecommand{\dsetUF}{d}
\providecommand{\dsetCompF}{\sigma}

\providecommand{\sect}{Section~}

\newcommand{\fig}{Figure~}

\newcommand{\tab}{Table~}

\newcommand{\app}{Appendix~}
\newcommand{\lem}{Lemma~}
\newcommand{\thm}{Theorem~}

\providecommand{\dfn}{Definition~}

\providecommand{\defas}     {\operatorname{\triangleq}}
\providecommand{\PrS}{\Delta} %set of probabilities over .
\newcommand{\E}{\mathbf{E}}

\providecommand{\reals}{\mathbb {R} }
\providecommand{\real}{\reals} %<deprecated shorthand

\providecommand{\algName}[1]{{\sc{#1}}}
\providecommand{\problemName}[1]{\textsc{#1}}

\newcommand{\set}[1]{\mathcal{#1}}
\providecommand{\joint}[1]{\boldsymbol{#1}}

\providecommand{\discount}{\gamma}

\providecommand{\ts}{t}             %time steps, ranging from 0,...,h-1
\providecommand{\hor}{H}
\providecommand{\h}{\hor}              %the horizon

\providecommand{\agentSymb}{D}
\providecommand{\agentS}{\mathcal{\agentSymb}}
\providecommand{\agentI}[1]{{#1}}
\providecommand{\nrA}{n} % the number of agents

\providecommand{\excl}[1]{-{#1}}

\providecommand{\Aug}[1]        {\bar{#1}}
\providecommand{\sAugA}[1]      {\argsA{\Aug{s}}{#1}}
\providecommand{\sAugAT}[2]     {\argsAT{\Aug{s}}{#1}{#2}}

\providecommand{\sS}{\set{S}}   % the state set (state space)
\providecommand{\s}{s}  
\providecommand{\sT}[1]{\argsT{\s}{#1}} 
 
\providecommand{\sA}[1]{\argsA{\s}{#1}}    %<-- note the link with \mfA (that is a local state in an LFM, this more general)
\providecommand{\sAS}[1]{\argsA{\sS}{#1}}  %<-- 
\providecommand{\sAT}[2]{\argsAT{\s}{#1}{#2}} 

\providecommand{\nrSF}{|\sfacS|} % the number of state-factors

\providecommand{\factorSymb}{F}
\providecommand{\factorSetSymb}{\set{F}}    %the set of state factors

\providecommand{\sfacS}     {\factorSetSymb}                    % the set of all state factors
\providecommand{\sfac}      {\factorSymb}                       % a state factor 
\providecommand{\sfacI}[1]  {\argsI{\factorSymb}{#1}}           % a state factor with index #1
\providecommand{\sfacT}[1]  {\argsT{\factorSymb}{#1}}           %                       
\providecommand{\factorValueSymb}{f}
\providecommand{\factorValueSetSymb}{\set{F}}

\providecommand{\sfacvI}[1]     {\argsI{\factorValueSymb}{#1}}     

\providecommand{\sfacvIS}[1]    {\argsI{\factorValueSetSymb}{#1}}  
     
\providecommand{\AC}{a}                 % an action (of the single agent)
\providecommand{\ACS}{\set{A}}                     % the set of actions
\providecommand{\aA}[1] {\argsA {\AC}   {#1}}
\providecommand{\aAS}[1]{\argsA {\ACS}  {#1}}      
\providecommand{\aAT}[2]{\argsAT{\AC}   {#1}{#2}}

\providecommand{\ja}    {\joint {\AC}}          % a joint action
\providecommand{\jaS}   {\joint {\ACS}}         % the set of joint actions
\providecommand{\jaT}[1]{\argsT {\ja}   {#1}}   % a joint action with index #1
\providecommand{\jaG}[1]    {\argsG{\AC}    {#1}} % a (joint) action of group #1 
\providecommand{\jaGT}[2]   {\argsGT{\AC}   {#1}{#2}} % a (joint) action of group #1 
\providecommand{\OB}{o}                
\providecommand{\OBS}{\set{O}}
\providecommand{\oA}[1] {\argsA {\OB}   {#1}}
\providecommand{\oAS}[1]{\argsA {\OBS}  {#1}}      
\providecommand{\oAT}[2]{\argsAT{\OB}   {#1}{#2}}
\providecommand{\jo}        {\joint {\OB}}      
\providecommand{\joS}       {\joint {\OBS}}    
\providecommand{\joT}[1]    {\argsT {\jo}   {#1}}
      
\providecommand{\joGT}[2]   {\argsGT{\OB}   {#1}{#2}}

\providecommand{\Tfunc}     {T}
\providecommand{\Ofunc}     {O}

\providecommand{\REW}{r}    %base notation for a (sampled, or received) reward
\providecommand{\REWF}{R}   %base notation for the reward function
\providecommand{\rA}[1] {\argsA {\REW}   {#1}}

\providecommand{\R}     {\REWF}                     % the global reward function
\providecommand{\RA}[1] {\argsA {\REWF}   {#1}}     % an agent's local reward function
\providecommand{\RAT}[2]{\argsAT{\REWF}   {#1}{#2}}

\providecommand{\bSymbol}           {{b}}
\providecommand{\bO}                {\argsT\bSymbol{0}} %- the initial belief
\providecommand{\bA}[1]             {\argsA\bSymbol{#1}}

\providecommand{\gbA}[1]             {\argsAT\bSymbol{#1}{g}}
\providecommand{\AOH}{\vec{h}}
\providecommand{\AOHS}{\vec{\mathcal{H}}}
\providecommand{\JAOH}{\vec{\joint{h}}}

\providecommand{\aoHist}        {\JAOH}

\providecommand{\aoHistA}[1]    {\argsA   {\AOH}{#1}}

\providecommand{\aoHistAT}[2]   {\argsAT{\AOH}{#1}{#2}}
\providecommand{\aoHistATS}[2]  {\argsAT  {\AOHS}{#1}{#2}}

\providecommand{\aoHistGT}[2]   {\JAOH\argsGT{}{#1}{#2}}

\providecommand{\OH}{\vec{o}}

\providecommand{\JOH}{\vec{\joint{o}}{}}

\providecommand{\oHistT}[1]    {\JOH\argsT{}{\,#1}}

\providecommand{\oHistAT}[2]   {\argsAT  {\OH}{#1}{\,#2}}

\providecommand{\oHistGT}[2]   {\argsGT\JOH{#1}{\,#2}}

\providecommand{\aoHistEmpty}{()}

\providecommand{\POL}{\pi}

\providecommand{\jpol}      {\joint{\POL}}	                % a joint policy - general term

\providecommand{\jpolG}[1]  {\argsG     {\joint{\POL}}{#1}}	% a joint policy for group #1
\providecommand{\jpolGT}[2] {\argsGT    {\joint{\POL}}{#1}{#2}}	% a joint policy for group #1

\providecommand{\polA}[1]   {\argsA{\POL}{#1}}	% a policy - general term
\providecommand{\DR}{\delta}

\providecommand{\drAT}[2]{\argsAT{\DR}{#1}{#2}}

\renewcommand{\joint}[1]    {#1}
\jairheading{70}{2021}{789-870}{05/2020}{02/2021}
\ShortHeadings{A Sufficient Statistic for Influence in Structured Multiagent Environments}
{Oliehoek, Witwicki, \& Kaelbling}
\firstpageno{789}

\usepackage{graphicx}
\usepackage{amsmath}
\usepackage{amssymb}
\usepackage{amsthm}
\usepackage{colortbl}
\usepackage{multirow}
\usepackage{hhline}
\usepackage{algorithmic}

\usepackage{psfrag}
\usepackage{longtable}
\usepackage{verbatim}
\usepackage{icomma} %intelligent comma in math
\usepackage{ifthen}
\usepackage{fancyhdr}
\usepackage{datetime}

\usepackage{tikz}

\usepackage{enumitem}
\setlist{nolistsep}
\setlist{topsep=2mm,partopsep=0mm,parsep=.1ex,itemsep=.5ex,leftmargin=7mm}

\numberwithin{equation}{section}

\theoremstyle{plain}
\newtheorem{theorem}{Theorem}%[section]

\newtheorem{lemma}{Lemma}%[section]
\newtheorem{proposition}{Proposition}%[section]
\theoremstyle{definition}
\newtheorem{definition}{Definition}%[section]
\theoremstyle{remark}

\newtheorem*{example}{Example}
\newtheorem*{observation}{Observation}

\def\<#1>{%
    \expandafter\ifx\csname<#1>\endcsname\relax
        \errmessage{abbreviation <#1> undefined!}
    \else
        \csname<#1>\endcsname
    \fi
}
\def\abbr#1#2{%
    \expandafter\def\csname<#1>\endcsname{#2}%
}

\abbr{iid}{i.i.d.}

\abbr{housesearch}{\problemName{House Search}}
\abbr{dectig}{\problemName{Dec-Tiger}}
\abbr{sdectig}{\problemName{Skewed Dec-Tiger}}
\abbr{bcc}{\problemName{BroadcastChannel}}
\abbr{grid}{\problemName{GridSmall}}
\abbr{boxPush}{\problemName{Cooperative Box Pushing}}
\abbr{boxPushShort}{\problemName{C.\ Box Pushing}}
\abbr{recycling}{\problemName{Recycling Robots}}
\abbr{recyclingShort}{\problemName{Recycling R.}}
\abbr{hotel}{\problemName{Hotel 1}}
\abbr{ff}{\problemName{FireFighting}}
\abbr{FFG}{\problemName{FFG}}
\abbr{FFGfull}{\problemName{\problemName{Fi\-re\-Fight\-ing\-Graph}}}
\abbr{Aloha}{\problemName{Aloha}}

\abbr{AOH}{action-observation history}
\abbr{JAOH}{joint action-observation history}

\abbr{Astar}            {\algName{A*}}
\abbr{MAA}              {\algName{MAA*}}
\abbr{GMAA}             {\algName{GMAA*}}        
\abbr{GMAAC}            {\algName{GMAA-Cluster}}
\abbr{GMAAIC}           {\algName{GMAA*-IC}}    
\abbr{GMAAICE}          {\algName{GMAA*-ICE}}   
\abbr{FSPC}             {\algName{FSPC}} 

\abbr{QMDP}{\algName{QMDP}}
\abbr{QPOMDP}{\algName{QPOMDP}}
\abbr{QBG}{\algName{QBG}}

\makeatletter
\@ifpackageloaded{natbib}{}{\usepackage{theapa}}

\@ifundefined{showcaptionsetup}{}{%
 \PassOptionsToPackage{caption=false}{subfig}}
\usepackage{subfig}
\makeatother

\begin{document}
\title{A Sufficient Statistic for Influence \\
in Structured Multiagent Environments}
\author{\name Frans A. Oliehoek 
\email f.a.oliehoek@tudelft.nl \\
\addr Department of Intelligent Systems\\
Delft University of Technology\\
Mourik Broekmanweg 6\\ 
Delft, 2628 XE, The Netherlands  
\AND
\name Stefan Witwicki
\email stefan.witwicki@nissan-usa.com\\
\addr Alliance Innovation Lab Silicon Valley\\ 
Nissan Technical Center North America \\
3400 Central Expressway\\ 
Santa Clara, CA 95051, USA  
\AND
\name Leslie P. Kaelbling 
\email lpk@csail.mit.edu \\        
\addr CSAIL\\ 
Massachusetts Institute of Technology\\        
32 Vassar Street \\
Cambridge, MA 02139, USA
}
\maketitle
\begin{abstract}
Making decisions in complex environments is a key challenge in artificial
intelligence~(AI). Situations involving multiple decision makers
are particularly complex, leading to computational intractability
of principled solution methods. A body of work in AI has tried to
mitigate this problem by trying to distill interaction to its essence:
how does the policy of one agent \emph{influence} another agent? If
we can find more compact representations of such influence, this can
help us deal with the complexity, for instance by searching the space
of influences rather than the space of policies. However, so far these
notions of influence have been restricted in their applicability to
special cases of interaction. In this paper we formalize \emph{influence-based
abstraction~(IBA)}, which facilitates the elimination of latent state
factors \emph{without any loss in value}, for a very general class
of problems described as factored partially observable stochastic
games (fPOSGs). On the one hand, this generalizes existing descriptions
of influence, and thus can serve as the foundation for improvements
in scalability and other insights in decision making in complex multiagent
settings. On the other hand, since the presence of other agents can
be seen as a generalization of single agent settings, our formulation
of IBA also provides a sufficient statistic for decision making under
abstraction for a single agent. We also give a detailed discussion
of the relations to such previous works, identifying new insights
and interpretations of these approaches. In these ways, this paper
deepens our understanding of abstraction in a wide range of sequential
decision making settings, providing the basis for new approaches and
algorithms for a large class of problems.
\end{abstract}
\input{./preamble.natbib-emul.tex}

\section{Introduction}

\label{sec:Introduction}

One of the important ideas in the development of algorithms for multiagent
systems (MASs) is the identification of compressed representations
of the information that is relevant for an agent~\citep{Becker03AAMAS,Becker04AAMAS,Varakantham09ICAPS,Petrik09JAIR,Witwicki10AAMAS,Witwicki10ICAPS,Witwicki11AAMAS,Velagapudi11AAMAS,Witwicki11PhD,Witwicki12AAMAS,Oliehoek12AAAI_IBA,HernandezLeal17arxiv,Bazinin18GCAI}.
For instance, when a cook and a waiter collaborate, the waiter might
not need to know all details of how the cook prepares the food; it
may be sufficient if he/she has an understanding of the time that
it will take.

In this paper we investigate abstractions that aim at decomposing
structured MASs into a set of smaller interacting problems~\citep{Oliehoek12AAAI_IBA,Witwicki10ICAPS}.
In particular, we describe in detail the concept of \emph{influence-based
abstraction (IBA)}, which facilitates the abstraction of latent state
variables without sacrificing task performance. It constructs a smaller,
local model for one of the agents given the policies of the other
agents. IBA consists of two steps: first, we compute a so-called \emph{influence
point\textemdash }a more abstract representation of how an agent's
local problem is affected by other agents and external (i.e., non-local)
parts of the problem\textemdash , second, this influence is used to
construct the smaller \emph{influence-augmented local model (IALM).
}This IALM can subsequently be used to compute a best response.

IBA does not only give a new perspective on best-response computations
themselves, but this new perspective also has broader implications.
For instance, it forms the basis of \emph{influence search }\citep{Becker03AAMAS,Witwicki10ICAPS,Witwicki12AAMAS,Bazinin18GCAI},
which can provide significant speedup for multiagent planning by searching
the space of \emph{joint influences} rather than the potentially much
bigger space of joint policies. It also can underpin guarantees on
the quality of heuristic solutions, by considering \emph{optimistic
}influences \citep{Oliehoek15IJCAI}, or approximate influences~\citep{Congeduti20arxiv}\textbf{.
}While in this article, we assume that the model (which can be seen
as a specific type of dynamic Bayesian network) is known in advance,
future work could consider learning such representations. Moreover,
IBA can serve as inspiration, in the context of deep reinforcement
learning, for neural network architectures that compute approximate
versions of influence, which can improve learning, both in terms of
speed as well as performance \citep{Suau19ALA}.

This article gives a formal definition of influence that can be used
to perform IBA for general factored partially observable stochastic
games (fPOSGs)~\citep{Hansen04AAAI,Boutilier99JAIR}, and proves
that an IALM constructed using this definition of influence in fact
allows computation of an \emph{exact} best-response. In other words,
it shows that this description of influence is a \emph{sufficient
statistic} of the policy of the other agents: it is sufficient to
predict observations and rewards and to thereby optimize value. This
article extends our previous paper \citep{Oliehoek12AAAI_IBA} in
the following ways:
\begin{enumerate}
\item it provides a complete proof of the claimed exactness of IBA;
\item it elaborates on a number of technical subtleties, such as dealing
with multiple sources of influence, and specifying initial beliefs
in the IALM;
\item it provides an extension of IBA and corresponding proofs to fPOSGs
with intra-stage dependencies, which are critical for the expressiveness
of the formalism (cf. \sect\ref{sec:IBA:def-of-influence:links-sources-destinations});
\item it provides additional illustration and explanation, making the concept
of IBA more accessible;
\item it deepens the discussion of the relation to special cases of fPOSGs,
and more explicitly identifies ways in which future work can improve
scalability of these sub-classes;
\item it provides a much more extensive discussion of related work, including
the more recent work on deep reinforcement learning (RL). Specifically,
by building on the theoretical results provided in this paper, it
generates insights into the nature of the `approximate value factorization'
assumption which has been successfully exploited by a popular class
of deep RL methods.
\end{enumerate}
Additionally, in \sect\ref{sec:LFMs-and-best-responses} we make
a simple (but, in the context of IBA, novel) observation: the presence
of other agents can be seen as a generalization of single agent settings,
which directly implies that \emph{our formulation of IBA also provides
a sufficient statistic for decision making under abstraction for a
single agent}. While there is a multitude of performance loss bounds
available for abstractions, e.g., see \citet{Dearden97AIJ,Dean97UAI,Givan03AIJ,Iyengar05MOR,Li06ISAIM,Petrik14NIPS,Abel16ICML},
these are usually based on\emph{ }assumed\emph{ }quality bounds on
the transition probabilities and rewards of the abstracted model (see
\sect\ref{sec:Other-Forms-of-abstraction} for more details). In
contrast, our work here shows how an abstracted model can preserve
exact transition and reward predictions, by `remembering' appropriate
elements of the local history. In the words of \citet{McCallum95PhD},
we detail an approach to \emph{perfectly }``uncover {[}...{]} hidden
state'' in abstractions for a large class of structured problems.

As such, the contributions of this paper are of a theoretical nature:
they provide a principled understanding of lossless abstractions in
structured (multiagent) decision problems by providing a formal framework
that gives a unified perspective on previous work, while at the same
time providing new insights and extending the scope of applicability.
The main technical result is the proof of sufficiency given in \prettyref{sec:Sufficiency}:
the smaller influence-augmented local model produced by IBA can be
used instead of the original larger model \emph{without any loss }in
solution quality (i.e., value). The proof is not only a certification
of the theory, it also serves a practical purpose: it isolates the
core technical property that needs to hold for sufficiency, thus providing
1) insight into \emph{how }abstraction of latent state factors affects
value, 2) a derivation that can be used to obtain a simplification
of influence in simpler cases, and 3) a recipe of how to prove similar
results for more complex settings.

This paper is organized as follows: First, \prettyref{sec:background}
provides the necessary background by introducing single and multiagent
models for decision making. \prettyref{sec:LFMs-and-best-responses}
introduces the concept of computing best responses (using global value
functions) to the policies of other agents and the concept of `local
form models' which formalizes a desired abstraction for an agent.
Next, in \prettyref{sec:IBA}, we bring these concepts together: we
show how an agent can locally compute a best-response (compute a local
value function) provided it is given an influence point. \prettyref{sec:IBA-with-IS-deps}
extends this framework to problems with intra-stage dependencies.
\prettyref{sec:Sufficiency} then presents the main proof of sufficiency
of our influence points, i.e., it shows that they provide sufficient
information to compute optimal policies without any loss in value.
\prettyref{sec:sub-classes-with-Compact-representations} discusses
reinterpretations of previous work on forms of influence-based abstraction
in our more general framework, while \prettyref{sec:Related-Work}
details the relations to other related work. Finally, \prettyref{sec:conclusions}
concludes.

\section{Background}

\label{sec:background} Here we concisely provide background on some
of the models that we use. The main purpose is to introduce the notation
formally. For an extensive introduction to \emph{partially observable
Markov decision processes (POMDPs)} we refer to \citet{Kaelbling98AI}
and \citet{Spaan12RLBook}, for an introduction to multiagent variants
see \citet{Seuken08JAAMAS,Oliehoek12RLBook} and \citet{Oliehoek16Book}.

Unavoidably, this manuscript contains a fair amount of terminology
and mathematical notations. To aid the reader we have included a list
of acronyms (\app\ref{sec:List-of-Acronyms}) and a list of recurring
notation~(\app\ref{sec:List-of-Notation}).

\subsection{Single-Agent Models: POMDPs}

\label{sec:background:Single-Agent-Models:-POMDPs}

Partially observable Markov decision processes, or POMDPs, provide
a formal framework for the interaction of an agent with a stochastic,
partially observable environment. That is, they provide an agent with
the capabilities to reason about both action uncertainty as well as
state uncertainty.

\subsubsection{Model}

A POMDP is a discrete time model, in which the agent selects an action
at every time step or \emph{stage. }It extends the regular \emph{Markov
decision process (MDP) }\citep{Puterman94} to settings in which the
state of the environment cannot be observed.\emph{ }It can be formally
defined as follows.

\begin{definition}[POMDP]\label{def:POMDP}

A \emph{partially observable Markov decision process (POMDP) }is defined
as a tuple $\mathcal{M}^{POMDP}=\left\langle \sS,\aAS{},\Tfunc,\R,\oAS{},\Ofunc,\hor,{\bO}\right\rangle $
with the following components:
\begin{itemize}
\item $\sS$ is a (finite) set of states $\s$. The state at some stage
$\ts$ is denoted $\sT{\ts}$;
\item $\aAS{}$ is the (finite) set of actions $\aA{}$;
\item $\Tfunc$ is the transition probability function, that specifies $\Tfunc(\s'|\s,\aA{})=\Pr(\sT{\ts+1}=\s'\mid\sT{\ts}=\s,\ \aAT{}{\ts}=\aA{}),$
the probability of a next state $\s'$ given a current state $\s$
and action $\aA{}$. This directly demonstrates the primed shorthand
notation we will occasionally use;
\item $\R$ is the immediate reward function $\R:\sS\times\jaS\times\sS\rightarrow\real$.
With $R(\s,\aA{},\s')$ we denote the reward specified for a particular
transition $s,a,s'$;
\item $\oAS{}$ is the set of observations;
\item $\Ofunc$ is the observation probability function, which specifies
$\Ofunc(\oAT{}{}|\aA{},\s')=\Pr(\oAT{}{\ts+1}=\oAT{}{}\mid\aAT{}t=\aA{},\,\sT{\ts+1}=\s')$,
the probability of a particular observation $\oA{}$ after $\aA{}$
and resulting state $s'$;
\item $\hor$ is the horizon of the problem as mentioned above;
\item $\bO\in\PrS(\mathcal{S})$, is the initial state distribution at time
$\ts=0$.\footnote{$\PrS(\cdot)$ denotes the set of probability distributions over $(\cdot)$.}
\end{itemize}
\end{definition} 

In many cases, the set of states is huge, and states can be thought
of as composed of values assigned to different variables:

\begin{definition}[Factored POMDP]\label{def:fPOMDP} In a \emph{factored
POMDP} \emph{(fPOMDP)}, the state space $\sS$ is spanned by a set
$\sfacS=\left\{ \sfacI1,\dots,\sfacI\nrSF\right\} $ of state variables
$\sfacI k$ (that are also called \emph{factors}). Each of these can
take values from its domain $\sfacvIS k$, such that the set of states
is defined as $\sS=\sfacvIS1\times\dots\times\sfacvIS\nrSF$. \end{definition} 

The merit of such a factored POMDP is that, by making the structure
of the problem (i.e., how different factors influence each other)
explicit, the model can be much more compact. In particular, the initial
state distribution can be compactly represented as a \emph{Bayesian
network}~\citep{Pearl88,Bishop06book,KollerFriedman09}, and the
transition and reward model can be specified compactly using a \emph{two-stage
dynamic Bayesian network (2DBN)}~\citep{Boutilier99JAIR}, and a
similar approach can be taken for the observation model~\citep{Poupart05PHD}.
(An example of a 2DBN will be discussed in \fig\ref{fig:house-dbn}
\vpageref{fig:house-dbn}.)

The fPOMDP model is closely related to the framework of \emph{influence
diagrams }\citep{Howard84,Tatman90IEEESMC}. In fact, by unrolling
(over time) the 2DBN we create an influence diagram. We point out,
however, that our notion of \emph{influence }(i.e., the influence
point and the resulting influence-based abstraction we will detail
in \sect\ref{sec:IBA}) is novel; it has not been considered in influence
diagrams. 

\subsubsection{Beliefs}

In contrast to regular MDPs, in a POMDP the agent cannot observe the
state; it only observes the observations. However, the observations
are not a Markovian signal: i.e., the last observation $\oAT{}{\ts}$
made by the agent does not provide the same amount of information
(to predict the rewards and the future of the process) as the \emph{action-observation
history (AOH)}, the entire history of actions and observations $\aoHistAT{}{\ts}=\left(\aAT{}0,\oAT{}1,\dots,\aAT{}{\ts-1},\oAT{}{\ts}\right)$.
This means that in general the agent needs to select its actions based
on $\aoHistAT{}{\ts}$ in order to achieve optimal performance.

Luckily, for a POMDP this history can be summarized compactly as a
\emph{belief, }which is defined as the posterior probability distribution
over states given the history: 
\[
\bA{}(\s)\defas\Pr(\s|\bO,\aoHistAT{}{\ts}).
\]
The belief does not only summarize the history, it does so in a lossless
way. That is, a belief is a \emph{sufficient statistic }for optimal
decision making~\citep{Bertsekas05DPBook_vol1}; it allows an agent
to reach the same performance as an agent that would act optimally
based on the AOH $\aoHistAT{}{\ts}$.

This belief can be recursively computed, which means that an agent
can update its belief as it interacts with its environment. We write
$\bA{}'=BU(\bA{},\aA{},\oA{})$, where $BU(\bA{},\aA{},\oA{})$ is
the belief update operator that, given a previous belief $b$ taken
action $\aA{}$ and received observation $\oA{}$, produces the next
belief:
\begin{equation}
\forall\s'\qquad BU(\bA{},\aA{},\oA{})(\s')=\frac{1}{\Pr(\oA{}|\bA{},\aA{})}\Pr(\oA{}|\aA{},\s')\sum_{\s}\Pr(\s'|\s,\aA{})\bA{}(\s).\label{eq:POMDP_BU}
\end{equation}
Here, $\Pr(\oA{}|\bA{},\aA{})$ is a normalization constant:

\[
\Pr(\oA{}|\bA{},\aA{})=\E_{\sT{}\sim\bA{},\sT{\prime}\sim T(\sT{},\aAT{}{},\cdot)}\left[\Ofunc(\aAT{}{},\sT{\prime},\oAT{}{})\right]=\sum_{\s'}\Pr(\oA{}|\aA{},\s')\sum_{\s}\Pr(\s'|\s,\aA{})\bA{}(\s).
\]

\subsubsection{Policies and Value Functions}

In a POMDP, the agent employs a \emph{policy,} $\polA{}$, to interact
with its environment. Such a policy is a (deterministic) mapping from
beliefs to actions. Note that, given the initial belief~$\bO$, such
a policy will specify an action for each observation history.\footnote{This can be seen as follows: for $\bO$ the policy specifies an action,~$\aAT{}0$,
then given $\oAT{}1$ we can compute $\argsT{\bA{}}{1}$ which we
can use to look up $\aAT{}1$, etc. }

The goal of the decision maker, or agent, in the POMDP is to choose
a policy $\polA{}$ that maximizes the expected (discounted) cumulative
reward:
\begin{equation}
\E\left[\sum_{\ts=0}^{\hor-1}\gamma^{\ts}\R(\sT{\ts},\aAT{}{\ts},\sT{\ts+1})|\bO,\polA{}\right],
\end{equation}
here
\begin{itemize}
\item $\hor$ is the horizon, i.e., the number of time steps, or \emph{stages,
}for which we want to plan,
\item the expectation is over sequences of states and observations induced
by the policy $\polA{}$,
\item $\gamma\in[0,1]$ is the discount factor.
\end{itemize}
In this work, we focus on the finite-horizon case, in which it is
typical (but not necessary) to assume $\gamma=1$.

For a finite-horizon POMDP, the optimal (action-)value function for
stage~$\ts$ can be expressed as
\begin{equation}
\QT{\ts}(\bA{},\aA{})=\begin{cases}
\R(\bA{},\aA{}), & \ts=\hor-1\\
\R(\bA{},\aA{})+\discount\sum_{\oA{}}\Pr(\oA{}|\bA{},\aA{})\V^{{\ts+1}}(BU(\bA{},\aA{},\oA{})), & \text{otherwise}
\end{cases}
\end{equation}
where $\V^{{\ts+1}}(\bA{}')=\max_{\aA{}'}\QT{\ts+1}(\bA{}',\aA{}')$
is the value of acting optimally in the next time step and $\R(\bA{},\aA{})$
is the expected immediate reward: 
\begin{equation}
\R(\bA{},\aA{})=\E_{\s\sim\bA{},\s'\sim\Tfunc(\cdot|\s,\aA{})}\left[\R(\s,\aA{},\s')\right]=\sum_{\s}\bA{}(\s)\sum_{\s'}\Pr(\s'|\s,\aA{})\R(\s,\aA{},\s').
\end{equation}

\subsection{Multiagent Models: POSGs}

The POMDP model can be extended to include multiple self-interested
agents as follows.

\begin{definition}[POSG]\label{def:POSG}

A \emph{partially observable stochastic game (POSG) }is defined as
a tuple $\mathcal{M}^{POSG}=\left\langle \agentS,\sS,\jaS,\Tfunc,\set{\REWF},\joS,\Ofunc,\hor,{\bO}\right\rangle $
with the following components:
\begin{itemize}
\item $\agentS=\{\agentI1,\dots,\agentI\nrA\}$ is the set of \emph{$\nrA$}
agents. 
\item $\sS$ is a (finite) set of states.
\item $\jaS=\aAS1\times\dots\aAS\nrA$ is the set of \emph{joint} actions
$\ja=\left\langle \aA1,\dots,\aA{\nrA}\right\rangle $, with $\aAS i$
the set of individual actions for agent~$i$.
\item $\Tfunc$ is the transition probability function, that now depends
on joint actions: $\Tfunc(\sT{\ts+1}|\sT{\ts},\jaT{\ts})=\Pr(\sT{\ts+1}|\sT{\ts},\jaT{\ts})$.
\item $\set{\REWF}=\left\langle \RA1,\dots\RA{\nrA}\right\rangle $ is the
collection of immediate reward functions (one for each agent). Each
$\RA i:\sS\times\jaS\times\sS\rightarrow\real$ maps from states,
joint actions and next states to an immediate reward for agent~$i$.
\item $\joS=\oAS1\times\dots\times\oAS\nrA$ is the set of joint observations~$\jo=\langle\oA1,\dots,\oA\nrA\rangle$,
with $\oAS i$ the set of individual observations for agent~$i$,
\item $\Ofunc$ is the observation probability function, which specifies
$\Pr(\jo|\ja,\s')$, the probability of a particular joint observation
$\jo$ after $\ja$ and resulting state $\s'$.
\item $\hor$ is the horizon of the problem as mentioned above.
\item $\bO\in\PrS(\sS)$, is the initial state distribution at time $\ts=0$.
\end{itemize}
\end{definition} 

Since in a POSG each agent has its own goal, there no longer is a
definition of optimality. Instead it is customary to focus on game-theoretic
solution concepts~\citep{Hansen04AAAI}. Such solutions, e.g., Nash
equilibria, typically specify a tuple of policies $\jpol=\left\langle \polA1,\dots,\polA{\nrA}\right\rangle $,
one for each agent, that are in equilibrium. In general, we will refer
to a tuple of policies $\jpol$ as a \emph{joint policy. }

Of course, it is also possible to consider cooperative teams of agents.
In this case, we align the goals of the agents by giving them the
same reward function:

\begin{definition}[Dec-POMDP]\label{def:Dec-POMDP} A \emph{decentralized
partially observable Markov decision process (Dec-POMDP)} is a POSG
where all agents share the same reward function: $\forall_{i,j}\;\RA i=\RA j$.
\end{definition} 

Since interests are aligned, in a Dec-POMDP we can speak about optimality.
Moreover, there is guaranteed to be at least one \emph{deterministic
}joint policy that is optimal~\citep{Oliehoek08JAIR}. As was the
case for POMDPs, we can also consider variants of the multiagent models
with factored state spaces. We will refer to these as \emph{factored
POSGs (fPOSGs) }and \emph{factored Dec-POMDPs (fDec-POMDPs)}~\citep{Oliehoek08AAMAS}.\footnote{More recently, researchers have also investigated deterministic and
non-deterministic versions, called (factored) qualitative Dec-POMDP~\citep{Brafman13AAAI}.
We will not particularly target this special case in this paper, but
note that ideas of influence search can be exploited in this context
too \citep{Bazinin18GCAI}.}

\begin{figure}
\begin{centering}
{
\psfrag{1}[cc][cc]{1}
\psfrag{2}[cc][cc]{2}
\psfrag{g}[cc][cc]{T}\includegraphics[width=4cm]{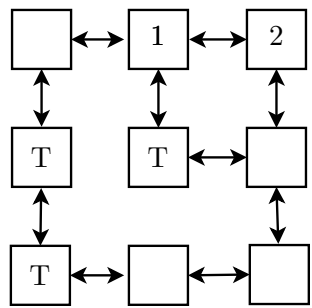}}
\par\end{centering}
\caption{A possible instantiation of the \<housesearch> problem: 1, 2 represent
the starting locations of the agents, while `T' encodes the possible
locations of the target.}

\label{fig:gsearch}

\end{figure}

As an example, we consider the \<housesearch>  problem \citep{Oliehoek11EUMAS},
in which a team of robots must find a target (say a remote control)
in a house with multiple rooms. This task is representative of an
important class of problems in which a team of agents needs to locate
objects or targets. In \<housesearch> the assumption is that a prior
probability distribution over the location of the target is available
and that the target is stationary or moves in a manner that does not
depend on the strategy used by the searching agents.

\begin{example} The \<housesearch>  environment can be represented
by a graph, as illustrated in \fig\ref{fig:gsearch} for the case
of two agents. At every time step each agent can stay in the current
room or move to a next one. The location of an agent~$i$ at time
step $t$ is denoted $l_{i}^{t}$ and that of the target is denoted
$l_{tgt}^{t}$. In general, the target could move with probabilities
$p(l_{tgt}^{'}|l_{tgt})$. The actions (movements) of each agent have
a specific cost $c_{i}(l_{i},\aA i)$ (e.g., the energy consumed by
navigating to a next room) and can fail; we allow for stochastic transitions
$p(l_{i}'|l_{i},\aA i)$. Also, each robot might receive a penalty
$c_{time}$ for every time step that the target is not found yet.
When a robot is in (or near) the same node as the target, there is
a probability of detecting the target $p(detect_{i}|l_{tgt},l_{i})$,
which will be modeled by a Boolean state variable `target found'
$\argsT{f}{\ts}$, which both agents can observe (thus modeling a
communication channel which the agents can only use to inform each
other of detection). When the target is detected, the agents also
receive a reward $r_{detect}$. Given the prior distribution and model
of target behavior, the goal is to optimize the sum (over time) of
rewards, thus trading off movement cost and probability of detecting
the target as soon as possible. In this paper, we focus on the local
perspective of a protagonist agent and therefore will assume that
each agent has its individual rewards (so the POSG setting).\footnote{In previous work, the house search problem was treated as a Dec-POMDP
by defining the team reward as the sum of the individual rewards~\citep{Oliehoek11EUMAS}.}\end{example}

\begin{figure}[tbh]
\hfill{}\subfloat[With intra-stage connections.]{\begin{centering}
{\input{figs/frag_housesearch_2DBN.tex}\includegraphics[scale=0.4]{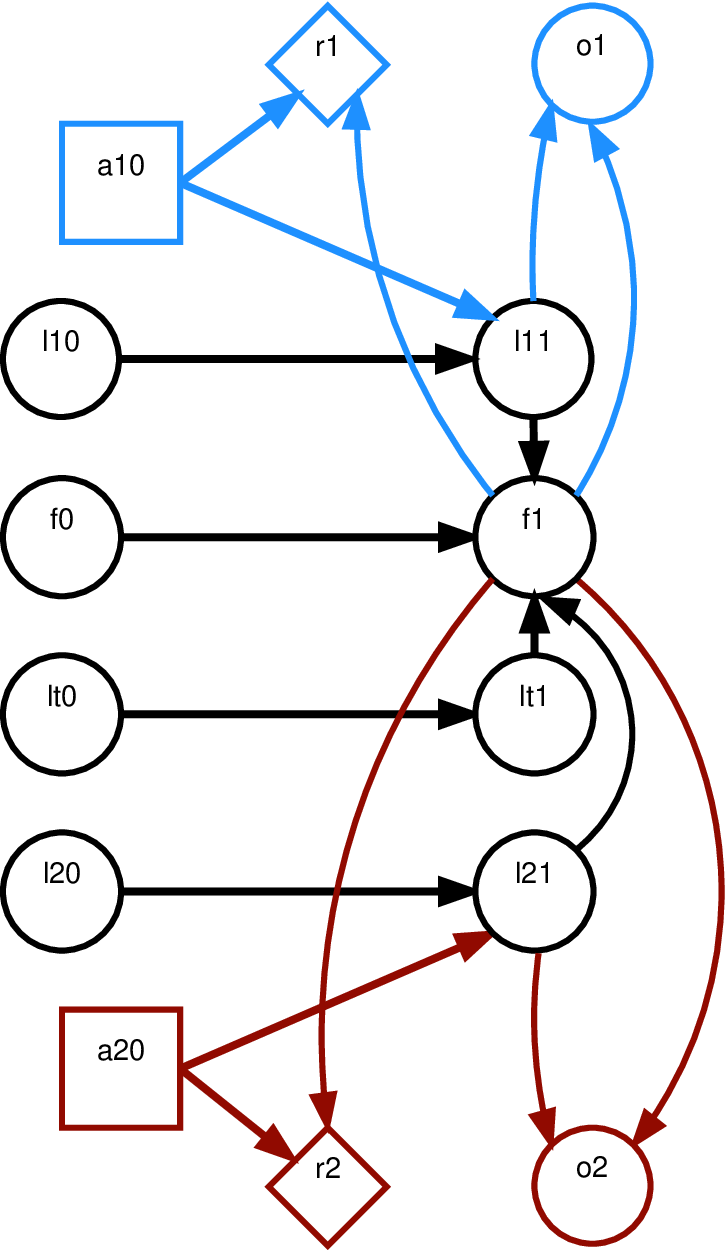}}
\par\end{centering}

\label{fig:house_dbn_intrastage}}\hfill{}\subfloat[Without intra-stage connections.]{\begin{centering}
{\input{figs/frag_housesearch_2DBN.tex}\includegraphics[scale=0.4]{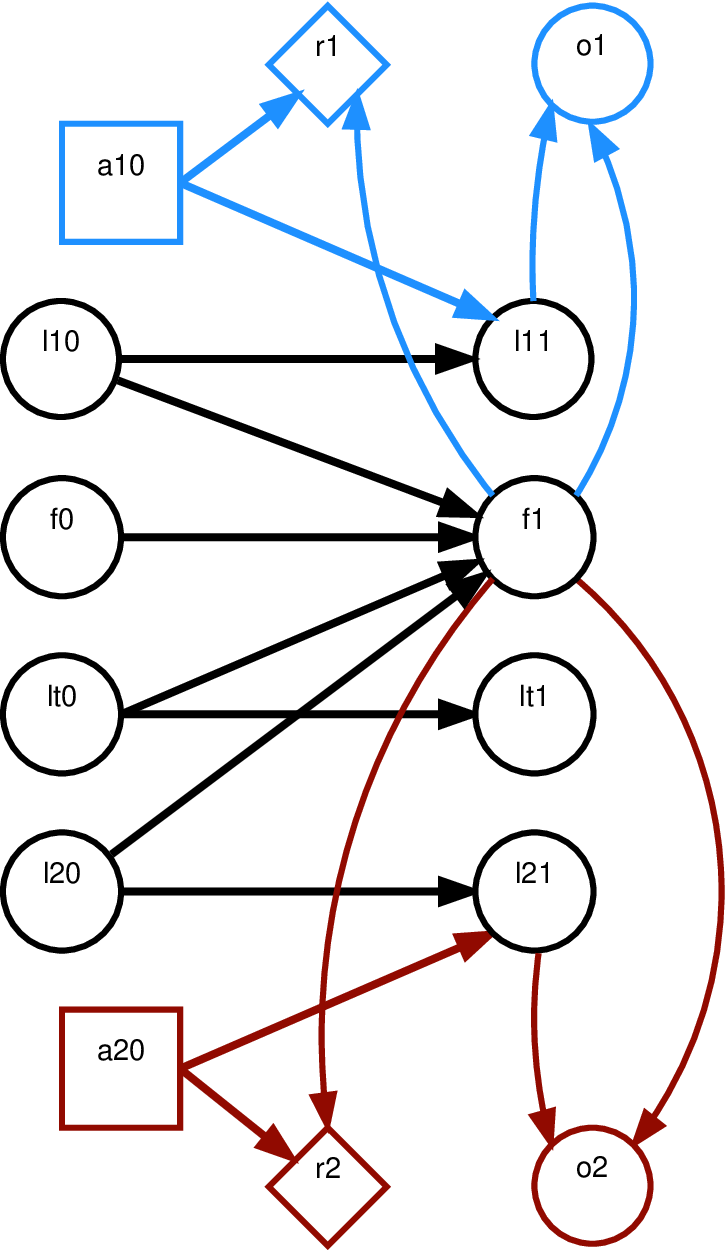}}
\par\end{centering}

\label{fig:house_dbn_no-intra}}\hfill{}

\caption{Factored representation of the \<housesearch> problem. Actions, observations
and rewards of the first agent are in light blue, while those of agent~2
are in dark red. State variables are in black. We use standard shapes
for influence diagrams: rectangles for actions, circles for random
variables, and diamonds for rewards~\protect\citep[e.g.,][]{RussellNorvig09Book3rdEd}.}

\label{fig:house-dbn}
\end{figure}

\fig\ref{fig:house_dbn_intrastage} demonstrates how a two-stage
dynamic Bayesian network (2DBN) can be used to compactly represent
the transition, observation, and reward model~\citep{Boutilier99JAIR}.\footnote{More formally, since we include actions (decisions) and rewards (utilities),
diagrams like this are a type of influence diagram~or decision network.
However, not to introduce further terminology, we will refer to them
simply as 2DBN.} For instance, for each state variable at a state $t+1$, the 2DBN
shows which other entities (state factors and actions) influence it.
The figure illustrates that most dependencies are \emph{across-stage}
(e.g., $l_{2}^{t}$ influences $l_{2}^{t+1}$) but that it is also
possible to have \emph{intra-stage} \emph{dependencies (ISDs). }For
instance, whether the target will be detected at stage $t+1$ depends
on $l_{2}^{t+1}$ not on $l_{2}^{t}$. The representation of the transition
model is compact since it can be represented as a product of \emph{conditional
probability tables (CPTs)}, each of which are exponential only in
the number of incoming dependencies. So as long as the number of incoming
connections is limited, the transition probabilities can be represented
compactly. \fig\ref{fig:house_dbn_intrastage} also shows that this
type of representation can also be employed for observation probabilities,
as well as rewards.

Since ISDs complicate the notation and definition of influence, we
also consider a version of the problem that has no intra-stage connections,
shown in \fig\ref{fig:house_dbn_no-intra}. For rewards and observations,
intra-stage connections are still allowed. (In fact, since the observation
probabilities in the standard POMDP definition depend on the next
state $\s'$, there is no way of representing them without intra-stage
connections). Note that this is a slightly different problem than
the problem represented in \fig\ref{fig:house_dbn_intrastage}: in
the problem without ISDs the agents have a chance of detecting the
target at stage $\ts+1$ if they are co-located with the target at
stage $\ts$, which means that there is a one-step delay incurred
before they receive the reward. This illustrates the fact the ISDs
do allow for a more expressive model, and that therefore developing
theory that support such connections is an important goal. 

To facilitate easier exposition, in \sect\ref{sec:IBA} we will first
introduce the concept of influence-based abstraction without ISDs.
These will be considered in \sect\ref{sec:IBA-with-IS-deps}. Before
we can jump to the topic of influence-based abstraction, however,
we will need to discuss decision problems from a local perspective,
in \sect\ref{sec:LFMs-and-best-responses}, which covers problems
with ISDs.

\section{Best Responses and Local-Form Models }

\label{sec:LFMs-and-best-responses}

In contrast to the typical solutions to POSGs and Dec-POMDPs, which
try and identify a joint policy as the solution, this paper focuses
on the local perspective of an individual agent. From this perspective,
the agent's goal is to compute a best-response to the policies used
by other agents. That is, given a multiagent model with state uncertainty
(either a POSG or Dec-POMDP) and given some policy for the other agents
$\jpolG{\excl i}=\left\langle \polA1,\dots,\polA{i-1},\polA{i+1},\dots,\polA{\nrA}\right\rangle $,
we want to compute the best response $\polA i^{BR}$ for agent~$i$.
Such best-response computation is obviously important for self-interested
agents (i.e., in POSGs), but is also an important component in many
Dec-POMDP solution methods~\citep{Nair03IJCAI,Nair05AAAI,Kim06AAAISS,Pajarinen11NIPS,Lauri20JAAMAS}.
Also, let us point out that we make no restrictions on the policies
employed by the other agents: they are general mappings from the action-observation
histories $\aoHistAT j{\ts}$ to probability distributions over actions.
For instance, their policies could be learning algorithms such as
Q-learning. As such, the setting we consider is very general.

\begin{figure}
\begin{centering}
\includegraphics[scale=0.75]{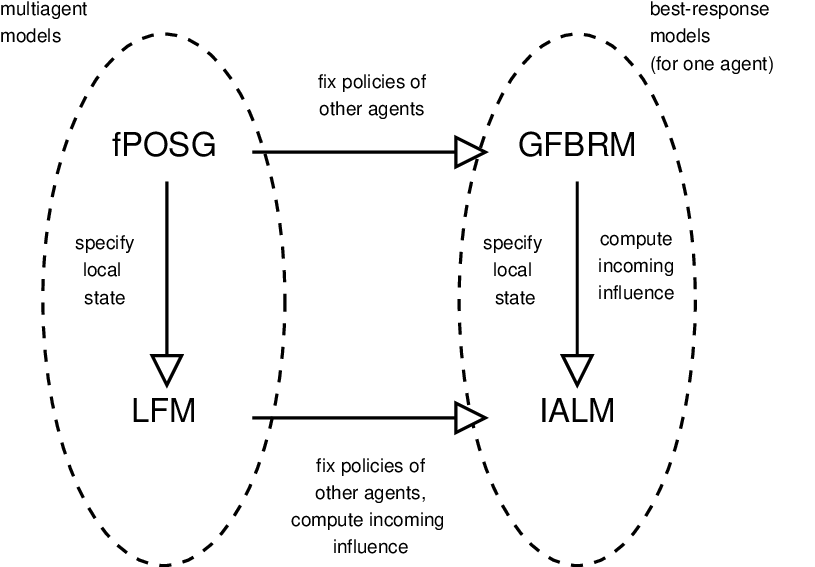}
\par\end{centering}
\caption{Overview of various models used.}

\label{fig:modelTypes}
\end{figure}

As illustrated in \fig\ref{fig:modelTypes}, we will consider a number
of different types of models in this paper. The starting point is
given by the fPOSG or a special case thereof (e.g., a Dec-POMDP).
We refer those models as global-form models. For such models, it is
possible to directly compute a best-response by fixing the policies
of the other agents. We refer to the resulting POMDP as a global-form
best-response model (GFBRM); these models will be introduced next.
Subsequently, we will introduce \emph{local-form models (LFMs), }which
restrict the state factors that each agent primarily cares about.
That is, an agent in an LFM only reasons about a subset of factors.
This will then form the basis for computing best-responses in such
a local model, called \emph{influence-augmented local model (IALM),}
which will be enabled by influence-based abstraction introduced in
\prettyref{sec:IBA}.

\subsection{Global-Form Best-Response Model }

\label{sec:GFBRM}

In this section we define a Global-Form Best-Response Model (GFBRM)
that an agent can use in order to compute a best-response in a general
POSG. We first define this model and then talk about value functions
for this model.\footnote{Our formulation here is closely related to the way best-responses
are computed in DP-JESP~\citep{Nair03IJCAI}: essentially our representation
here is a reformulation that makes explicit the fact that fixing the
policies of other agents leads to a single-agent POMDP model.}

\paragraph{Specification of the Model}

The basic idea of defining a best-response model is shown in \fig\ref{fig:GFBRM}.
By fixing $\jpolG{\excl i}$, the policies of the other agents, all
the choice nodes are turned into random variables that now depend
on the AOHs that those agents observed~\citep{Nair03IJCAI}. So the
key construct here is that the AOH of the other agent(s) is made part
of the hidden state (often termed latent state factors) of the best-response
model. This can be formalized as follows.
\begin{figure}[tb]
\centering{}{\input{figs/frag_brm.tex}\includegraphics[scale=0.4]{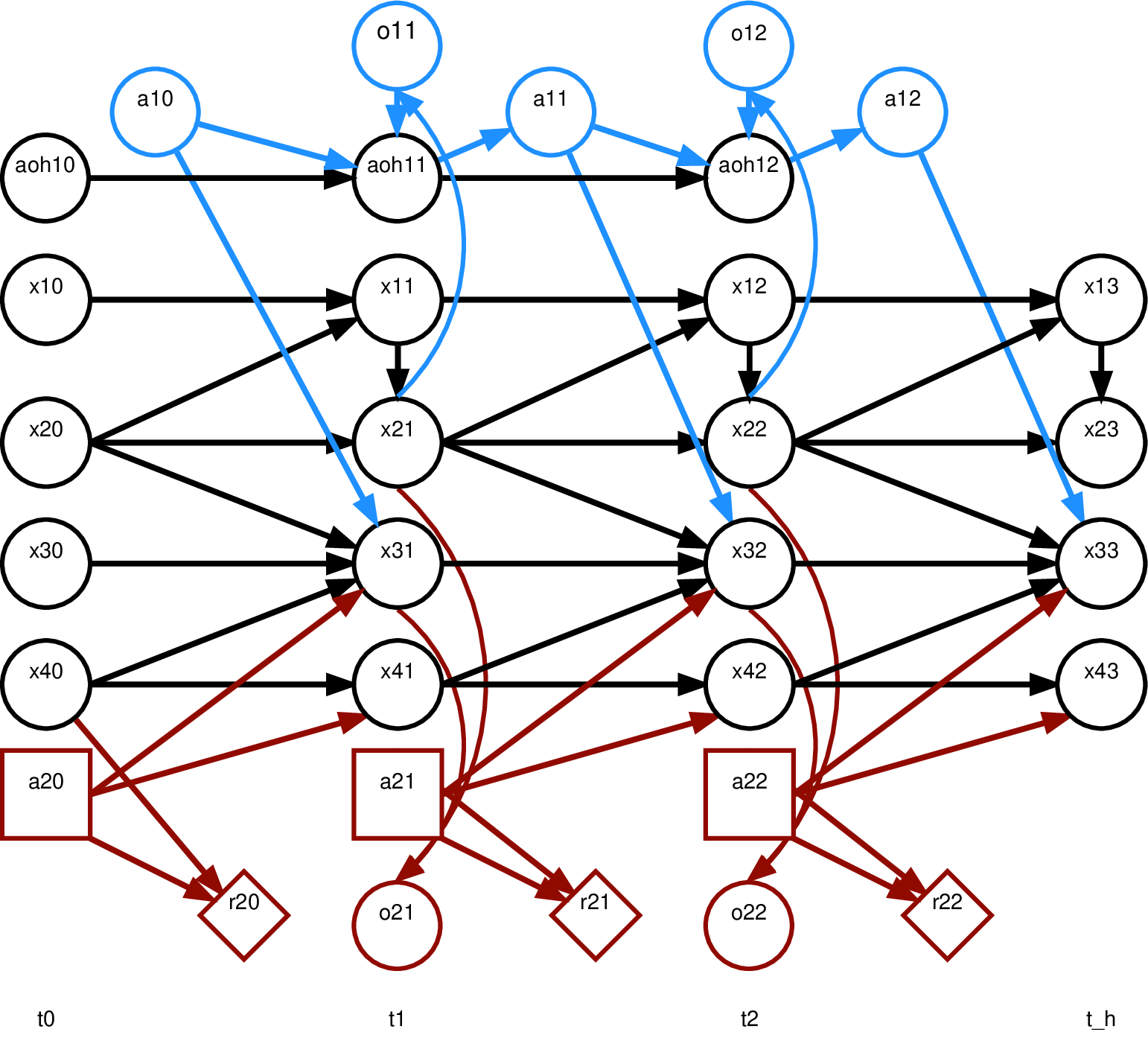}}\caption{A hypothetical global-form best-response model for agent~2, unrolled
over time. This model has a number of state factors $\sfacI k$. In
addition, the action-observation history $\aoHistAT1{\ts}$ (or, more
general, internal state) of agent~1 can be interpreted as a state
factor in this model. \label{fig:GFBRM}}
\end{figure}

\begin{definition}[Global-Form Best-Response Model]{\sloppy Let
$\mathcal{M}^{POSG}=\left\langle \agentS,\sS,\jaS,\Tfunc,\set{\REWF},\joS,\Ofunc,\hor,{\bO}\right\rangle $
be a (f)POSG and let $\jpolG{\excl i}$ be a profile of policies for
all agents but $i$. We say that the POMDP $\mathcal{M}_{i}^{GFBR}(\mathcal{M}^{POSG},\jpolG{\excl i})=\left\langle \bar{\sAS i},\aAS{i},\bar{\Tfunc}_{i},\bar{\REWF}_{i},\oAS{i},\bar{\Ofunc}_{i},\hor,\bar{\bA i}{}^{0}\right\rangle $
is a \emph{Global-Form Best-Response Model (GFBRM) }for agent~$i$,
where}
\begin{itemize}
\item $\bar{\sAS i}$ is the set of augmented states $\sAugAT i\ts=\langle\s,\aoHistAT{\excl i}{\ts}\rangle$
that specify an underlying state of the POSG as well as an AOH history
for all the other agents.
\item $\aAS{i},\oAS{i}$ are the (unmodified) sets of actions and observations
for agent~$i$.
\item The transitions
\begin{eqnarray}
\bar{\Tfunc}_{i}(\sAugAT i{t+1}|\sAugAT it,\aAT i{\ts}) & = & \bar{\Tfunc}_{i}(\langle\sT{\ts+1},\aoHistAT{\excl i}{\ts+1}\rangle|\langle\sT{\ts},\aoHistAT{\excl i}{\ts}\rangle,\aAT i{\ts})\nonumber \\
 & = & \bar{\Tfunc}_{i}(\langle\sT{\ts+1},\left(\aoHistAT{\excl i}{\ts},\aAT{\excl i}{\ts},\oAT{\excl i}{\ts+1}\right)\rangle|\langle\sT{\ts},\aoHistAT{\excl i}{\ts}\rangle,\aAT i{\ts})\nonumber \\
 & = & \Pr(\oAT{\excl i}{\ts+1},\sT{\ts+1},\aAT{\excl i}{\ts}|\sT{\ts},\aAT i{\ts},\aoHistAT{\excl i}{\ts})\nonumber \\
 & = & \Pr(\oAT{\excl i}{\ts+1}|\aAT i{\ts},\aAT{\excl i}{\ts},\sT{\ts+1})\Pr(\sT{\ts+1}|\sT{\ts},\aAT i{\ts},\aAT{\excl i}{\ts})\Pr(\aAT{\excl i}{\ts}|\aoHistAT{\excl i}{\ts})\nonumber \\
 & = & \left[\sum_{\oAT i{\ts+1}}O(\joT{\ts+1}|\jaT{\ts},\sT{\ts+1})\right]T(\sT{\ts+1}|\sT{\ts},\jaT{\ts})\polA{\excl i}(\aAT{\excl i}{\ts}|\aoHistAT{\excl i}{\ts})
\end{eqnarray}
with $\polA{\excl i}(\aAT{\excl i}{\ts}|\aoHistAT{\excl i}{\ts})=\prod_{j\neq i}\polA j(\aAT j{\ts}|\aoHistAT j{\ts})$
the probability of $\aAT{\excl i}{\ts}$ given $\aoHistAT{\excl i}{\ts}$
according to $\polA{\excl i}$.
\item The observations
\begin{eqnarray}
\bar{\Ofunc}_{i}(\oAT i{\ts+1}|\aAT i{\ts},\sAugAT i{t+1}) & = & \bar{\Ofunc}_{i}(\oAT i{\ts+1}|\aAT i{\ts},\langle\sT{\ts+1},\left(\aoHistAT{\excl i}{\ts},\aAT{\excl i}{\ts},\oAT{\excl i}{\ts+1}\right)\rangle)\nonumber \\
 & = & \Pr(\oAT i{\ts+1}|\aAT i{\ts},\aAT{\excl i}{\ts},\sT{\ts+1},\oAT{\excl i}{\ts+1})\nonumber \\
 & = & \frac{\Pr(\oAT i{\ts+1},\oAT{\excl i}{\ts+1}|\aAT i{\ts},\aAT{\excl i}{\ts},\sT{\ts+1})}{\Pr(\oAT{\excl i}{\ts+1}|\aAT i{\ts},\aAT{\excl i}{\ts},\sT{\ts+1})}\nonumber \\
 & = & \frac{O(\joT{\ts+1}|\jaT{\ts},\sT{\ts+1})}{\sum_{\oAT i{\ts+1}}O(\joT{\ts+1}|\jaT{\ts},\sT{\ts+1})}
\end{eqnarray}
(Note that $\joT{\ts+1}=\left\langle \oAT i{\ts+1},\oAT{\excl i}{\ts+1}\right\rangle $,
such that the summation is over the component of $\joT{\ts+1}$ corresponding
to agent~$i$).
\item $\bar{\R}_{i}$ is the augmented reward model
\begin{eqnarray}
\bar{\R}_{i}(\sAugAT it,\aAT i{\ts},\sAugAT i{t+1}) & = & \bar{\R}_{i}(\langle\sT{\ts},\aoHistAT{\excl i}{\ts}\rangle,\aAT i{\ts},\langle\sT{\ts+1},\aoHistAT{\excl i}{\ts+1}=\left(\aoHistAT{\excl i}{\ts},\aAT{\excl i}{\ts},\oAT{\excl i}{\ts+1}\right)\rangle)\nonumber \\
 & = & \RA i(\sT{\ts},\aAT i{\ts},\aAT{\excl i}{\ts},\sT{\ts+1})
\end{eqnarray}
Note that $\aAT{\excl i}{\ts}$ is specified by $\sAugAT i{t+1}$.
\item $\hor$ is the (unmodified) horizon.
\item $\bar{\bA i}{}^{0}$ is the initial belief\emph{. }
\end{itemize}
\end{definition}

A GFBRM is a POMDP, which means that an agent can track a belief,
which is now a distribution over \emph{augmented states }$\sAugA i=\langle\sT{\ts},\aoHistAT{\excl i}{\ts}\rangle$,
as usual. We will refer to such beliefs as \emph{global-form beliefs,
}denoted $\gbA i$. The initial global-form belief follows directly
from the initial belief of the POSG. Since at the first stage, the
history of the other agents is the empty history $\aoHistEmpty$,
it is trivially constructed from $\bO$: 
\[
\forall_{s}\ \gbAT i0(\left\langle s,\aoHistEmpty\right\rangle )\defas\bO(s).
\]

Note that the description of the GFBRM depends rather crucially on
the fact that we choose AOHs for the representation of the internal
state of the other agent(s). That is, we assume that the policies
of the other agent(s) are based on their AOHs. While this is a very
general model, other models of other agents with a more limited description
of internal state (e.g., finite state controllers) can be useful too.
For such more compact descriptions, however, it is not always possible
to construct a POMDP model with an independent transition and observation
model. Instead, one may need to replace $\bar{T},\bar{O}$ by a combined
`dynamics function' $\bar{D}$ that specifies $\bar{D}(\sAugAT i{t+1},\oAT i{\ts+1}|\sAugAT it,\aAT i{\ts})$.
For more details see \citet{Oliehoek14MSDM}.\footnote{{\renewcommand{\aoHistAT}[2]{I_{#1}^{#2}}Essentially in such a setting
we have that augmented states are tuples of nominal states and internal
states of other agents $\sAugAT it=\langle\sT{\ts},\aoHistAT{\excl i}{\ts}\rangle$.
The internal states of the other agent are updated based upon the
taken actions and observations, but do not store those actions and
observations. This means that, in general, $\bar{D}$ is specified
as a marginal:
\begin{multline*}
\bar{D}(\sAugAT i{t+1},\oAT i{\ts+1}|\sAugAT it,\aAT i{\ts})=\Pr(\langle\sT{\ts+1},\aoHistAT{\excl i}{\ts+1}\rangle,\oAT i{\ts+1}|\langle\sT{\ts},\aoHistAT{\excl i}{\ts}\rangle,\aAT i{\ts})\\
=\sum_{\aAT{\excl i}{\ts},\oAT{\excl i}{\ts+1}}\Pr(\aoHistAT{\excl i}{\ts+1}|\aoHistAT{\excl i}{\ts},\aAT{\excl i}{\ts},\oAT{\excl i}{\ts+1})O(\joT{\ts+1}|\jaT{\ts},\sT{\ts+1})T(\sT{\ts+1}|\sT{\ts},\jaT{\ts})\polA{\excl i}(\aAT{\excl i}{\ts}|\aoHistAT{\excl i}{\ts})
\end{multline*}
and it is not possible to decompose it into a separate transition
and observation function.}}

\paragraph{Value Function}

Since a GFBRM is just a POMDP, all POMDP theory and solution methods
apply. E.g., the optimal (action-)value function is given by:

\begin{equation}
\QAT i\ts(\gbA i,\aAT i\ts)=\RA i(\gbA i,\aAT i\ts)+\discount\sum_{\oAT i{\ts+1}}\Pr(\oAT i{\ts+1}|\gbA i,\aAT i\ts)\VAT i{\ts+1}(BU(\gbA i,\aAT i\ts,\oAT i{\ts+1}))\label{eq:Q(b,a)__GFBRM}
\end{equation}
where 
\begin{align}
\RA i(\gbA i,\aAT i\ts) & =\E_{\sAugAT it\sim\gbA i,\sAugAT i{t+1}\sim\Aug T_{i}(\sAugAT it,\aAT it,\cdot)}\left[\Aug{\RA i}(\sAugAT i\ts,\aAT i\ts,\sAugAT i{\ts+1})\right]\nonumber \\
 & =\sum_{\sT{\ts}}\sum_{\sT{\ts+1}}\sum_{\jaG{\excl i}}\Pr(\sT{\ts+1}|\sT{\ts},\ja)\RA i(\sT{\ts},\ja,\sT{\ts+1})\sum_{\aoHistAT{\excl i}{\ts}}\Pr(\jaG{\excl i}|\aoHistAT{\excl i}{\ts})\gbA i(\sT{\ts},\aoHistAT{\excl i}{\ts})\label{eq:R(gfb,a)__GFM}
\end{align}
(see \app\ref{app:GFBRM_reward}) and
\begin{multline}
\Pr(\oAT i{\ts+1}|\gbA i,\aAT it)=\E_{\sAugAT it\sim\gbA i,\sAugAT i{t+1}\sim\Aug T_{i}(\sAugAT it,\aAT it,\cdot)}\left[\Aug O_{i}(\oAT i{\ts+1}|\aAT it,\sAugAT i{t+1})\right]\\
=\sum_{\sT{\ts}}\sum_{\sT{\ts+1}}\sum_{\jaG{\excl i}}\sum_{\joGT{\excl i}{\ts+1}}\Pr(\sT{\ts+1}|\sT{\ts},\ja)\Pr(\joT{\ts+1}|\ja,\sT{\ts+1})\sum_{\aoHistAT{\excl i}{\ts}}\Pr(\jaG{\excl i}|\aoHistAT{\excl i}{\ts},\jpolG{\excl i})\gbA i(\sT{\ts},\aoHistAT{\excl i}{\ts})\label{eq:P(o|gfb,a)__GFM}
\end{multline}
(see \app\ref{app:GFBRM_observ}.)

Solution of the GFBRM gives the best-response value for agent~$i$:

\begin{equation}
\VA i(\polA{\excl i})\defas\VAT i0(\gbAT i0).\label{eq:V(b0)_JESP}
\end{equation}

\subsection{Local-Form Model }

GFBRMs allow an agent~$i$ to compute a best-response policy against
the fixed policies $\jpolG{\excl i}$ of the other agents. A difficulty
here is that agent~$i$ needs to reason about many state factors
as well as the internal state (the action-observation history) of
the other agents. That is, drawing an analogy to human interactions,
it is like in a simple collaborative task (e.g., carrying a table),
we would need to reason over the inner working of our collaborator's
brain, as well as over the sequence of images that he or she perceives.
Clearly, such an approach is infeasible in general. To make a step
in the direction to overcome this problem, here we introduce \emph{local-form
models (LFMs)} which restrict the set of state factors that each agent
primarily cares about, and eliminates the dependence on the AOH of
other agents. 

\paragraph{Local States}

An LFM augments an fPOSG with a function that provides a description
of each agent\textquoteright s \emph{local state}, i.e., the set of
variables that each agent will model as part of its local problem.\footnote{Note that the word `local' does not need to imply any form of spatial
proximity. For instance, in \<housesearch> the agent might model
its own location (which is spatial), and whether the target has been
found (not spatial).}  Local state descriptions comprise potentially overlapping subsets
of state factors that will allow us to decompose an agent\textquoteright s
best-response computation from the global state. We start with some
definitions.

\begin{definition}[Local state function] The \emph{local state function
}$\LSF:\agentS\rightarrow2^{\sfacS}$ maps from agents to subsets
of state factors $\LSF(\agentI i)\subseteq\sfacS$. \end{definition}

The local state function defines the local state space of each agent.
In particular, we say that a state factor $\sfac\in\sfacS$ is \emph{modeled}
by an agent~$i$ if it is part of its local state space: $\sfac\in\LSF(\agentI i)$.

\begin{definition}[Local state space] The\emph{ local state space
}of agent~$i$ is defined as the Cartesian product of the values
that its modeled state factors can take:
\begin{equation}
\slfmAS i\;\defas\;\prod_{\substack{k\text{ s.t.}\sfacI k\in\LSF(\agentI i)}
}\sfacvIS k
\end{equation}
(remember that $\sfacS$ is the set of state factors, while $\sfacvIS k$
is the set of values that the $k$-th state factor $\sfacI k$ can
take).\end{definition}

\begin{definition}[Observation-relevant factor]We say that a state
factor $\sfac$ is \emph{observation-relevant} for an agent~$i$,
denoted $\text{ORel}{}_{i}(\sfac)$, if it affects the probability
of the agent\textquoteright s observation. That is, when in the 2DBN
there is a link from $\sfacT t$ to $\oAT i{\ts}$ (i.e., $\sfac$
is a parent of $\oAT i{\ts}$).\end{definition}

\begin{definition}[Reward-relevant factor]Similarly, a state factor
$\sfac$ is \emph{reward-relevant} for an agent~$i$, $\text{RRel}{}_{i}(\sfac)$
if it affects the agent\textquoteright s rewards, i.e., if $\sfacT t$
or $\sfacT{t+1}$ is a parent of $\RAT{i}{t}$. \end{definition}

We can now define the local-form model.

\begin{definition}[Local-form model] \label{dfn:lfm} A \emph{local-form
POSG,} also referred to as\emph{ local-form model (LFM),} is a pair
$\mathcal{M}^{LFM}=\left\langle \mathcal{M},\LSF\right\rangle $,
where $\mathcal{M}$ is an fPOSG and $\LSF$ is a local state function
such that, for all agents: 
\begin{enumerate}
\item All observation-relevant factors are in the local state: $\forall_{\agentI i}\forall_{\sfac}\;\text{ORel}{}_{i}(\sfac)\implies\sfac\in\LSF(\agentI i).$
\item All reward-relevant factors are in the local state: $\forall_{\agentI i}\forall_{\sfac}\;\text{RRel}{}_{i}(\sfac)\implies\sfac\in\LSF(\agentI i).$
\end{enumerate}
\end{definition} 

\paragraph{Modeled and Non-modeled Factors}

The basic idea behind the definition of the local-form model is to
avoid reasoning over the subset of variables from the global-form
model that are superfluous when it comes to computing the best response.
Therefore, these non-modeled factors can be abstracted away. The requirements
on observation- and reward-relevant factors make certain that the
observation probabilities and rewards are still specified in this
abstracted model. Note also that this means that we will only be
able to abstract away (latent) state variables, not observation variables
themselves. We will show that such latent factor abstraction can,
in principle, be performed without loss in value. This certainly would
not be the case for abstracting away observation variables: in general
this would lead to a loss of information and a corresponding drop
in achievable value~\citep{Oliehoek08JAIR}.

The focus in this text is on the best-response perspective for one
agent~$i$. This allows us to divide the set of state factors in
ones modeled by agent~$i$'s local problem (indicated with $\MF$)
and ones that are not modeled (indicated with $\NMF$).\footnote{More generally, from the perspective of agent~$i$, $\LSF$ partitions
the modeled factors $\LSF(i)$ in two sets: a set of \textit{private}
factors that it models but other agents do not, and a set of \textit{mutually-modeled
factors} (MMFs) that are modeled by agent~$i$ as well as some other
agent~$j$. This distinction plays a crucial role in influence search
for TD-POMDPs~\citep{Witwicki10AAMAS}, but is less important for
computing best-responses as considered in this document.} To reduce the notational load, we will no longer distinguish between
a factor ($\sfacI k$ above) and its values ($\sfacvIS k$ above).
In particular, we will simply write
\begin{itemize}
\item $\mfI k$ (an instantiation of) a modeled factor (with index $k$),
\item $\mfA i$ (an instantiation of) all modeled factors of agent~$i$,
\item $\nmfI k$ (an instantiation of) a non-modeled factor (with index
$k$),
\item $\nmfA i$ (an instantiation of) all non-modeled factors of agent~$i$,
\end{itemize}
such that $\sT\ts=\langle\mfAT i{\ts},\nmfAT i{\ts}\rangle$. We stress
that `modeled' is different from `observed'. In particular, our
aim is to construct a smaller POMDP with fewer (modeled) factors,
but those factors may not be observable. In fact, all state factors
$\mfI k$ (and of course also $\nmfI k$) are expressed as latent
variables. When an agent can somehow (noisily) perceive information
about $\mfI k$, this should be modeled by the observation function:
there should be an arrow from such factors to the observation $\oA i$
of the agent and the CPT of $\oA i$ should appropriately express
the observability of factor. Note that by construction of the LFM
(cf.\ \dfn\ref{dfn:lfm}), no such dependencies may exist from a
$\nmfI k$ to $\oA i$. In general, the observation $\oA i$ may itself
consists of multiple observation factors, but we will not consider
this in this paper.

\paragraph{Transition Probabilities}

In an LFM, the probability of the next local state is the marginal
of the entire state:
\begin{equation}
\Pr(\mfAT i{\ts+1}|\sT{\ts},\aA i,\jaG{\excl i})=\sum_{\nmfAT i{\ts+1}}\Pr(\mfAT i{\ts+1},\nmfAT i{\ts+1}|\sT{\ts},\aA i,\jaG{\excl i})\label{eq:P_xm_sa}
\end{equation}

In an LFM, just as in a normal fPOSG, the flat transition probabilities
on the right hand side of this equation are given by the product of
the CPTs. However, from the perspective of an agent~$i$ we can now
group these CPTs in three different categories: 1) those corresponding
to modeled factors that are only affected by other factors and actions
that are modeled, 2) those corresponding to modeled factors that are
affected by at least one factor or action of the external problem,
and 3) those corresponding to non-modeled factors. We will refer to
the state factors corresponding to these as:
\begin{enumerate}
\item \emph{Only-locally-affected factors (OLAFs)~}$\mflI k$. These can
have incoming arrows from all modeled factors $\mfAT i{\ts}$ at the
previous stage, and from all modeled factors $\mfAT i{\ts+1}$ intra-stage
(but, obviously, excluding $\mflIT k{\ts+1}$ itself, and respecting
a non-cyclic structure as any 2DBN).
\item \emph{Non-locally-affected factors (NLAFs)}~$\mfnI k$. These are
affected by at least one non-modeled (intra-stage or previous-stage)
factor or action of another agent.
\item \emph{Non-modeled factors (NMFs)}~$\nmfI k$.
\end{enumerate}
(Note that the oversets on $\mfI{}$ were chosen to resemble `o'
and `n' for OLAF and NLAF respectively). These three types of factors
are illustrated in \fig\ref{fig:LFM-factor-types}, which shows a
hypothetical local-form model. Using the introduced notation, we can
write the transition probabilities as: 
\begin{multline}
\Pr(\sT{\ts+1}|\sT{\ts},\aA i,\jaG{\excl i})=\left[\Pr(\mflAT i{\ts+1}|\dots)\Pr(\mfnAT i{\ts+1}|\dots)\Pr(\nmfAT i{\ts+1}|\dots)\right]\\
=\Pr(\mflAT i{\ts+1}|\mfAT i{\ts},\mfnAT i{\ts+1},\aA i)\Pr(\mfnAT i{\ts+1}|\mfAT i{\ts},\mflAT i{\ts+1},\nmfAT i{\ts},\nmfAT i{\ts+1},\aA i,\jaG{\excl i})\Pr(\nmfAT i{\ts+1}|\mfAT i{\ts},\mfAT i{\ts+1},\nmfAT i{\ts},\aA i,\jaG{\excl i})\label{eq:P_xxm_sa__LFM_transitionProbs}
\end{multline}
with
\begin{itemize}
\item $\Pr(\mflAT i{\ts+1}|\mfAT i{\ts},\mfnAT i{\ts+1},\aA i)$ representing
a product of CPTs of OLAFs $\mflI k$:
\begin{equation}
\Pr(\mflAT i{\ts+1}|\mfAT i{\ts},\mfnAT i{\ts+1},\aA i)=\prod_{k\in OLAF(i)}\Pr(\mflIT k{\ts+1}|\mfAT i{\ts},\mfAT i{\ts+1},\aA i)\label{eq:LFM:T:OLAFs}
\end{equation}
Note that although such individual factors $\mflIT k{\ts+1}$ can
have intra-stage dependencies on other OLAFs $\mflIT l{\ts+1}$ (i.e.,
 $\mflIT k{\ts+1}$ can depend on $\mfAT i{\ts+1}$ which can include
other OLAFs $\mflIT l{\ts+1}$), the product term $\Pr(\mflAT i{\ts+1}|\mfAT i{\ts},\mfnAT i{\ts+1},\aA i)$
itself can only have intra-stage dependencies on $\mfnAT i{\ts+1}$.
\footnote{Note that the intra-stage OLAFs $\mflAT i{\ts+1}$ will not appear
in the conditioning set (`behind the pipe') as they have all been
multiplied in (they are `before the pipe'). Since the 2DBN is non-cyclical
per definition, this does not present any problems. A more explicit
way of writing this is as follows. In general the OLAFs can now depend
on some NLAFs $\mfnAIT i{ISD}{\ts+1}$ that act as intra-stage dependencies:
\[
\Pr(\mflAT i{\ts+1}|\mfAT i{\ts},\mfnAIT i{ISD}{\ts+1},\aAT i{\ts})\defas\prod_{k\in OLAF(i)}\Pr(\mflIT k{\ts+1}|\mfAT i{\ts},\aAT i{\ts},\mfIT{ISD(k)}{\ts+1})
\]
with $\mfIT{ISD(k)}{\ts+1}$ denoting the intra-stage parents of $\mflIT k{\ts+1}$.
To reduce the notational burden, however, we will use the shorthands
from \eqref{eq:LFM:T:OLAFs}.}
\item $\Pr(\mfnAT i{\ts+1}|\mfAT i{\ts},\mflAT i{\ts+1},\nmfAT i{\ts},\nmfAT i{\ts+1},\aA i,\jaG{\excl i})$
the product of NLAF probabilities: 
\begin{equation}
\Pr(\mfnAT i{\ts+1}|\mfAT i{\ts},\mflAT i{\ts+1},\nmfAT i{\ts},\nmfAT i{\ts+1},\aA i,\jaG{\excl i})=\prod_{k\in NLAF(i)}\Pr(\mfnIT k{\ts+1}|\mfAT i{\ts},\mfAT i{\ts+1},\nmfAT i{\ts},\nmfAT i{\ts+1},\aA i,\jaG{\excl i})\label{eq:LFM:T:NLAFs}
\end{equation}
\item $\Pr(\nmfAT i{\ts+1}|\mfAT i{\ts},\mfAT i{\ts+1},\nmfAT i{\ts},\aA i,\jaG{\excl i})$
the product of probabilities of the NMFs~$\nmfI k$:
\begin{equation}
\Pr(\nmfAT i{\ts+1}|\mfAT i{\ts},\mfAT i{\ts+1},\nmfAT i{\ts},\aA i,\jaG{\excl i})=\prod_{k\in NMF(i)}\Pr(\nmfIT k{t+1}|\mfAT i{\ts},\mfAT i{\ts+1},\nmfAT i{\ts},\nmfAT i{\ts+1},\aA i,\jaG{\excl i})\label{eq:LFM:T:NMFs}
\end{equation}
\end{itemize}
\begin{figure}
\hfill{}\subfloat[Illustration of an abstract local-form model for agent~2. Factors
can be divided into non-modeled factors ($\sfacI1$), non-locally-affected
factors ($\sfacI2$, $\sfacI3$, shaded in this figure), and locally-affected
factors ($\sfacI4$). Also note that $\sfacI4$ is reward-relevant,
while $\sfacI2$ and $\sfacI3$ are observation-relevant factors.]{\begin{centering}
{\input{figs/frag_brm_2DBN.tex}\includegraphics[scale=0.42]{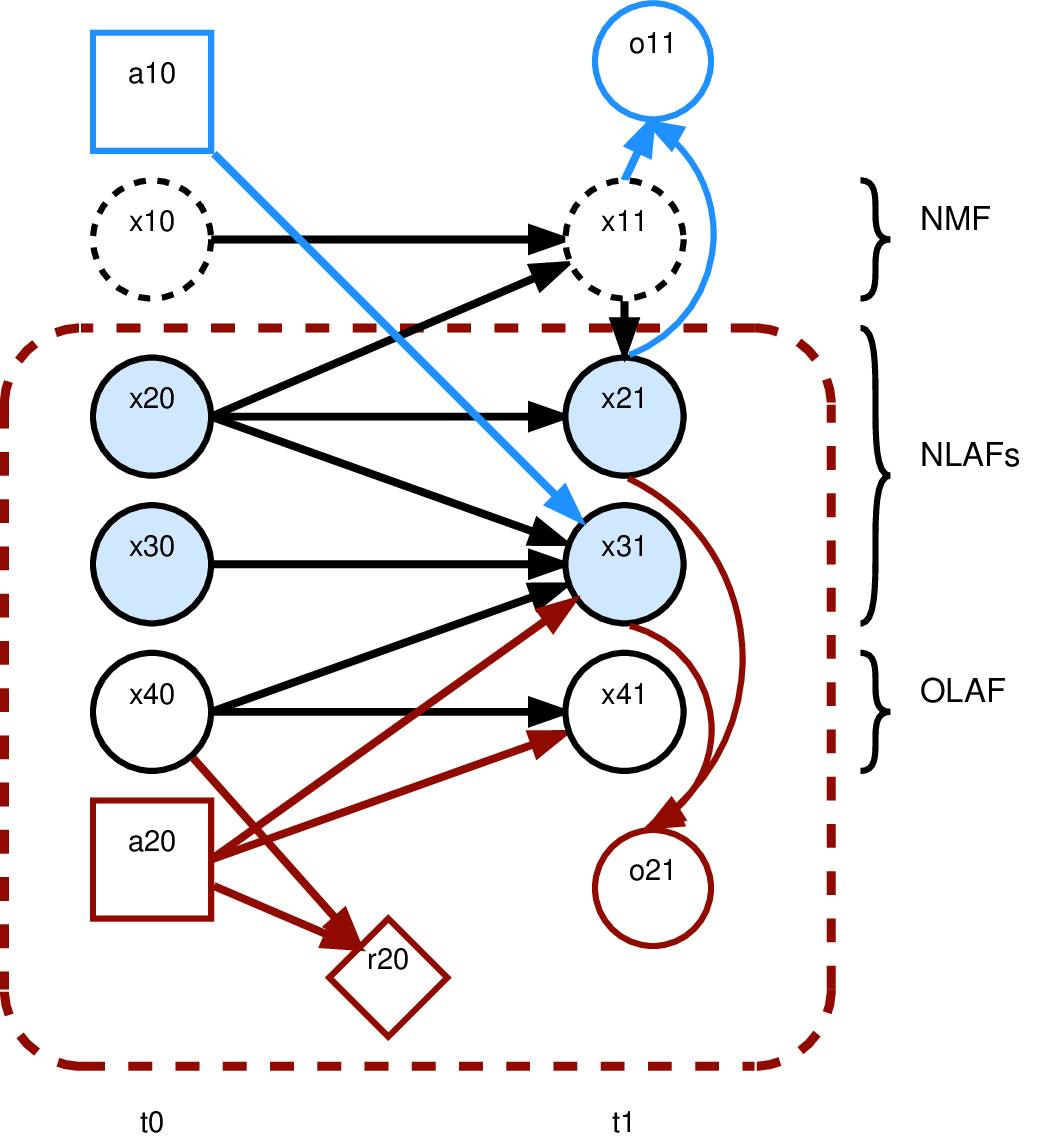}}
\par\end{centering}

\label{fig:LFM-factor-types}}\hfill{}\subfloat[Local-form model for agent~2 in the house search problem without
intra-stage dependencies. The `found' variable $f$ is the only NLAF
since it is affected by NMF $\argsAT l1\ts.$]{\begin{centering}
{\input{figs/frag_housesearch_2DBN.tex}\includegraphics[scale=0.42]{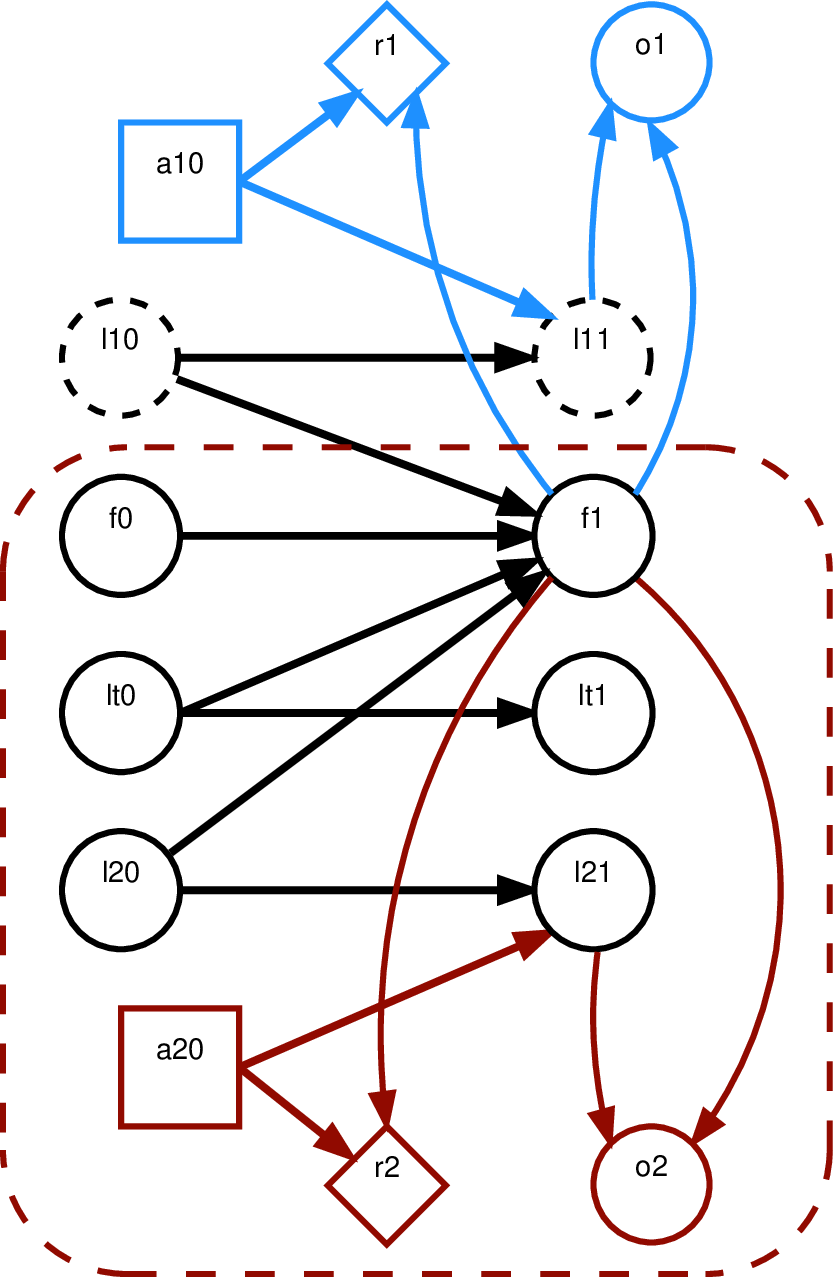}}
\par\end{centering}

\label{fig:LFM-housesearch--no-intrastage}}\hfill{}

\caption{Local-form models.}

\end{figure}

\paragraph{Value Function}

An LFM contains an fPOSG and as such best-response values for an agent~$i$
can be defined using the techniques discussed above in \prettyref{sec:GFBRM}.
In particular, we can just ignore the local state function and apply
the definition of Q-value \eqref{eq:Q(b,a)__GFBRM} with the previously
stated definitions of $R_{i}(\gbA i,\aA i)$ \eqref{eq:R(gfb,a)__GFM}
and $\Pr(\oAT i{\ts+1}|\gbA i,\aA i)$ \eqref{eq:P(o|gfb,a)__GFM}.

Clearly, however, we would like to now rewrite the value function
in a way that represents the local structure imposed by the LFM requirements
and exploits this for computational benefits. The former is possible:
for an LFM, we can indeed derive a expression for $R_{i}(\gbA i,\aA i)$
that is more local (see \app\ref{app:LFM:exp_reward}). 
\begin{equation}
R_{i}(\gbA i,\aAT i\ts)=\sum_{\mfAT i{\ts}}\sum_{\mfAT i{\ts+1}}R_{i}(\mfAT i{\ts},\aAT i\ts,\mfAT i{\ts+1})\Pr(\mfAT i{\ts},\mfAT i{\ts+1}|\gbA i,\aAT i\ts,\jpolG{\excl i}),\label{eq:R(gfb,a)__LFM}
\end{equation}
where (remember $\sT{\ts}=\langle\mfAT i{\ts},\nmfAT i{\ts}\rangle$)

\begin{equation}
\Pr(\mfAT i{\ts},\mfAT i{\ts+1}|\gbA i,\aAT i\ts,\jpolG{\excl i})\defas\sum_{\nmfAT i{\ts}}\sum_{\jaGT{\excl i}\ts}\Pr(\mfAT i{\ts+1}|\sT{\ts},\aAT i\ts,\jaGT{\excl i}\ts)\sum_{\aoHistAT{\excl i}{\ts}}\Pr(\jaGT{\excl i}\ts|\aoHistAT{\excl i}{\ts},\jpolG{\excl i})\gbA i(\sT{\ts},\aoHistAT{\excl i}{\ts}).\label{eq:P_xmxm_gfb}
\end{equation}
And, similarly, we can find a new, local, expression for the observation
probability (\app\ref{app:LFM:exp_obs}):

\begin{equation}
\Pr(\oAT i{\ts+1}|\gbA i,\aAT i\ts)=\sum_{\mfAT i{\ts+1}}\Pr(\oAT i{\ts+1}|\aAT i\ts,\mfAT i{\ts+1})\Pr(\mfAT i{\ts+1}|\gbA i,\aAT i\ts,\jpolG{\excl i})\label{eq:P(o|gfb,a)__LFM}
\end{equation}
where 
\begin{equation}
\Pr(\mfAT i{\ts+1}|\gbA i,\aAT i\ts)\defas\sum_{\sT{\ts}}\sum_{\jaGT{\excl i}\ts}\Pr(\mfAT i{\ts+1}|\sT{\ts},\aAT i\ts,\jaGT{\excl i}\ts)\sum_{\aoHistAT{\excl i}{\ts}}\Pr(\jaGT{\excl i}\ts|\aoHistAT{\excl i}{\ts}\jpolG{\excl i})\gbA i(\sT{\ts},\aoHistAT{\excl i}{\ts}).\label{eq:P_fm__gfb}
\end{equation}
These new definitions of $R_{i}(\gbA i,\aA i),\,\Pr(\oAT i{\ts+1}|\gbA i,\aA i)$
can be used directly in conjunction with the definition of Q-value~\eqref{eq:Q(b,a)__GFBRM}.

However, even though these definitions \eqref{eq:R(gfb,a)__LFM} and
\eqref{eq:P(o|gfb,a)__LFM} are local, they still depend on the global-form
belief and this must perform summations over full states $\sT{\ts}$
and histories of other agents $\aoHistAT{\excl i}{\ts}$ via \eqref{eq:P_xmxm_gfb}
and \eqref{eq:P_fm__gfb}, rendering them intractable for larger problems.
In the next section, we will investigate formulations that are based
on more local beliefs to try and overcome this computational hurdle.
Before jumping to this, we first state an observation: \begin{observation}
The presented definition of an LFM with multiple agents is a strict
generalization of a single agent problem. \end{observation}

While this is a simple observation, the upshot of this is that the
theory of influence-based abstraction that we will introduce in the
remainder of this paper also directly applies to single-agent settings.\footnote{We acknowledge Craig Boutilier, for pointing this out.}
Specifically, the formulas and results we will derive have more specific
forms for the single-agent case. We discuss relations to abstraction
methods for single-agent settings in \sect\ref{sec:Other-Forms-of-abstraction}.

\section{Influence-Based Abstraction}

\label{sec:IBA}

In the previous section we introduced the GFBRM, which could be used
to compute a best response against a fixed policy of other agents.
This model gives a straightforward way of formulating the problem
of computing a best-response. However, it is specified over the global
state and internal state of other agents (i.e., their AOHs), which
means that solving this model is computationally intractable. 

To provide a more localized perspective, the local-form POSG defines
for each agent a subset of factors that it should be concerned with.
However, even if the policies of the other agents are fixed, it is
not clear how an agent~$i$ can restrict its reasoning to its local
state $\mfA i$: the non-modeled factors will still affect the local
state transitions. Intuitively, we need to capture the \emph{influence}
that the non-modeled part of the problem exerts on the modeled part. 

In this section, we formalize this intuition. In particular, we treat
an LFM from the perspective of one agent and consider how that agent
is affected by the other agents and can compute a best response against
that `incoming' influence.\footnote{An agent also exerts `outgoing' influence on other agents, but this
is irrelevant for best response computation.}

In an attempt to avoid notation overload, we first present a formulation
without considering intra-stage connections. The general formulation
that can deal with such connections is given in \sect\ref{sec:IBA-with-IS-deps}.

\subsection{Definition of Influence}

\label{sec:IBA:def-of-influence}

As discussed in \sect\ref{sec:GFBRM}, when the other agents are
following a fixed policy, they can be regarded as part of the environment.
The resulting decision problem can be represented by the complete
unrolled DBN, as we saw in \fig\ref{fig:GFBRM}\vpageref{fig:GFBRM}.
In this figure, a node $\sfacT{\ts}$ is a different node than $\sfacT{\ts+1}$
and an edge at (emerging from) stage~$t$ is a different from the
edge at $t+1$ that corresponds to the same edge in the 2DBN. Given
this uniqueness of nodes and edges, we can define the `influence'
as follows.

\subsubsection{Influence Links, Sources and Destinations}

\label{sec:IBA:def-of-influence:links-sources-destinations}

Intuitively, the influence of other agents is the effect of those
edges leading into the agent's local problem. We say that every directed
edge from outside the local model (e.g., from an NMF or action of
another agent) to inside the local model (e.g., to a modeled state
factor, observation variable, or reward), is an \emph{influence link}
$\langle\ifsT\ts,\ifdT{\ts}\rangle$, where $\ifsT\ts$ is called
the \emph{influence source} and $\ifdT{\ts}$ is the \emph{influence
destination}. In this section, we will assume that influence links
traverse a stage of the process (i.e., that the influence source for
a destination $\ifdT{\ts}$ lies in the stage $\ts-1$), but since
we will also consider intra-stage influence links at a later point
in this document, to keep notation consistent, we label an entire
influence link with the stage-index of its destination.

For example, let us consider the \<housesearch>  problem's LFM shown
in \prettyref{fig:LFM-housesearch--no-intrastage}. It shows that
the link from $\argsAT{l}{1}{\ts}$, the location of agent 1, to the
`target found' variable $\argsT{f}{\ts+1}$ is an influence link,
such that we would write the link as $\left\langle \ifsT{\ts+1}=\argsAT{l}{1}{\ts},\ifdT{\ts+1}=\argsT{f}{\ts+1}\right\rangle $,
similarly $\left\langle \ifsT{\ts}=\argsAT{l}{1}{\ts-1},\ifdT{\ts}=\argsT{f}{\ts}\right\rangle $
would denote the influence link in the preceding time step.

Assuming no intra-stage influence links, an influence source $\ifsT{\ts}$
can be either an action $\aAT j{t-1}$ or non-modeled state factor
$\nmfT{\ts-1}$. We write $\ifsAT i{\ts}=\langle\nmfAT u{\ts-1},\jaGT u{\ts-1}\rangle$
for an instantiation of all influence sources exerting influence on
agent~$i$ at stage $t$. That is, in the case of multiple influence
links pointing to modeled factors in stage $\ts$, $\nmfAT u{\ts-1}$
denotes the (value of) influence sources that are state factors, while
$\jaGT u{\ts-1}$ corresponds to those influence sources that are
actions. For instance, in our \<housesearch>  example, $\nmfAT u{\ts-1}=\{\argsAT{l}{1}{\ts-1}\},$
while $\jaGT u{\ts-1}=\emptyset$ since there are no actions that
are influence sources. We write $\aoHistGT u{\ts-1}$ for the AOHs
of those other agents whose action is an influence source (i.e., $\aoHistGT u{\ts-1}$
and $\jaGT u{\ts-1}$ involve the same agents) .

In general, an influence destination can be either a (per definition
non-locally-affected) modeled factor $\mfnT{\ts}$, an observation
variable $\oAT i{\ts}$, or a local reward node $\argsAT{\R}{i}{\ts}$.
But Definition~\eqref{dfn:lfm} requires reward- or observation-relevant
factors to be included in the local state; effectively we restrict
ourselves to the setting where the influence destination is an NLAF.
This restriction is without loss in generality: because we will introduce
(in \sect\ref{sec:IBA-with-IS-deps}) the machinery to deal with
\emph{intra-stage }influence links, influences on observations and
rewards can easily be dealt with by introducing a `dummy' NLAF that
acts as a proxy for the observation or reward.\footnote{E.g., to deal with an observation destination, we can transform the
observation $\oA i$ to a state factor $\sfacI{o}$ and introduce
a new observation variable that has a deterministic CPT depending
only on $\sfacI{o}$.} A similar construction can be used to deal with settings where actions
of other agents would directly influence the observations or rewards
of the agent under concern. As such, the capability of dealing with
such intra-stage dependencies is critical for the applicability of
the theory of influence-based abstraction.

\subsubsection{Sufficient Information to Predict Influences: D-Separating Sets}

\label{sec:IBA:def-of-influence:dset}

If agent~$i$ would in advance know the value of its influence sources
at different time steps, it could easily compute its best response
by making use of only this knowledge and its local model. For instance,
if in the \<housesearch> example of \fig\ref{fig:LFM-housesearch--no-intrastage}\vpageref{fig:LFM-housesearch--no-intrastage},
we would in advance perfectly know the location of agent~1 at each
timestep and thus know the sequence of values for $\argsAT l{1}{\ts}$,
we could decouple the local problem by just looking at the appropriate
slices of the CPT of $\argsT{f}{\ts}$.

Of course, this is in general not possible, since the influence sources
are random variables. However, the influence exerted on agent~$i$
can be captured if we know the probability distribution over their
values. That is, in order to predict the probability of some $\mfnAT i{\ts+1}$
(i.e., an influence destination) agent~$i$ only cares about the
following marginal probability 
\begin{equation}
\sum_{\ifsAT i{\ts+1}}\Pr(\mfnAT i{\ts+1}|\mfAT i{\ts},\aAT i{\ts},\ifsAT i{\ts+1})\Pr(\ifsAT i{\ts+1}|\ldots),\label{eq:predicting-an-NLAF-via-marginal-prob}
\end{equation}
where the dots ($\ldots$) indicate any information that agent~$i$
needs to predict the probability of the values of the influence sources
as accurately as possible. Moreover, since these probabilities will
be used to plan a best response, correlations between influence sources
and local states are important. This unfortunately means that in general,
we might need to condition $\Pr(\ifsAT i{\ts+1}|...)$ on the entire
history of actions, observations and and local states.

Fortunately, it turns out that in many cases we can find substantially
more compact representations of the conditional probability of $\ifsAT i{\ts+1}$,
by making use of the concept of \emph{d-separation} in graphical models
\citep{Bishop06book,KollerFriedman09}. In particular, when two nodes
$A,B$ in a Bayesian network are d-separated given some of subsets
$\dsetA{}$ of evidence nodes, then $A$ and $B$ are conditionally
independent given $\dsetA{}$, which means that $\Pr(A|D,B)=\Pr(A|D)$
and vice versa. Whether nodes are d-separated can be easily checked,
by applying a small set of rules on the graph \citep[chapter 8]{Bishop06book}.

Now, we can define the influence as a conditional probability distribution
over $\ifsAT i{\ts+1}$, given a d-separating set. Specifically, let
$\dsetAT i{t+1}$ be a subset of variables (possibly including state
factors and actions) in the local problem of agent~$i$ at stages
$0,\dots,t$, 

\begin{definition}[D-separating set]\label{dfn:d-separating-set-NO-IS}
$\dsetAT i{t+1}$ is a \emph{d-separating set} \emph{for agent~$i$'s
influence at stage $t+1$} if and only if it d-separates $\nmfAT u{\ts},\aoHistGT u{\ts}$
from $\mfAT it,\aoHistAT it$. That is, if: 
\begin{equation}
\forall_{\nmfAT u{\ts},\aoHistAT u{\ts}}\qquad\Pr(\nmfAT u{\ts},\aoHistAT u{\ts}|\mfAT it,\aoHistAT it,\dsetAT i{\ts+1},\bO,\jpolG{\excl i})=\Pr(\nmfAT u{\ts},\aoHistAT u{\ts}|\dsetAT i{\ts+1},\bO,\jpolG{\excl i}).\label{eq:dset-def}
\end{equation}
\end{definition}

This definition implies that remembering more than $\dsetAT i{\ts+1}$
is not useful for predicting $\nmfAT u{\ts},\aoHistAT u{\ts}$ and
hence for predicting $\ifsAT i{\ts+1}=\langle\nmfAT u{\ts},\jaGT u{\ts}\rangle$.
Given their policies, the actions of other agents only depend on their
AOHs. We note that when the other agents use simpler (e.g., memoryless)
policies, one might not need to predict the full action observation
history for agents whose actions are influence sources. Instead we
will only need to predict relevant part, denoted $\rho(\aoHistAT u{\ts})$.
Similarly, there might be a sufficient statistic $\dsetCompF$ that
summarizes $\dsetAT i{t+1}$ and still is enough to provide the conditional
independence. In such case we would only need 
\begin{equation}
\forall_{\nmfAT u{\ts},\aoHistAT u{\ts}}\qquad\Pr(\nmfAT u{\ts},\rho(\aoHistAT u{\ts})|\mfAT it,\aoHistAT it,\dsetCompF(\dsetAT i{\ts+1}),\bO,\jpolG{\excl i})=\Pr(\nmfAT u{\ts},\rho(\aoHistAT u{\ts})|\dsetCompF(\dsetAT i{\ts+1}),\bO,\jpolG{\excl i}).
\end{equation}
To avoid a further burden on notation, we will not explicitly consider
these special cases, and in our description assume that we condition
on the values of the variables in the d-separating set. However, we
will see examples of such more compact description of the information
needed to predict the influence sources.

Deciding on $\dsetAT i{\ts+1}$ needs to be done in advance to compute
the influence. When the d-separating set is compressed, $\sigma(\dsetAT i{\ts+1}),$
this will typically involve input by the human designer. However,
we note that efficient algorithms are known to compute a minimal d-separating
set~\citep{Acid96UAI,Tian98finding,Zander20UAI19} in cases where
this would be infeasible to do by hand.

\begin{example}\fig\ref{fig:dset-housesearch} illustrates a d-separating
set $\dsetAT i3$ for agent $i=2$ in \<housesearch>. It shows that,
in order to accurately compute the probability of influence source
$\argsAT l12$, agent 2 needs to condition on $\argsT{f}{0:2}$, the
history of the found variable, as well as the histories of the location
of the target $l_{tgt}$ and its own location~$\argsA l2$. This
dependence on the history in general leads to large conditioning sets,
but in many cases the history can be represented more compactly. For
instance, in \<housesearch> the `found' variable can only switch
on (not off) which means that its history $\argsT{f}{0:\ts}$ can
be summarized compactly. In cases where the target is static the same
holds for $\argsAT{l}{tgt}{0:\ts}$. \end{example}

\begin{figure}
\begin{centering}
\input{figs/frag_housesearch_unrolled.tex}\includegraphics[scale=0.4]{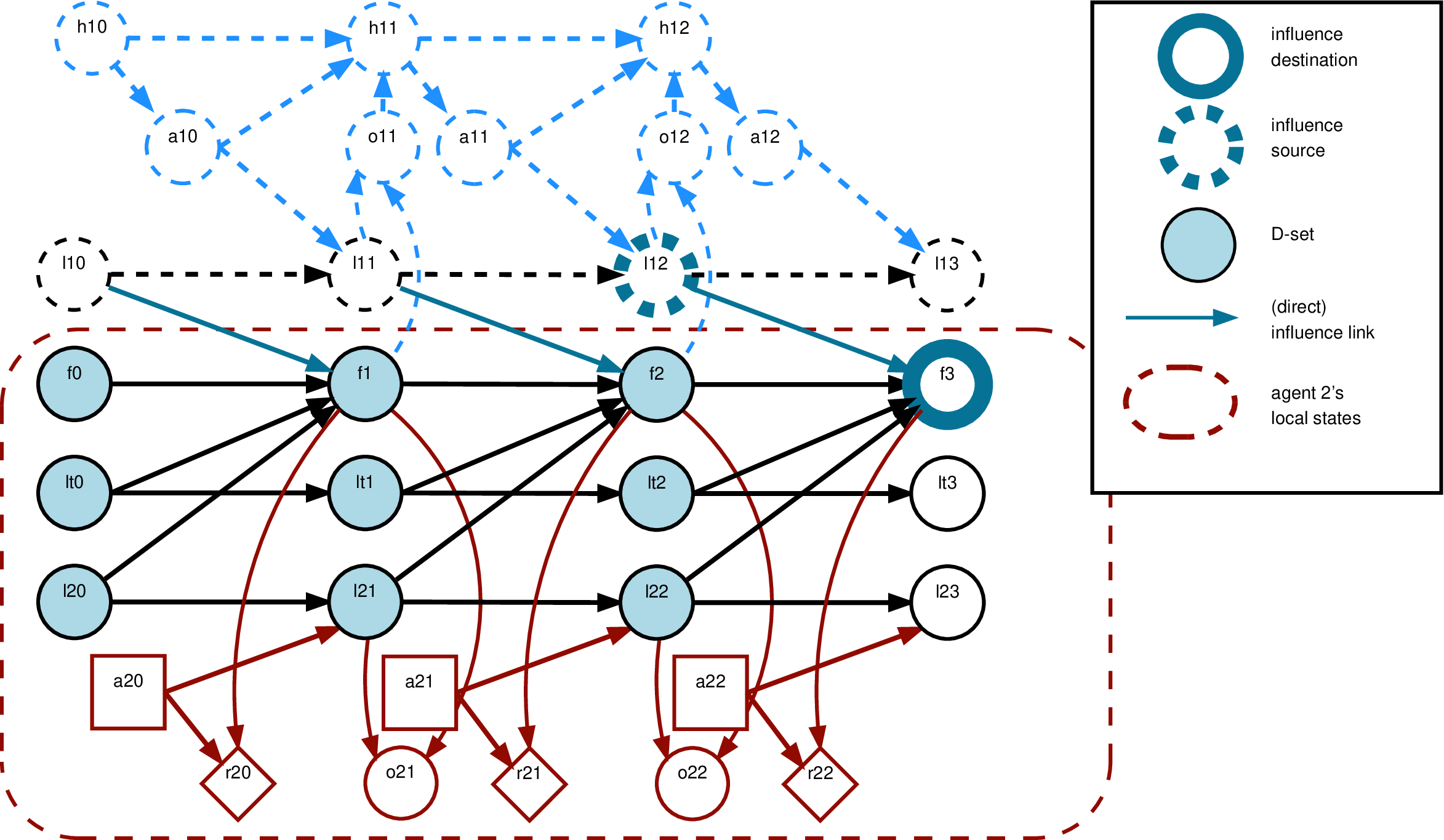}
\par\end{centering}
\caption{Illustration of the incoming influence on protagonist agent $i=2$
in \<housesearch> at stage $\ts=3$. $\argsT f3$ is the only influence
destination, with influence source $\nmfAT u2=\argsAT l12$ (i.e.,
$\ifsAT i3=\left\langle \argsAT l12\right\rangle $). The shaded nodes
indicate the d-separating set $\dsetAT i3$, which, in accordance
with \eqref{eq:dset-def}, d-separates the influence source $\argsAT l12$,
from agent 2's AOH $\aoHistAT i{\ts}$ and possibly remaining other
local variables $\mfAT i{\ts}$ (in this case there are no such variables,
but one could imagine adding a battery life variable for agent 2).}

\label{fig:dset-housesearch}
\end{figure}

\begin{example}\fig\ref{fig:dset-rover} describes a variant of
the \problemName{planetary exploration} domain~\citep{Witwicki10ICAPS}.
Here agent~2 is a mars rover which is tasked with navigating to some
goal. Agent~1 is a satellite which can aid the rover by planning
a path, but this will use up computational resources and battery power
modeled by $\argsAT{bt}{1}\ts$ (which it may want to use to support
other rovers too, for instance). In the figure this is illustrated
by the fact that the action of agent~1 $\aA1\in\{NOOP,PLAN\}$ (which
now is the influence source) determines if there is a plan available
for agent~2, modeled by a binary variable~$pl$ (which is the influence
destination). In this example, the d-separating set only contains
this variable $pl$. Again its history can be compactly summarized:
as having the plan can only turn true, we can just store the time
(if any) at which $pl$ was switched to true.

\end{example}

\begin{figure}
\begin{centering}
\input{figs/frag_rover_unrolled.tex}\includegraphics[scale=0.4]{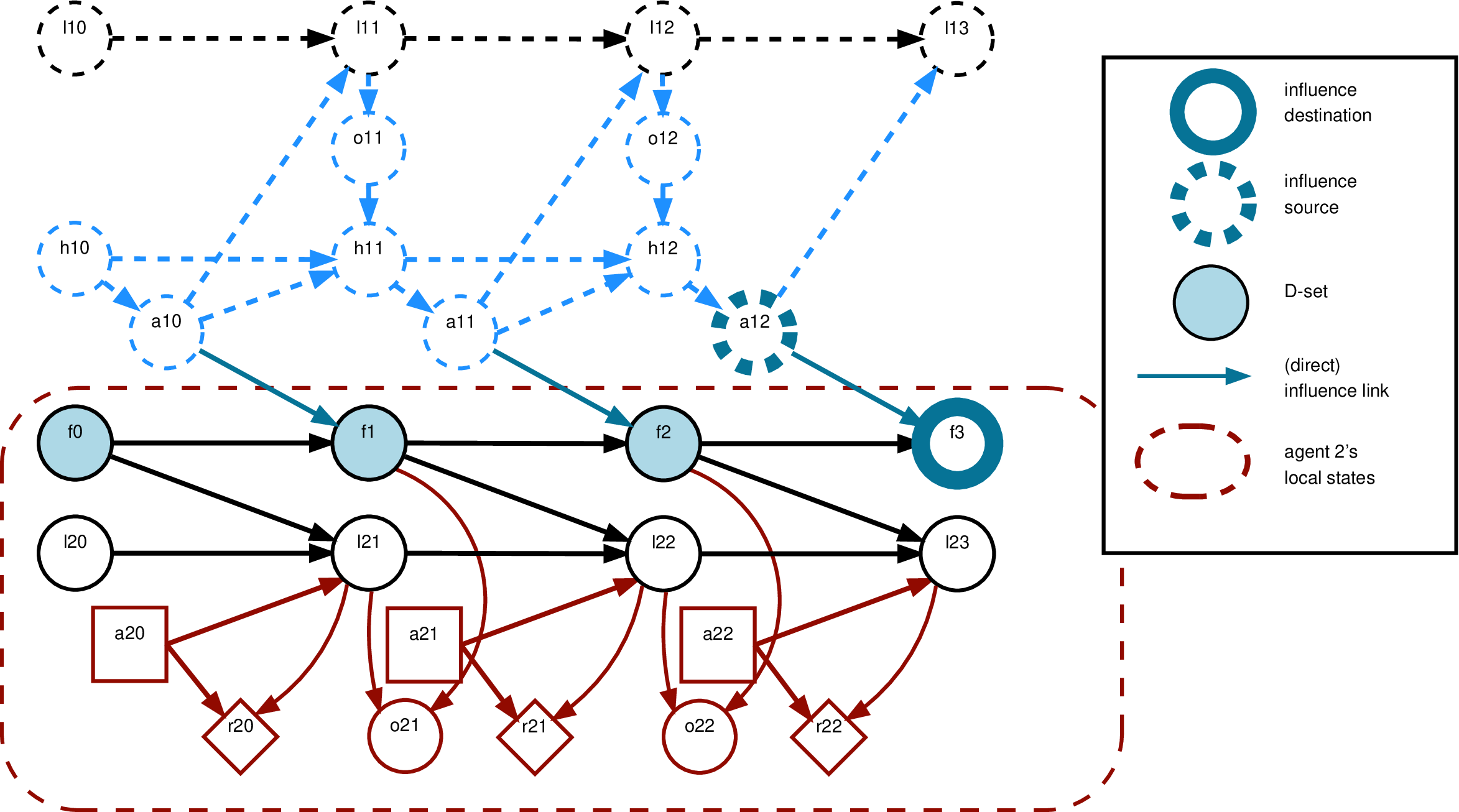}
\par\end{centering}
\caption{Illustration of the influence experienced by the mars rover (agent
$i=2$) at stage $\ts=3$ in the \problemName{planetary exploration} domain.
If the satellite (agent 1) computes and transmits a plan (\emph{pl)},
the rover can more effectively navigate from that point onward.}

\label{fig:dset-rover}
\end{figure}

\subsubsection{The Influence Exerted on Agent $i$}

Given the above machinery, we can now state our definition of influence:

\begin{definition}[Incoming Influence] \label{dfn:The-incoming-influence}
The \emph{incoming influence} at stage~$\ts+1$, denoted $\ifpiAT i{\ts+1}(\jpolG{\excl i})$,
is a conditional probability distribution over values of the influence
sources~$\ifsAT i{\ts+1}$: 
\begin{equation}
\iffunc(\ifsAT i{\ts+1}|\dsetAT i{\ts+1})\defas\sum_{\aoHistGT u{\ts}}\Pr(\jaGT ut|\aoHistGT u{\ts})\Pr(\nmfAT u{\ts},\aoHistGT u{\ts}|\dsetAT i{\ts+1},\bO,\jpolG{\excl i}).\label{eq:IncomingInfluence__non-isd}
\end{equation}
\end{definition}In order to predict $\jaGT ut$ (the `influence
source actions') we need to predict the action-observation histories
of the corresponding agents $\aoHistGT u{\ts}$, but otherwise these
histories are not needed and can thus be marginalized out. Note that,
to reduce notational burden we drop arguments that can be inferred,
such as $\bO,\jpolG{\excl i}$. That is, $\iffunc(\ifsAT i{\ts+1}|\dsetAT i{\ts+1})$
is shorthand for $\ifpiAT i{\ts+1}(\ifsAT i{\ts+1}|\dsetAT i{\ts+1},\bO,\jpolG{\excl i})$.
In cases where we want to refer to this distribution as a whole, we
will write $\ifpiAT i{\ts+1}(\jpolG{\excl i})$, or use the shorthand
$\ifpiAT i{\ts+1}$.

We will also say that this is the influence \emph{exerted }on agent~$i$
at stage $\ts$\emph{ }or \emph{experienced} by agent~$i$ at stage
$\ts+1$. So far, these notions coincide, but when we consider intra-stage
connections in the next section, we will discriminate between these
concepts.

Finally, we are in position to specify the complete influence on agent~$i$:

\begin{definition}[incoming influence point]\label{dfn:The-incoming-influence-point}
An \emph{incoming influence point} $\ifpiA{i}(\jpolG{\excl i})$ for
agent~$i$, specifies the incoming influences for all stages $\ifpiA{i}(\jpolG{\excl i})=\left(\ifpiAT i1(\jpolG{\excl i}),\dots,\ifpiAT i{\h}(\jpolG{\excl i})\right)$.
\end{definition}

As we will see in the remainder of this paper, an influence point
contains all the information about the non-modeled part of the problem
that agent~$i$ needs to compute a best response `locally', i.e.,
only using its local model and that influence point. This can bring
computational benefits for instance when there would be changes in
the local model that require repeatedly performing planning, or in
cases where the influence point can be computed easily. This form
of influence-based abstraction, however, is not providing a free lunch~\citep{Wolpert95TR}:
in general computing the incoming influences \eqref{eq:IncomingInfluence__non-isd}
for the different stages comprise a set of challenging inference problems.
Fortunately, traction can still be gained in many special cases of
problems identified in past work \citep{Becker03AAMAS,Becker04AAMAS,Varakantham09ICAPS,Petrik09JAIR,Witwicki10AAMAS,Witwicki10ICAPS,Witwicki11AAMAS,Velagapudi11AAMAS,Witwicki11PhD,Witwicki12AAMAS,Oliehoek12AAAI_IBA},
and IBA gives a unified perspective on these. Moreover, it can be
used as tool to identify further special cases that allow for efficient
solution, such as the class of ND-POMDPs discussed in \sect\ref{sec:sub-classes-with-Compact-representations}.
Given the potential benefits of using influence representations \citep{Witwicki12AAMAS},
such future search for special cases of problems that allow for compact
influence specifications together with the inference algorithms that
efficiently compute these is an important line of research. Our definition
of influence in this section provides the general framework in which
these special cases should be sought.

\subsection{The Influence-Augmented Local Model (IALM)}

Given the above definition of influence, we can now define a smaller
\emph{local} model for our protagonist agent~$i$. The main idea
is that given an incoming influence point, agent~$i$ no longer needs
to reason over the non-modeled part of the problem. Instead, it can
use the influence to compute marginal probabilities as expressed by
\eqref{eq:predicting-an-NLAF-via-marginal-prob}, and this will allow
it to compute an exact best-response. 

In this section, we will first investigate a single NLAF and how the
influence on it can be incorporated. Then we move to talk about the
case where multiple variables in the local state $\LSF(\agentI i)$
are non-locally affected. Then we proceed to the formal definition
of the IALM, and how it can be solved.

\subsubsection{Induced CPTs}

\label{sec:inducedCPT}

In the case of a single influence destination, we can interpret \eqref{eq:predicting-an-NLAF-via-marginal-prob}
as constructing a new `influence-induced' CPT:

\begin{definition}[Induced CPT]Let $\mfnT{\ts+1}$ be an influence
destination, and $\ifsT{\ts+1}$ (the instantiation of) the corresponding
influence sources. Given the influence $\ifpiAT i{\ts+1}(\jpolG{\excl i})$,
and its d-separating set $\dsetAT i{\ts+1}$, we define the \emph{induced
CPT }for\emph{ }$\mfnT{\ts+1}$ as the CPT that has probabilities:
\begin{equation}
p_{\ifpiAT i{\ts+1}}(\mfnT{\ts+1}|\mfAT i{\ts},\dsetAT i{\ts+1},\aA i)=\sum_{\ifsAT i{\ts+1}=\left\langle \nmfAT u{\ts},\jaG u\right\rangle }\Pr(\mfnT{\ts+1}|\mfAT i{\ts},\aA i,\ifsAT i{\ts+1})\iffunc(\ifsAT i{\ts+1}|\dsetAT i{\ts+1})\label{eq:induced-CPT}
\end{equation}
\end{definition}

It is important to note that an induced CPT is specified purely in
\emph{local }terms, i.e., making use of variables that are modeled
by our protagonist agent~$i$. Therefore, the basic idea is that
we can now define a smaller \emph{local }model\textemdash which we
will call the \emph{Influence-Augmented Local Model (IALM)}\textemdash by
replacing the CPTs of influence destinations (i.e., NLAFs) by induced
CPTs.

\subsubsection{Dealing With Multiple NLAFs}

In case that there are multiple NLAFs, i.e., multiple variables $\mfnT{\ts+1}$
in the local state space $\LSF(\agentI i)$ that are affected non-locally
at the same stage $\ts+1$, the story is slightly more involved, since
we need to deal with their correlations.

Ideally, we would want to treat induced CPTs in the same way as normal
CPTs; that is, we would represent the joint probability of NLAFs as
a the product of induced CPTs:
\begin{equation}
\Pr(\mfnAT i{\ts+1}|\mfAT i{\ts},\dsetAT i{\ts+1},\aA i,\ifpiAT i{\ts+1})=\prod_{k\in NLAF(i)}p_{\ifpiAT i{\ts+1}}(\mfnIT k{\ts+1}|\mfAT i{\ts},\dsetAT i{\ts+1},\aA i).\label{eq:P(NLAFS)__product-of-ICPTS}
\end{equation}
However, in general this is not possible since the different $\mfnIT k{\ts+1}$
are correlated via any common influence sources. That is, in general
the probability is given by:

\begin{equation}
\Pr(\mfnAT i{\ts+1}|\mfAT i{\ts},\dsetAT i{\ts+1},\aA i,\ifpiAT i{\ts+1})=\sum_{\ifsAT i{\ts+1}=\left\langle \nmfAT u{\ts},\jaG u\right\rangle }\iffunc(\ifsAT i{\ts+1}|\dsetAT i{\ts+1})\prod_{k\in NLAF(i)}\Pr(\mfnIT k{\ts+1}|\mfAT i{\ts},\aA i,\ifsAT i{\ts+1})\label{eq:P(NLAFS)__general}
\end{equation}

Of course, in certain cases a factorization as induced CPTs is possible.
The above equations directly make clear when this is the case. 

\begin{proposition}

If each NLAF $\mfnIT k{\ts+1}$ has its own influence sources $\ifsIT k{\ts+1}$
(and these do not overlap), and if these sources are conditionally
independent given $\dsetAT i{\ts+1}$:
\[
\iffunc(\ifsAT i{\ts+1}|\dsetAT i{\ts+1})=\prod_{k\in NLAF(i)}\iffunc(\ifsIT k{\ts+1}|\dsetAT i{\ts+1}),
\]
then the joint probability of NLAFs factorizes as the product of induced
CPTs as shown in \eqref{eq:P(NLAFS)__product-of-ICPTS}.\end{proposition}

\begin{proof} Under stated conditions, we can rewrite as follows:\footnote{For the last step of this proof, to see why we can swap summation
and product, note that the term on the third line has the form $\sum_{\left\langle a_{1},\dots a_{k},\dots,a_{K}\right\rangle }\prod_{k=1}^{K}a_{k}b_{k}$.
We take $K=2$ for the example and get:
\[
\sum_{\left\langle a_{1}a_{2}\right\rangle }a_{1}b_{1}a_{2}b_{2}=\sum_{a_{1}}a_{1}b_{1}\sum_{a_{2}}a_{2}b_{2}=\left[\sum_{a_{1}}a_{1}b_{1}\right]\left[\sum_{a_{2}}a_{2}b_{2}\right]=\prod_{k=1}^{K}\sum_{a_{k}}a_{k}b_{k}.
\]
}
\begin{align*}
\eqref{eq:P(NLAFS)__general}= & \sum_{\ifsAT i{\ts+1}=\left\langle \dots,\ifsIT k{\ts+1},\dots\right\rangle }\iffunc(\ifsAT i{\ts+1}|\dsetAT i{\ts+1})\prod_{k\in NLAF(i)}\Pr(\mfnIT k{\ts+1}|\mfAT i{\ts},\aA i,\ifsIT k{\ts+1})\\
= & \sum_{\ifsAT i{\ts+1}=\left\langle \dots,\ifsIT k{\ts+1},\dots\right\rangle }\left[\prod_{k\in NLAF(i)}\iffunc(\ifsIT k{\ts+1}|\dsetAT i{\ts+1})\right]\prod_{k\in NLAF(i)}\Pr(\mfnIT k{\ts+1}|\mfAT i{\ts},\aA i,\ifsIT k{\ts+1})\\
= & \sum_{\ifsAT i{\ts+1}=\left\langle \dots,\ifsIT k{\ts+1},\dots\right\rangle }\prod_{k\in NLAF(i)}\iffunc(\ifsIT k{\ts+1}|\dsetAT i{\ts+1})\Pr(\mfnIT k{\ts+1}|\mfAT i{\ts},\aA i,\ifsIT k{\ts+1})\\
= & \prod_{k\in NLAF(i)}\sum_{\ifsIT k{\ts+1}}\iffunc(\ifsIT k{\ts+1}|\dsetAT i{\ts+1})\Pr(\mfnIT k{\ts+1}|\mfAT i{\ts},\aA i,\ifsIT k{\ts+1})=\eqref{eq:P(NLAFS)__product-of-ICPTS}\qedhere
\end{align*}
\end{proof}

\subsubsection{The IALM: A Formal Model to Incorporate Influence}

Here we formally define the IALM, which is a non-stationary POMDP,
since at every stage the influence destinations can be influenced
in a different manner.

\begin{definition}[IALM]

\label{dnf:IALM} Given an LFM, $\mathcal{M}^{LFM}$, and profile
of policies for other agents~$\jpolG{\excl i}$, an \emph{Influence-Augmented
Local Model (IALM) }for agent~$i$ is a POMDP $\mathcal{M}_{i}^{IALM}(\mathcal{M}^{LFM},\jpolG{\excl i})=\left\langle \bar{\sS},\aAS{i},\bar{T}_{i},\bar{R}_{i},\oAS{i},\bar{O}_{i},\hor,\lbAT i0\right\rangle $,
where
\begin{itemize}
\item $\bar{\mathcal{S}}$ is the set of augmented states $\sAugAT it=\langle\mfAT i{\ts},\dsetAT i{\ts+1}\rangle$
that specify an underlying local state of the POSG, as well as the
d-separating set $\dsetAT i{\ts+1}$ for the next-stage influences.
Note that $\dsetAT i{\ts+1}$ typically needs to include certain state
factors for stage~$\ts$, such that $\mfAT i{\ts}$ and $\dsetAT i{\ts+1}$
both will specify such variables. This is no problem, as long as they
specify consistent assignments; we define $\bar{\mathcal{S}}$ to
be the set of states that are consistent.
\item $\aAS{i},\oAS{i}$ are the (unmodified) sets of actions and observations
for agent~$i$.
\item The transition function $\bar{T}_{i}(\sAugAT i{t+1}|\sAugAT it,\aAT i{\ts})$
which we will discuss in detail shortly.
\item The observation function $\bar{O}_{i}(\oAT i{\ts+1}|\aAT i{\ts},\sAugAT i{t+1})=O(\oAT i{\ts+1}|\aAT i{\ts},\mfAT i{\ts+1})$,
since agent $i$'s observations only depend on its local state (cf.
\dfn\ref{dfn:lfm}, property 1).
\item The reward function $\bar{\RA i}(\sAugAT it,\aAT i{\ts},\sAugAT i{t+1})=\RA i(\mfAT i{\ts},\aAT i{\ts},\mfAT i{\ts+1})$,
since agent $i$'s rewards only depend on its local state (cf. \dfn\ref{dfn:lfm},
property 2).
\item $\hor$ is the unmodified horizon.
\item $\lbAT i0$ is the initial state distribution, which is a \emph{local-form
belief. }It is a distribution over augmented states $\sAugAT i0=\langle\mfAT i0,\dsetAT i1\rangle$.
Since for the first stage $\dsetAT i1$ can only contain elements
from $\mfAT i0$, it can trivially be constructed from a probability
distribution over $\mfAT i0$, and such a distribution can be constructed
from $\bO$, as we discuss in a bit more detail below.
\end{itemize}
\end{definition}

In defining $\bar{T}_{i}$ and $\lbAT i0$, a few subtleties arise
that we now discuss.

\paragraph{Transition Probabilities }

Clearly, the IALM's transition probabilities should express 
\begin{equation}
\bar{T}_{i}(\sAugAT i{t+1}|\sAugAT it,\aAT i{\ts})\defas\Pr(\langle\mfAT i{\ts+1},\dsetAT i{\ts+2}\rangle|\langle\mfAT i{\ts},\dsetAT i{\ts+1}\rangle,\aAT i{\ts},\ifpiAT i{\ts+1}).
\end{equation}
For such probabilities to be specified, we need some further requirements
on the d-separating sets. In particular, we require that (the instantiation
of) $\dsetAT i{\ts+2}$ is fully specified by $\mfAT i{\ts},\aAT i{\ts},\mfAT i{\ts+1}$
and $\dsetAT i{\ts+1}$.

\begin{definition}[d-set update function] The d-set update function
is a function $d$ that takes the previous-stage d-separating set
and the latest transition, and that returns the next d-separating
set:
\[
\dsetAT i{\ts+2}=\dsetUF(\mfAT i{\ts},\aAT i{\ts},\mfAT i{\ts+1},\dsetAT i{\ts+1}).
\]
In other words: $\dsetUF$ `selects' the variables from $\mfAT i{\ts},\aAT i{\ts},\mfAT i{\ts+1},\dsetAT i{\ts+1}$
such that it forms the next d-separating set.\footnote{Note, that if further compression by means of a statistic $\dsetCompF$
is employed (cf. the discussion under \dfn\ref{dfn:d-separating-set-NO-IS})
than the update function should work on these statistics $\dsetCompF(\dsetAT i{\ts+2})=\dsetUF(\mfAT i{\ts},\aAT i{\ts},\mfAT i{\ts+1},\dsetCompF(\dsetAT i{\ts+1}))$.}\end{definition}

Given a d-set update function we can write:
\[
\Pr(\langle\mfAT i{\ts+1},\dsetAT i{\ts+2}\rangle|\langle\mfAT i{\ts},\dsetAT i{\ts+1}\rangle,\aAT i{\ts},\ifpiAT i{\ts+1})=\Pr(\mfAT i{\ts+1}|\langle\mfAT i{\ts},\dsetAT i{\ts+1}\rangle,\aAT i{\ts},\ifpiAT i{\ts+1})\KroD{\dsetAT i{\ts+2}}{\dsetUF(\mfAT i{\ts},\aAT i{\ts},\mfAT i{\ts+1},\dsetAT i{\ts+1})},
\]
where $\KroD{\cdot}{\cdot}$ denotes the Kronecker delta function.

A typical way to fulfill the requirement that $\dsetAT i{\ts+2}$
is fully specified by $\mfAT i{\ts},\aAT i{\ts},\mfAT i{\ts+1}$ and
$\dsetAT i{\ts+1}$ is to assume that the d-separating sets for all
stages are chosen as the history of the same subset $\dsetA i\subseteq\LSF(i)$
of modeled features. 

\begin{example}Looking at \prettyref{fig:dset-housesearch}, the
d-separating set $\dsetAT23$ for predicting $\argsT f3$ is given
by the history of the `found', `location of target' and `location
of agent~2' variables. So we can write $\dsetA2=\left\{ f,\argsA l{tgt},\argsA l{2}\right\} $,
and define $\dsetAT23$ to be its history at stage $\ts=2$: $\dsetAT23=\dHistAT22$.\end{example}

The probabilities $\Pr(\mfAT i{\ts+1}|\langle\mfAT i{\ts},\dsetAT i{\ts+1}\rangle,\aAT i{\ts})$
are now factored as the product of the CPTs of the OLAFs and the induced
probabilities for the NLAFs:
\begin{multline}
\bar{T}_{i}(\sAugAT i{t+1}|\sAugAT it,\aAT i{\ts})\defas\Pr(\mfAT i{\ts+1}|\langle\mfAT i{\ts},\dsetAT i{\ts+1}\rangle,\aAT i{\ts},\ifpiAT i{\ts+1})\KroD{\dsetAT i{\ts+2}}{\dsetUF(\mfAT i{\ts},\aAT i{\ts},\mfAT i{\ts+1},\dsetAT i{\ts+1})},\\
=\Pr(\mfnAT i{\ts+1}|\langle\mfAT i{\ts},\dsetAT i{\ts+1}\rangle,\aAT i{\ts},\ifpiAT i{\ts+1})\Pr(\mflAT i{\ts+1}|\mfAT i{\ts},\mfnAT i{\ts+1},\aA i)\KroD{\dsetAT i{\ts+2}}{\dsetUF(\mfAT i{\ts},\aAT i{\ts},\mfAT i{\ts+1},\dsetAT i{\ts+1})}.\label{eq:IALM-T}
\end{multline}
Here the first term is given by \eqref{eq:P(NLAFS)__general} and
the second term is given by \eqref{eq:LFM:T:OLAFs}.\footnote{Note that, even though we have not dealt with intra-stage dependencies
(ISDs) in the description of influences in this section, we refer
back to the term $\Pr(\mflAT i{\ts+1}|\mfAT i{\ts},\mfnAT i{\ts+1},\aA i)$
from section 3 which does allow for ISDs from NLAFs to OLAFs. This
will allow us to make only minimal changes to the definition of $\bar{T}_{i}$
when we do deal with ISDs in \sect\ref{sec:IBA-with-IS-deps}.}

\paragraph{Initial Local State Distribution}

\begin{figure}
\input{figs/frag_initial-belief.tex}

\hfill{}\subfloat[The Bayesian network $G^{0}$ representing the initial belief.]{\begin{centering}
\begin{minipage}[b][3.7cm][t]{5cm}%
\begin{center}
\includegraphics[scale=0.4]{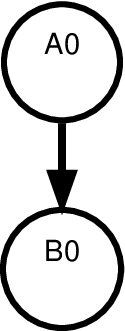}
\par\end{center}%
\end{minipage}
\par\end{centering}
}\hfill{}\subfloat[The 2DBN (a conditional Bayesian network) $G^{\rightarrow}$ representing
the transition and observation probabilities.]{\begin{centering}
\begin{minipage}[b][3.7cm][t]{7cm}%
\begin{center}
\includegraphics[scale=0.4]{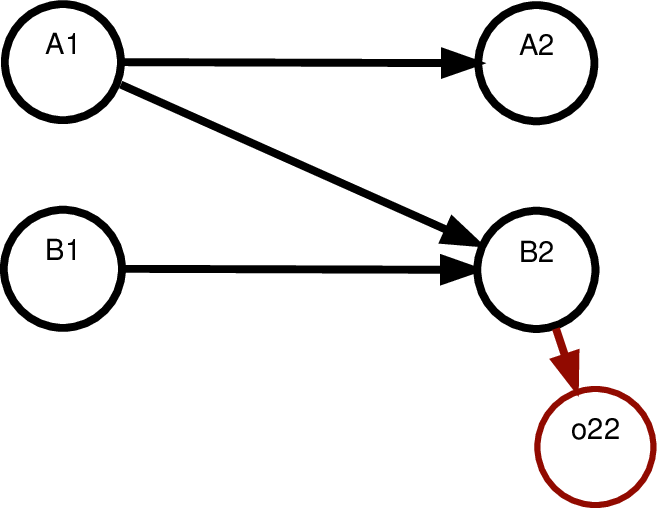}
\par\end{center}%
\end{minipage}
\par\end{centering}
}\hfill{}

\medskip{}

\begin{centering}
\subfloat[The unrolled network $G=\text{unroll}(G^{0},G^{\rightarrow})$. To
convert it to an IALM, the local-form initial belief $\lbAT i0(B^{0})$
and incoming influences $\ifpiAT i{\ts+1}(A^{t}|B^{0},\dots,B^{\ts})$
need to be computed via inference. See text for further explanation. ]{\begin{centering}
~~~~~~~\includegraphics[width=10cm]{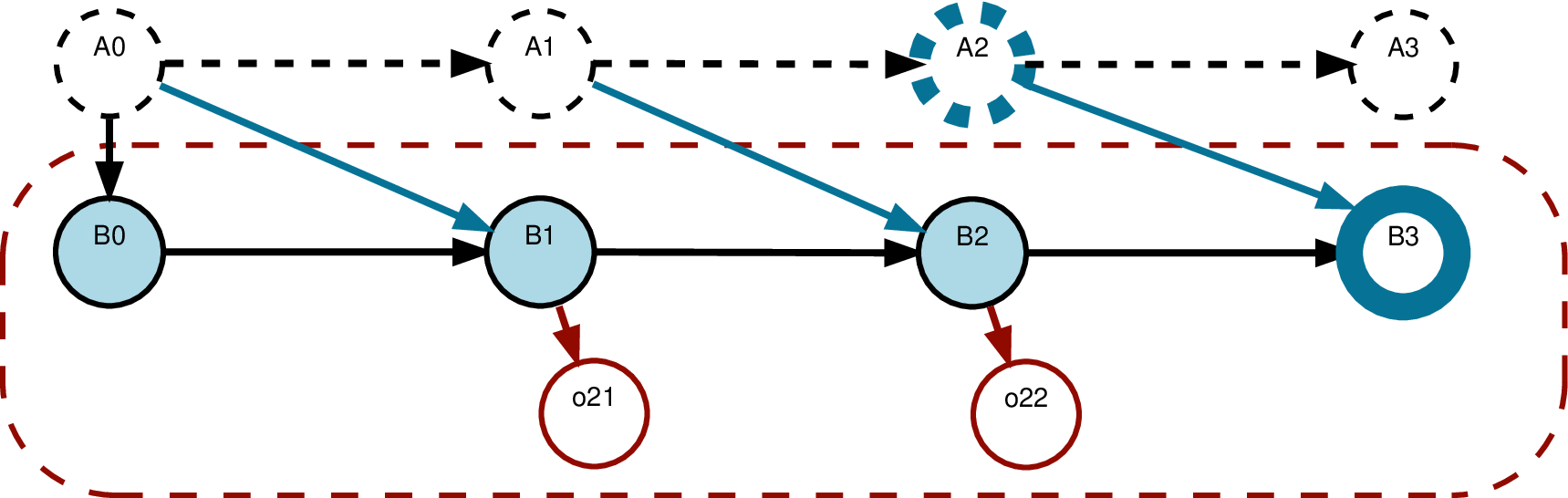}~~~~~~~
\par\end{centering}
}
\par\end{centering}
\caption{Construction of the IALM.}

\label{fig:initial-belief-local-construction}
\end{figure}

Here we discuss some of the issues involved in defining the initial
belief in the IALM. Note that in a factored models such as fPOSGs,
the initial state distribution $\bO$ is specified as a Bayesian network
$G^{0}$. Together with the 2DBN, $G^{\rightarrow}$ (which in fact
is a conditional probability distribution), it can form the unrolled
DBN $G=\text{unroll}(G^{0},G^{\rightarrow})$ which specifies the
joint distribution over all the state variables, as is illustrated
in \fig\ref{fig:initial-belief-local-construction}. Note that the
figure gives a simplified representation not involving any actions.

Now we will discuss how to specify the initial belief $\lbAT i0(\mfAT i0)$
of the IALM. The basic idea is to simply restrict $G^{0}$ to those
variables in the set $\LSF(i)$ of agent $i$'s local state variables.
However, this can lead to problems when there are arrows in $G^{0}$
pointing from variables not included in $\LSF(i)$ to variables included
in $\LSF(i)$. For instance, in \fig\ref{fig:initial-belief-local-construction},
the initial belief is factored: $\bO(\s)=\Pr(A^{0})\Pr(B^{0}|A^{0})$.
The initial local-form belief, however, should only be specified over
$B^{0}$. The solution is to marginalize out the dependencies:
\[
\lbAT i0(B^{0})=\sum_{A^{0}}\Pr(A^{0})\Pr(B^{0}|A^{0}).
\]

This is also gives the general recipe for any other problem: construction
of $\lbAT i0$ from $\bO$ is a marginal inference task. Certainly,
for certain complex problems this could be intractable, but the hope
is that for many real-world problems the prior $\bO$ is sufficiently
sparsely structured for this not to be an issue. Also, any of the
vast number of (exact or approximate) inference methods developed
in the last decades can be used~\citep{KollerFriedman09,Boyen+Koller98UAI,Jordan99MLJ,Murphy02PHD,Wainwright08FaT}. 

\paragraph{Impact of Correlations of Initial State Factors on the D-separating
Set}

\label{par:Impact-of-Correlations-ISD}

Note that the correlation of the initial state distribution can affect
d-separation and therefore what variables need to be included in the
d-separating set $\dsetAT i{\ts}$. For instance, if in the above
example there additionally is a state factor $C$, which is not connected
to $A$ or $B$ in the 2DBN $G^{\rightarrow}$, but which is a parent
of $A$ in $G^{0}$, we get the unrolled DBN as shown in \fig\ref{fig:initial-belief-influence-on-dset}.

\begin{figure}
\begin{centering}
\input{figs/frag_initial-belief.tex}\includegraphics[width=10cm]{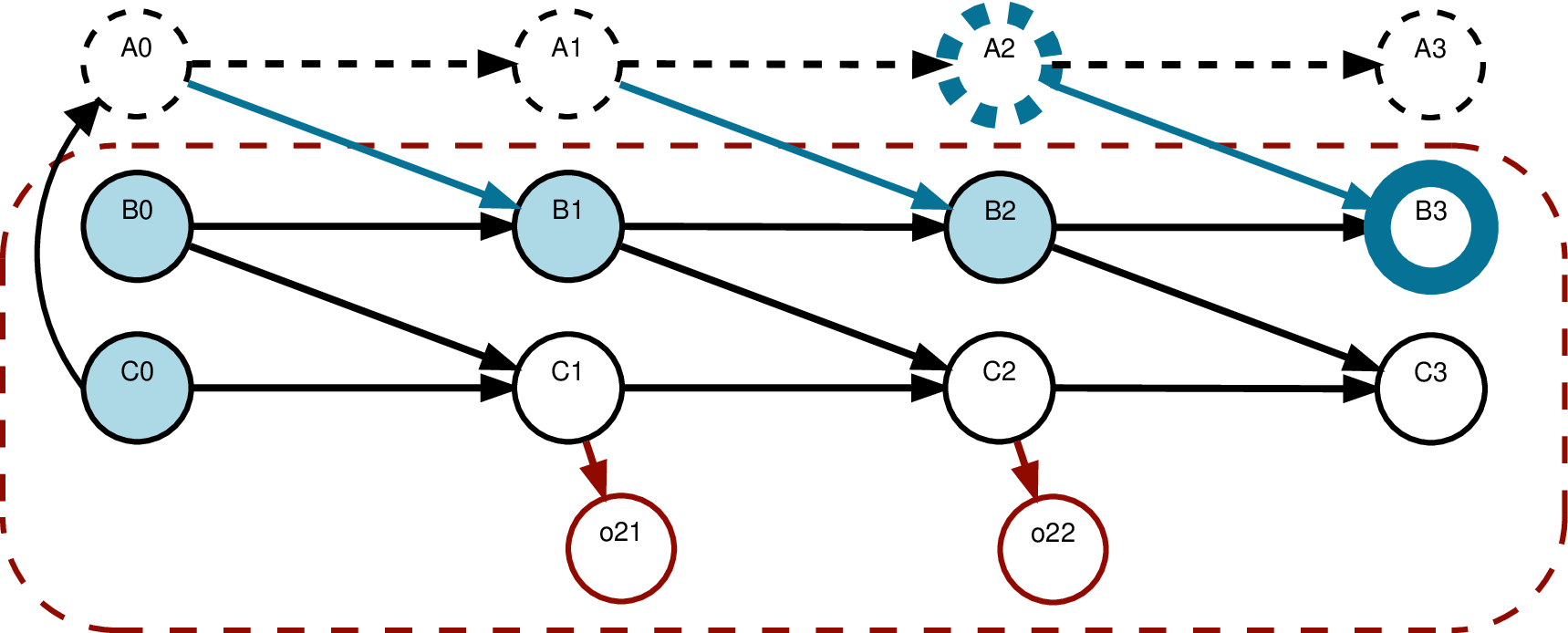}
\par\end{centering}
\caption{Impact of initial belief connectivity on the d-separating set of the
IALM.}

\label{fig:initial-belief-influence-on-dset}
\end{figure}

Now, to define the IALM, we will need the induced probability of $B^{3}$,
which according to \eqref{eq:P(NLAFS)__general} can be written\textbf{
}as
\[
\Pr(B^{3}|\left\langle B^{2},\dsetAT i3\right\rangle ,\ifpiAT i3)=\sum_{A^{2}}\iffunc(A^{2}|\dsetAT i3)\Pr(B^{3}|B^{2},A^{2}).
\]
Therefore $\dsetAT i3$ needs to contain any variables that can be
used to better predict $A^{2}$ (more formally, any variables that
d-separate $A^{2}$ from $\aoHistAT i{\ts}$ and any remaining variables
$x_{i}^{t}$, cf. \dfn\ref{dfn:d-separating-set-NO-IS}). However,
looking at the figure, we see that this means that $C^{0}$ needs
to be included in the d-set.

At the same time, however, we see that we do not need to condition
on the entire history $\vec{C}^{t}$. This may appear counter intuitive,
since observations at later time steps (e.g, $\oAT i1$ and $\oAT i2$)
certainly provide information about $A^{2}$, while they also depend
on $\vec{C}^{t}$. But this is precisely the point: by including $C^{0}$
in the d-separating set, it becomes part of the hidden state\textemdash e.g.,
for $\ts=2$ we have $\sAugAT i2=\langle\mfAT i2,\dsetAT i3\rangle=\langle\left\langle C^{2}\right\rangle ,\left\langle B^{0},B^{1},B^{2},C^{0}\right\rangle \rangle$\textemdash and
those later observations certainly provide information as to what
that hidden state is.

\subsection{Planning in an IALM}

\label{sec:IBA:planning}

Here we look at how we can plan using an IALM. It turns out that this
is surprisingly simple, since an IALM \emph{is a }(special case of)
POMDP.

\begin{observation}An influence-augmented local model is a POMDP.\begin{proof}This
can simply be verified by comparing \dfn\ref{dnf:IALM} to the definition
of a POMDP (\dfn\ref{def:POMDP}).\end{proof}\end{observation}

This means that belief updates and definition of value functions follow
as usual. For completeness and future reference, we write these out
in detail below.

\subsubsection{Local-form Belief Update}

As implied by \dfn\ref{dnf:IALM}, in an IALM, an agent uses a \emph{local-form
belief}:

\begin{definition}[local-form belief]

A \emph{local-form belief $\lbAT i\ts$ }for an IALM constructed for
agent~$i$ is the posterior probability distribution over augmented
states $\sAugAT i{\ts}=\langle\mfAT i{\ts},\dsetAT i{\ts+1}\rangle$.

\end{definition}

The belief update for such a local-form belief is as in a regular
POMDP, cf. \eqref{eq:POMDP_BU}:
\begin{multline}
BU(\lbA i,\aAT it,\oAT i{\ts+1})(\sAugAT i{t+1})=\frac{1}{\Pr(\oAT i{\ts+1}|\lbA i,\aAT it)}\bar{O}_{i}(\oAT i{\ts+1}|\aAT i{\ts},\sAugAT i{t+1})\sum_{\sAugAT it}\bar{T}_{i}(\sAugAT i{t+1}|\sAugAT it,\aAT i{\ts})\lbA i(\sAugAT it)=\\
\frac{O(\oAT i{\ts+1}|\aAT i{\ts},\mfAT i{\ts+1})}{\Pr(\oAT i{\ts+1}|\lbA i,\aAT it)}\sum_{\mfAT i{\ts},\dsetAT i{\ts+1}}\Pr(\mfAT i{\ts+1}|\langle\mfAT i{\ts},\dsetAT i{\ts+1}\rangle,\aAT i{\ts},\ifpiAT i{\ts+1})\KroD{\dsetAT i{\ts+2}}{\dsetUF(\mfAT i{\ts},\aAT i{\ts},\mfAT i{\ts+1},\dsetAT i{\ts+1})}\lbA i(\mfAT i{\ts},\dsetAT i{\ts+1})\label{eq:IALM_BU}
\end{multline}
The expected observation probability (the normalization factor) in
this case can be shown (see the derivation in \app\ref{app:IALM:exp_obs})
to satisfy
\begin{eqnarray}
\Pr(\oAT i{\ts+1}|\lbA i,\aAT it) & = & \E_{\sAugAT it\sim\lbA i,\sAugAT i{t+1}\sim\Aug T(\sAugAT it,\aAT it,\cdot)}\left[\Aug O(\oAT i{\ts+1}|\aAT it,\sAugAT i{t+1})\right]\nonumber \\
 & = & \E_{\left\langle \mfAT i{\ts},\dsetAT i{\ts+1}\right\rangle \sim\lbA i,\left\langle \mfAT i{\ts+1},\dsetAT i{\ts+2}\right\rangle \sim\Aug T(\left\langle \mfAT i{\ts},\dsetAT i{\ts+1}\right\rangle ,\aAT it,\cdot)}\left[O(\oAT i{\ts+1}|\aAT it,\mfAT i{\ts+1})\right]\nonumber \\
 & = & \sum_{\mfAT i{\ts+1}}O(\oAT i{\ts+1}|\aAT it,\mfAT i{\ts+1})\Pr(\mfAT i{\ts+1}|\lbA i,\aAT it,\ifpiAT i{\ts+1}),\label{eq:P(o|lfb,a)}
\end{eqnarray}
with
\begin{equation}
\Pr(\mfAT i{\ts+1}|\lbA i,\aAT it,\ifpiAT i{\ts+1})\defas\sum_{\mfAT i{\ts},\dsetAT i{\ts+1}}\Pr(\mfAT i{\ts+1}|\langle\mfAT i{\ts},\dsetAT i{\ts+1}\rangle,\aAT i{\ts},\ifpiAT i{\ts+1})\lbA i(\mfAT i{\ts},\dsetAT i{\ts+1}).\label{eq:P_fm__lb}
\end{equation}

\subsubsection{IALM Value}

Putting everything together, we can show that for an IALM, the value
function is similar to the normal POMDP value function:

\begin{proposition}[IALM value function]The value function is given
by

\begin{equation}
\QAT i\ts(\lbA i,\aAT it)=\RA i(\lbA i,\aAT it)+\gamma\sum_{\oAT i{\ts+1}}\Pr(\oAT i{\ts+1}|\lbA i,\aAT it)\VAT i{\ts+1}(BU(\lbA i,\aAT it,\oAT i{\ts+1})),
\end{equation}
\begin{equation}
\VAT i{\ts+1}(\lbA i)=\max_{\aA i}\QAT i{\ts+1}(\lbA i,\aA i),
\end{equation}
where 

\begin{eqnarray}
\RA i(\lbA i,\aAT it) & = & \E_{\sAugAT it\sim\lbA i,\sAugAT i{t+1}\sim\Aug T(\sAugAT it,\aAT it,\cdot)}\left[\Aug{R_{i}}(\sAugAT it,\aAT it,\sAugAT i{t+1})\right]\nonumber \\
 & = & \sum_{\mfAT i{\ts}}\sum_{\mfAT i{\ts+1}}\RA i(\mfAT i{\ts},\aAT it,\mfAT i{\ts+1})\Pr(\mfAT i{\ts},\mfAT i{\ts+1}|\lbA i,\aAT it,\ifpiAT i{\ts+1})\label{eq:R(lfb,a)}
\end{eqnarray}
with 
\begin{equation}
\Pr(\mfAT i{\ts},\mfAT i{\ts+1}|\lbA i,\aAT it,\ifpiAT i{\ts+1})\defas\sum_{\dsetAT i{\ts+1}}\Pr(\mfAT i{\ts+1}|\langle\mfAT i{\ts},\dsetAT i{\ts+1}\rangle,\aAT i{\ts},\ifpiAT i{\ts+1})\lbA i(\mfAT i{\ts},\dsetAT i{\ts+1}).\label{eq:P_xmxm_lfb}
\end{equation}

\begin{proof} This follows from the value function of regular POMDPs
together with the derivations of $\RA i(\lbA i,\aAT it)$ and $\Pr(\mfAT i{\ts},\mfAT i{\ts+1}|\lbA i,\aAT it,\ifpiAT i{\ts+1})$
in \app\ref{app:IALM:exp_reward}.\end{proof}

\end{proposition}

The solution of the IALM gives the influence-based best-response value,
defined as the value of the initial local-form belief:

\begin{equation}
\VA i(\ifpiA i(\polA{\excl i}))\defas\VAT i0(\lbAT i0).\label{eq:V(b0)_inf}
\end{equation}

\subsection{IBA by Example}

\label{sec:IBA-by-Example}

To further clarify the process of influence-based abstraction, and
provide some intuition of the potential computational savings and
in which cases they could arise, we discuss two examples in some more
detail.

\paragraph{The \problemName{Planetary Exploration} Domain.}

First, let us consider the planetary rover domain illustrated in \fig\ref{fig:dset-rover}.
We will give a characterization of this problem in terms of number
of states, for both the global-form best-response model (GFBRM) and
IALM, thus providing an analysis of the computational savings that
can arise in this case.

We will use the following notations and assumptions:
\begin{itemize}
\item $L$ is the number of locations
\item $pl\in\left\{ yes,no\right\} $ indicates if a plan has been sent
to the rover.
\item $B$ is the number of private states $\argsA{bt}{1}$ of the satellite
(e.g., number of battery levels).
\item $\left|\oAS1\right|$ is the size of the observation set of agent~1.
In case that $\oA1=\argsA{bt}{1}$, i.e., the satellite can perfectly
observe its battery level, we would have $\left|\oAS1\right|=B$.
\item $\aA1\in\left\{ NOOP,\,PLAN\right\} $ is the action of the satellite.
\end{itemize}
So now, in the GFBRM model, the augmented state for the rover (who
is agent~2) is $\sAugAT2t=\langle\sT{\ts},\aoHistAT1{\ts}\rangle=\langle\argsAT{bt}{1}\ts,\argsT{pl}\ts,\argsAT l2\ts,\aoHistAT1{\ts}\rangle$.
Given the above assumptions, the number of AOHs for the satellite
at stage $\ts$ is $|\aoHistATS1\ts|=\left(2\left|\oAS1\right|\right)^{\ts}$.
And therefore the state space of the GFBRM is of size $\left|\argsAT{\bar{\sS}}i\ts\right|=2LB\cdot\left(2\left|\oAS1\right|\right)^{\ts}$.

In contrast, in an IALM, we have states $\sAugAT2t=\langle\mfAT2{\ts},\dsetAT2{\ts+1}\rangle=\left\langle \left\langle \argsT{pl}\ts,\argsAT l2\ts\right\rangle ,\argsT{pl}{0:\ts}\right\rangle =\left\langle \argsAT l2\ts,\argsT{pl}{0:\ts}\right\rangle $,
meaning that the size of the state space in the IALM is $L\cdot2^{\ts+1}$,
which is strictly smaller than the GFBRM model.

Moreover, we can exploit that fact that $pl$ can only turn on, meaning
that $\argsT{pl}{0:\ts}$ has only $\ts+2$ possible values: it got
turned on on one of the stages $0\dots t$ or it has not yet been
turned on (``$NotYet$''). We refer to this re-coded variable as
$PlanIssueTime^{\ts}=\dsetCompF(\argsT{pl}{0:\ts})$ such that we
can write $\sAugAT it=\left\langle \argsA l2,PlanIssueTime^{\ts}\right\rangle $.
This means that the number of IALM states at stage $\ts$ in the \problemName{planetary exploration} problem
can be further reduced to $L\left(\ts+2\right)$. This suggests that
the IALM will be much cheaper to solve for larger horizons: it scales
linearly with $\ts$, whereas the GFBRM scales exponentially with
$\ts$.

However, our discussion so far has excluded the time it takes to construct
the IALM. Specifically, to compute the transition probabilities for
every stage $\ts$, we will need to compute the incoming influence:
\[
\iffunc(\ifsAT1{\ts+1}|\dsetAT1{\ts+1})=\Pr(\aAT1{\ts}|PlanIssueTime^{\ts}).
\]

In general, we would need to compute this for all possible instantiations
of $\dsetAT2{\ts+1}$. However, in this case, the action of the satellite
agent is only relevant in when $PlanIssueTime^{\ts}=NotYet$, which
mean that we only need to compute the probability $\Pr(\aAT1{\ts}|PlanIssueTime^{\ts}=NotYet)$.
Applying \eqref{eq:IncomingInfluence__non-isd} yields (we leave $\bO,\jpolG{\excl i}$
implicit):
\[
\Pr(\aAT1{\ts}|PlanIssueTime^{\ts})=\sum_{\aoHistAT1{\ts}}\polA1(\aAT1{\ts}|\aoHistAT1{\ts})\Pr(\aoHistAT1{\ts}|PlanIssueTime^{\ts}),
\]
which shows that if we have $\Pr(\aoHistAT1{\ts}|PlanIssueTime^{\ts})$
for each stage $\ts$, we can directly derive $\Pr(\aAT1{\ts}|PlanIssueTime^{\ts})$.
The main issue therefore is to compute $\Pr(\aoHistAT1{\ts}|PlanIssueTime^{\ts})$
for all $\ts=1\dots h-1$. This is essentially a belief tracking problem.
Specifically, we can model this as a special type of hidden Markov
model where the hidden state is $\langle\argsAT{bt}{1}\ts,\aoHistAT1{\ts}\rangle$,
and our observations are $\argsT{pl}\ts$. The number of such states
at stage $\ts$ is $B\cdot\left(2\left|\oAS1\right|\right)^{\ts}$
and that also is the dominant term in the complexity.

This shows that in this example, computing a best-response using a
GFBRM requires to solve a POMDP with $\left|\mathcal{S}^{GFBR}\right|=2LB\cdot\left(2\left|\oAS1\right|\right)^{\ts}$,
while doing it using an IALM requires solving a POMDP with $\left|\mathcal{S}^{IALM}\right|=L\left(\ts+2\right)$
states and a construction cost of $O\left(B\cdot\left(2\left|\oAS1\right|\right)^{\ts}\right)$.
This means that particularly when also the number $L$ of locations
is considerable, we can tackle much larger problems, since we have
isolated the exponential complexity of tracking $\aoHistAT1{\ts}$
from the impact of the number of locations $L$. Similarly, in the
case where $B$ (in general the number of states induced by non-modeled
factors) is very large, this cost now only appears in the IALM construction
and is not multiplied with $L$.\footnote{Of course, the reader could wonder in how far these terms are relevant,
given the remaining exponential dependence on the horizon via the
cost of tracking $\aoHistAT1{\ts}$. To answer this, let us point
out that this exponential dependence is directly the consequence of
the fact that in this paper we have not restricted the class of policies
considered for other agents, but assumed these are general mappings
from AOHs to actions. However, this problem is inherently complex.
In fact, unless a particular compact description is available, the
size of the policy of the satellite $\polA1$, i.e., \emph{the size
of the input }(tabular representations of the $\polA1$) of the best-response
problem, is exponential in the horizon. However, in cases where the
policy $\polA1$ of the other agent has a compact representations
(e.g., a finite state controller with $K$ states) it may be possible
to substantially reduce the cost of tracking the other agent's internal
state (e.g., $\aoHistAT1{\ts}$ is replaced by agent 1's controller
node and tracking can be done in time $O(B\cdot K)$.}

\paragraph{The \<housesearch> Domain}

Next, we turn to the \<housesearch> problem, illustrated in \fig\ref{fig:dset-housesearch}.
In contrast to \problemName{Planetary Exploration}, \<housesearch>
exhibits a more substantial d-separating set $\dsetAT2{\ts+1}$ and
so we expect less computational savings. In fact, as we will see below,
the IALM provides little to no computational benefit \emph{except}
in the face of additional assumptions on the structure of the problem.

Let us again define the sizes of the relevant quantities:
\begin{itemize}
\item $L$ is the number of locations.
\item $f\in\left\{ yes,no\right\} $ indicates if the target has been found.
Once a target has been found the location $\argsA l1$ of agent~1
no longer has any effect.
\item $\left|\oAS1\right|$ is the size of the observation set of agent~1.
In case that $\oA1=\left\langle \argsA l{1},f\right\rangle $, i.e.,
the agent can perfectly observe its location and if the target is
found, we would have $\left|\oAS1\right|=2L$.
\item $\aAS1=\aAS2$ are the action sets that can allow the agents to move
to adjacent rooms.
\end{itemize}
Repeating the analysis, we see that the GFBRM state is $\sAugAT2t=\langle\sT{\ts},\aoHistAT1{\ts}\rangle=\left\langle \left\langle \argsAT l{1}\ts,\argsT{f}\ts,\argsAT{l}{tgt}{\ts},\argsAT l2{\ts}\right\rangle ,\aoHistAT1{\ts}\right\rangle $.
Since $|\aoHistATS1\ts|=\left(\left|\aAS1\right|\left|\oAS1\right|\right)^{\ts},$
the state space of the GFBRM is of size $\left|\argsT{\bar{\sS}}\ts\right|=2L^{3}\left(\left|\aAS1\right|\left|\oAS1\right|\right)^{\ts}$.
If we assume the agent has 4 movement actions and can perfectly observe
it location and if the target is found, as above, this becomes $2L^{3}\left(4\cdot2L\right)^{\ts}=2L^{3}\left(8L\right)^{\ts}=2^{3t+1}L^{t+3}$.

On the other hand, the IALM has states $\sAugAT2t=\langle\mfAT2{\ts},\dsetAT2{\ts+1}\rangle=\left\langle \argsT{f}{0:\ts},\argsAT{l}{tgt}{0:\ts},\argsAT l2{0:\ts}\right\rangle $.
Therefore, without further simplifications, the size of the state
space at stage $\ts$ in the IALM is $2^{\ts+1}\cdot L^{2(t+1)}$.
In other words, even disregarding the construction costs of the IALM,
this would only guaranteed to be smaller for time step $\ts=1$, as
illustrated in \tab\ref{tab:housesearch-sizes}.

Simplifications are possible, however: as before, the `found' variable
$f$ may only turn on which reduces the IALM state to $\sAugAT2t=\left\langle foundTime^{t},\argsAT{l}{tgt}{0:\ts},\argsAT l2{0:\ts}\right\rangle $
and the state space size to $(\ts+2)\cdot L^{2(t+1)}$. When the target
is stationary, this reduces to $(\ts+2)\cdot L\cdot L^{(t+1)}$. Also,
it may not be realistic that all sequences of locations are realizable.
Given a fixed start position and 4 deterministic movement actions,
the number of realizable sequences of locations would be $O(4^{t}),$which
would lead to a `simplified IALM' with state space of size $(\ts+2)\cdot L\cdot4^{\ts}$.
Exploiting the realizable location sequences of agent~$1$ to similarly
reduce the state space of the GFBRM leads to $2\cdot L^{3}\cdot(\ts+2)\cdot4^{\ts}$
states. As such, the IALM representation in terms of $\sAugAT2t=\langle\mfAT2{\ts},\dsetAT2{\ts+1}\rangle$
enables us to capture structure of the problem to significantly reduce
its local state space.

\begin{table}
\begin{centering}
\begin{tabular}{rrcllll}
\hline 
 &  &  & \multicolumn{4}{c}{stage $\ts$}\tabularnewline
model & state space size &  & $1$ & $2$ & 3 & 4\tabularnewline
\hline 
 &  &  &  &  &  & \tabularnewline
GFBRM & $2^{3\ts+1}L^{\ts+3}$ &  & $16L^{4}$ & $128L^{5}$ & $1024L^{6}$ & $8192L^{7}$\tabularnewline
GFBRM simplified & $2\cdot L^{3}\cdot(\ts+2)\cdot4^{\ts}$ &  & $24L^{3}$ & $128L^{3}$ & $640L^{3}$ & $3072L^{3}$\tabularnewline
IALM naive & $2^{\ts+1}\cdot L^{2(\ts+1)}$ &  & $4L^{4}$ & $8L^{6}$ & $16L^{8}$ & $32L^{10}$\tabularnewline
IALM simplified & $L\cdot(\ts+2)\cdot4^{\ts}$ &  & $12L$ & $64L$ & $320L$ & $1536L$\tabularnewline
 &  &  &  &  &  & \tabularnewline
\hline 
\end{tabular}
\par\end{centering}
\caption{State space sizes for GFBRM and two versions of IALM models for the
\<housesearch> problem.}
\label{tab:housesearch-sizes}
\end{table}

Of course, we did not yet cover the cost of the construction cost
of the IALM by computing the influence. Here, we sketch what is involved
in the construction of the `simplified IALM' we constructed. Specifically
the influence specification now is:
\[
\iffunc(\ifsAT1{\ts+1}|\dsetAT1{\ts+1})=\Pr(\argsAT l1\ts|\left\langle FoundTime^{\ts},\argsA{l}{tgt},\argsAT l2{0:\ts-1}\right\rangle ).
\]

As before, we only care about cases where $FoundTime^{\ts}=NotYet$
(since otherwise the location $\argsAT l1\ts$ is irrelevant). However,
the different options for $\argsA{l}{tgt},\argsAT l2{0:\ts-1}$ should
be evaluated, which means that we need to solve the inference problem
for $O(L\cdot4^{\ts-1})$ instantiations of the d-set. Each of these
instantiations requires tracking hidden states of the form $\left\langle \argsAT l{1}\ts,\aoHistAT1{\ts}\right\rangle $,
and there are $L\left(\left|\aAS1\right|\left|\oAS1\right|\right)^{\ts}$
of them in general. For the simplified setting of deterministic moves
and perfect observations by agent~1 this can be limited to $L\cdot\ts\cdot4^{t}$
since in that case $\aoHistAT1{\ts}=\left\langle \argsT f{0:\ts},\argsAT l1{0:\ts}\right\rangle =\left\langle FoundTime^{\ts},\argsAT l1{0:\ts}\right\rangle $.

In summary, the simplified version of the \<housesearch> problem
enables us to reduce the state space of the best-response model substantially,
from $2^{3\ts+1}L^{\ts+3}$ to $L\cdot(\ts+2)\cdot4^{\ts}$. However
to construct the IALM state space at stage $\ts$ still requires time
of the order $(L\cdot4^{\ts-1})(L\cdot\ts\cdot4^{\ts})=L^{2}\cdot\ts\cdot4^{2\ts-1}$.

\subsection{More General Implications of IBA}

\label{sec:More-General-Implications}

In the above, we saw that IBA can lead to more efficient computation
of exact best-responses in settings that have sufficient structure
to exploit. However, our motivation for developing the theory presented
in this paper is more general than this. Here we elaborate on the
broader implications that we envision.

\begin{figure}[tbh]
\begin{centering}
\includegraphics[width=11cm]{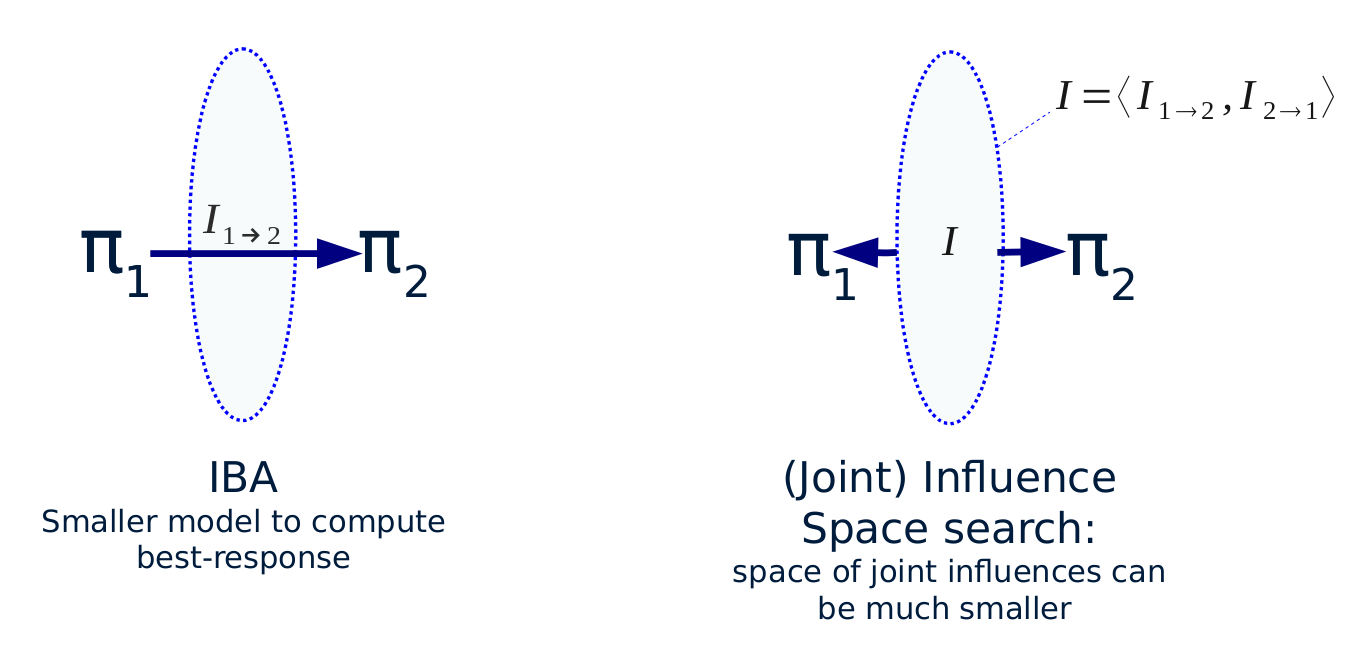}
\par\end{centering}
\caption{Illustration of the ideas of influence-based abstraction and influence
search.}

\label{fig:IBA-vs-IS}
\end{figure}

\subsubsection{Influence Search}

The ideas underlying influence-based abstraction were developed in
the research community focusing on multiagent sequential decision
making by people like \citet{Becker03AAMAS}, \citet{Varakantham09ICAPS}
and \citet{Petrik09JAIR}. The goal that these works pursued were
not the computation of merely a best response, but of the optimal
\emph{joint} policy. Hence these works performed a type of \emph{influence
search }\citep{Witwicki10ICAPS}. The idea, illustrated in \fig\ref{fig:IBA-vs-IS},
is that many policies of one agent, say agent~$1$, may correspond
to the same influence $\ifSD{1}{2}$ on agent~$2$, which would mean
that the set of such influences could be much smaller than the set
of possible policies $\polA2$. Therefore, if it is possible to search
through the space of \emph{joint influences }$\jifp=\left\langle \ifSD{1}{2},\ifSD{2}{1}\right\rangle $,
this could be much more effective than searching the much larger space
of joint policies $\jpol=\left\langle \polA1,\polA2\right\rangle $.
Specifically, \citet{Witwicki12AAMAS} showed orders of improvement
in computational cost over joint policy-search approaches.

So far, however, these ideas have only been exploited in sub-classes
of fDec-POMDPs (see also \sect\ref{sec:sub-classes-with-Compact-representations}),
and generalizing influence search to general fDec-POMDPs or even fPOSG
(i.e., to find Nash equilibria) is still an open problem. Our definition
of influence in \dfn\ref{dfn:The-incoming-influence} can serve as
a starting point for such extensions.

\subsubsection{Approximate Influence Representations}

The discussion in \sect\ref{sec:IBA-by-Example} demonstrated that,
in cases with sufficient structure, the representation of the influence
$\iffunc(\ifsAT i{\ts+1}|\dsetAT i{\ts+1})$ can be compact, leading
to substantial benefits. However, at the same time it also showed
that in general, without exploiting special properties of the domain,
these representations can become very big and unwieldy due to the
dependence on the history of a subset of variables. Large influence
representations can not be exploited for more efficient best-response
computations, and they also suggest that the number of possible influences
will be large, possibly limiting the effectiveness of influence search.

However, even though exact representations of $\iffunc(\ifsAT i{\ts+1}|\dsetAT i{\ts+1})$
may be very large, it might be the case that approximate representations
$\hat{\iffunc}(\ifsAT i{\ts+1}|\dsetAT i{\ts+1})$ can be compactly
represented while still affording good performance; for the purposes
of making good predictions it is usually not required to remember
the full history \citep{Littman94memoryless,McCallum95PhD,Kaelbling98AI,Meuleau99UAI}.
Moreover, learning such approximate influence points is a supervised
learning problem (a specific type of sequence prediction problem),
which means that we can directly build on recent advancements for
such prediction problems, including work on\emph{ deep learning}~
\citep{Schmidhuber15NN,LeCun15Nature}. Of course, whether the successes
from natural language processing~\citep{Vinyals15CVPR,Young18IEEE},
machine translation~\citep{Cho14EMNLP}, speech recognition~\citep{Graves13speech,Weninger15speech}
or biological sequence data~\citep{Jurtz17bioinformatics} will transfer
to the task of influence prediction remains to be investigated, but
there already are some positive indications.

Specifically, some studies have shown that approximate representations
of influence can enable further scalability in a variety of respects.
For instance, \citet{Oliehoek13AAMAS} introduced the idea of \emph{transfer
planning, }which defines a number of smaller source tasks, whose solution
is transferred to the larger (involving more agents) target task.
The definition of the source tasks ignores the actual influence of
the rest of the system, and hence can be seen as a very naive special
case of approximate influence-based abstraction: it assumes an arbitrary
influence point for each sub-problem. Nevertheless, the authors empirically
showed that this can lead to good performance in Dec-POMDPs with many
agents. This was corroborated by \citet{Oliehoek15IJCAI} who employed
\emph{optimistic influences} (also an approximate form of influence)
to compute factored upper bounds on the Dec-POMDP value function.
They demonstrated that in some cases the solution found by transfer
planning for Dec-POMDPs with over 50 agents was essentially optimal.
Recently, \citet{He20NeurIPS} demonstrated that, by using learned
(recurrent neural network) representations of influence in online
planning, it is possible to get better task performance when the time
for action selection is limited. Of course, giving hard performance
guarantees for such approaches is very difficult, but \citet{Congeduti20arxiv}
show that it is possible to derive performance loss bounds for approximate
influence representations. Their analysis also suggests that typical
machine learning approaches that minimize the cross-entropy loss may
be well aligned with minimizing the performance loss.

As such, there is substantial evidence that approximate extension
of the formal IBA framework presented in this paper can lead to various
benefits. Related to this is the new perspective these approaches
give on the systems they aim to control. For instance, both \citet{Oliehoek15IJCAI}
and \citet{He20NeurIPS} experimented with forms of ``influence strength''
to better understand parametrized domains by looking at the impact
on the resulting solution quality. Further formalization and refinements
of such notions could lead to a better understanding of the application
of decision making methods to complex domains.

\subsubsection{Identifying Inductive Biases}

Notions like influence strength can enable us to better understand
the problems that we are trying to tackle, and the IBA perspective
can generate more of such insights. For instance, the discussion on
the impact of the initial state distribution \vpageref{par:Impact-of-Correlations-ISD}
neatly exemplifies some different types of structure we can expect
to encounter when dealing with abstraction in structured decision
making processes.

Identification of such structure is important even for deep learning:
even though the representations are learned automatically, no learning
methods are effective without the appropriate inductive biases~\citep{Mitchell80bias,Wolpert96NC}.
For instance, convolutional neural networks are so successful for
image processing because they exploit the fact that there is local
and repeated structure in real-world images. In a similar way, in
the discussion on the initial state distribution, we noticed that
certain forms of structure, such as dependence on certain state factors
at stage $\ts=0$, might be common in sequential decision processes
involving abstraction.

In fact, recent research provides clear evidence that structure as
implied by influence-based abstraction can be effectively used to
bias deep reinforcement learning \citep{Suau19arxiv}. Specifically,
that work shows that by equipping a policy and/or value network with
a recurrent sub-network that is only fed with a subset of variables
(roughly corresponding to the d-separating set) can lead to higher
performance than feed-forward networks, while learning much faster
than a full-sized recurrent neural network. Further connections to
deep RL and multiagent RL approaches are discussed in \sect\ref{sec:Related-Work}.

\section{IBA With Intra-Stage Dependencies}

\label{sec:IBA-with-IS-deps}

The previous section presented the framework of influence-based abstraction,
which enables us to abstract away hidden state variables in so-called
local-form models. We illustrated how this can lead to speeding up
best-response computations and discussed more general implications
of the theory. So far we assumed there are no intra-stage connections:
all influence links span a time step. However, intra-stage dependencies
(ISDs) can be useful to specify a more intuitive model, as we saw
for \<housesearch> in \fig\ref{fig:house_dbn_intrastage}. Additionally,
there could be problems that only have a correct formulation using
intra-stage dependencies: since intra-stage connections can model
additional correlations, the transition functions $\Tfunc(\sT{\ts+1}|\sT{\ts},\jaT{\ts})$
that can be represented without ISDs is a strict subset of those we
can represent with ISDs. Simply removing ISDs from problems that need
them is not possible, as it is not clear what probabilities the CPTs
should specify for the remaining parents.

Moreover, intra-stage connections enable us to introduce `dummy'
variables, as discussed in \sect\ref{sec:IBA:def-of-influence:links-sources-destinations}.
Without this capability, the requirement of including all observation-relevant
and reward-relevant variables in the local state (cf.\ \dfn\ref{dfn:lfm})
would limit the applicability of influence-based abstraction. For
instance, imagine the setting where our agent's reward is directly
affected by how many other agents take the some action $a$. Without
intra-stage connections, we would be forced to model all the action
variables of the other agents in the local state, making the local
model intractable. However, using intra-stage connections, we can
instead introduce a count variable that affects our reward, while
we do not model (abstract away) all the individual actions of other
agents. As such, the ability to use ISDs can allow us to effectively
describe scenarios with anonymous interactions, such as mean-field
games and others~\citep{Jovanovic88anonymous,Kizilkale12TAC,Varakantham14AAAI,Robbel16AAAI,Nguyen17AAAI,Subramanian19AAMAS},
in the IBA framework.

Therefore, this section extends our definition of influence to also
be applicable for models that have such \emph{intra-stage dependencies
(ISDs)}.

\subsection{Definition of Influence under ISDs}

Here we extend IBA by adapting notions of influence sources, d-separating
sets, and incoming influence points to properly take into account
ISDs.

\subsubsection{Intra-Stage Influence Sources}

In settings with intra-stage dependencies, there is at least one non-modeled
factor $\nmfT{\ts+1}$ that influences an NLAF $\mfnT{\ts+1}$.  If
there are multiple such factors, we let $\nmfAT u{\ts+1}$ denote
them. Therefore, in order to perform IBA in settings with ISDs, we
will need to predict influence sources $\ifsAT i{\ts+1}=\left\langle \nmfAT u{\ts},\jaGT ut,\nmfAT u{\ts+1}\right\rangle $.
In order to correctly deal with the intra-stage sources~$\nmfAT u{\ts+1}$,
we will additionally need to consider those variables that influence
\emph{them}.

\subparagraph{Indirect Sources}

\begin{figure}
\begin{centering}
\input{figs/frag_housesearch_unrolled.tex}\includegraphics[scale=0.4]{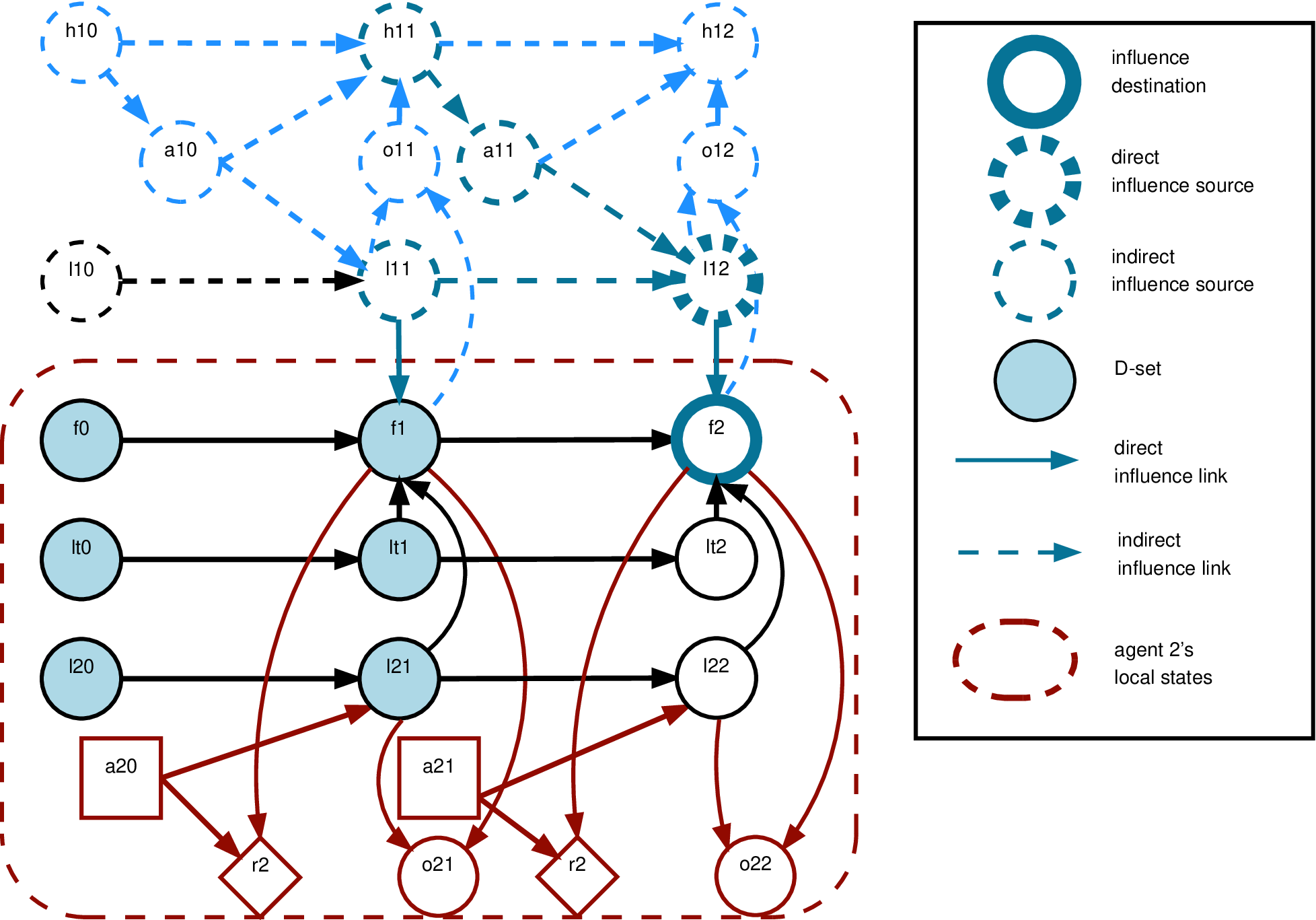}
\par\end{centering}
\caption{Illustration of the influence experienced by protagonist agent $i=2$
in the intra-stage version of \<housesearch> at stage $\ts=2$. $\argsT f2$
is the influence destination, with \emph{direct }influence source
$\nmfAT u2=\left\langle \argsAT l12\right\rangle $ (i.e., $\ifsAT i2=\left\langle \argsAT l12\right\rangle $).
Additionally, the figure highlights the \emph{indirect}~influence
sources $\nmfAT v1=\left\langle \argsAT l11\right\rangle $, $\jaGT vt=\left\langle \aAT11\right\rangle $
and $\aoHistGT v1=\left\langle \aoHistAT11\right\rangle $, which
determine the influence that is exerted at stage $\ts=1$. (Note that
$\argsAT l11$ in fact is also a \emph{direct }influence source for
the influence \emph{experienced at }stage $\ts=1$.)}

\label{fig:intra-stage-housesearch_influence-sources}
\end{figure}

In particular, we use `$v$' as the symbol to denote such `indirect'
or `second order' influences and will write $\mfAT v{\ts},\nmfAT v{\ts},\aAT i{\ts},\jaGT vt,\mfAT v{\ts+1}$
and $\nmfAT v{\ts+1}$ for the possible\footnote{Of course, in any given problem not all of these types of variables
are relevant. For instance, if there is no action $\aAT j{\ts}$ of
another agent~$j$ that would influence an ISD influence source,
then $\jaGT vt$ can be removed from the equations.} ancestors in the 2DBN of intra-stage sources~$\nmfAT u{\ts+1}$.

\begin{example}\fig~\ref{fig:intra-stage-housesearch_influence-sources}
illustrates the direct and indirect influence sources for \<housesearch>.
In order to be able to make accurate predictions of the influence
destination $\argsT f2$, at stage $\ts=1$ we should be able to predict
$\argsAT l11\,(\nmfAT v1)$ and $\aAT11\,(\jaGT vt)$ as accurately
as possible. Given that we assume access to the policy of agent~$1$,
we can equivalently predict $\nmfAT v1=\argsAT l11,\aoHistGT v1=\aoHistAT11$.\end{example}

Now, in order to define the influence, we will need to consider the
probability of such $\nmfAT u{\ts+1}$ given variables that we know
how to predict at stage $\ts$. In general it is given by: 
\begin{equation}
\Pr(\nmfAT u{\ts+1}|\mfAT v{\ts},\nmfAT v{\ts},\aAT i{\ts},\jaGT vt,\mfAT v{\ts+1})=\sum_{\nmfAT v{\ts+1}}\Pr(\nmfAT u{\ts+1},\nmfAT v{\ts+1}|\mfAT v{\ts},\nmfAT v{\ts},\aAT i{\ts},\jaGT vt,\mfAT v{\ts+1}),\label{eq:P(intra-stage-source)}
\end{equation}
where:
\begin{itemize}
\item $\Pr(\nmfAT u{\ts+1},\nmfAT v{\ts+1}|\mfAT v{\ts},\nmfAT v{\ts},\aAT i{\ts},\jaGT vt,\mfAT v{\ts+1})$
is the product of CPTs of (both direct and indirect) intra-stage sources\textemdash in
\fig~\ref{fig:intra-stage-housesearch_influence-sources} this is
simply $\Pr(\argsAT l12|\argsAT l11,\aAT11)$,
\item $\mfAT v{\ts}$ are those state factors at stage $\ts$ (``in the
left-hand slice of the 2DBN'') that are modeled by agent~$i$, and
are ancestor to an intra-stage influence source of agent~$i$ at
stage $\ts+1$ (``in the right-hand slice of the 2DBN'')\textemdash in
\fig~\ref{fig:intra-stage-housesearch_influence-sources} no such
variables exist,
\item $\nmfAT vt$ are those state factors in the left-hand slice of the
2DBN that are not modeled by agent~$i$, but are ancestor to an influence
destination of agent~$i$\textemdash in \fig~\ref{fig:intra-stage-housesearch_influence-sources}
this is $\argsAT l11$,
\item $\mfAT v{\ts+1},$ $\nmfAT v{\ts+1}$ are the modeled, respectively
non-modeled state factors at state $\ts+1$ that are ancestors to
an intra-stage influence source\textemdash in \fig~\ref{fig:intra-stage-housesearch_influence-sources}
no such variables exist,
\item $\aAT i{\ts}$ might directly or indirectly affect an an intra-stage
influence source, in which case it needs to be included in \eqref{eq:P(intra-stage-source)}\textemdash in
\fig~\ref{fig:intra-stage-housesearch_influence-sources} this is
not the case, and
\item $\jaGT vt$ are the actions of other agents that are ancestors of
an intra-stage influence source\textemdash in \fig~\ref{fig:intra-stage-housesearch_influence-sources}
this is $\aAT11$.
\end{itemize}
We will also write $\aoHistAT v{\ts}$ for the AOHs of the agents
$v$ that correspond to $\jaGT vt$ (i.e., those agents of which the
action is an ancestor in the 2DBN of an influence destination of agent~$i$).

\subparagraph{All sources}

So far we have introduced notation using $u$ for direct sources and
using $v$ for indirect sources. We will also want to consider the
union of direct and indirect sources, and for these purposes we will
write $w$. For example, we will write $\jaGT wt=\left\langle \jaGT ut,\jaGT vt\right\rangle $
for the actions of agents that either directly or indirectly influence
an influence destination.

\subsubsection{The D-Separating Set}

We now build on this insight to define the d-separating set in problems
with intra-stage dependencies:

\begin{definition}[d-separating set]\label{dfn:d-separating-set--IS}The
\emph{d-separating set for agent $i$}, $\dsetA i$, is a subset of
variables (state factors and/or actions), such that the history of
these variables d-separates $\nmfAT wt,\aoHistAT w{\ts}$ from $\mfAT i{\ts},\aoHistAT it$.
I.e., it is defined in such a way that 
\begin{equation}
\forall_{\nmfAT wt,\aoHistAT w{\ts}}\qquad\Pr(\nmfAT wt,\aoHistAT w{\ts}|\mfAT it,\aoHistAT it,\dsetAT i{\ts+1},\bO,\jpolG{\excl i})=\Pr(\nmfAT wt,\aoHistAT w{\ts}|\dsetAT i{\ts+1},\bO,\jpolG{\excl i}).\label{eq:dset-def--intrastage}
\end{equation}

\end{definition}

As before this should be interpreted to mean: $\dsetAT i{\ts+1}$
d-separates $\nmfAT wt,\aoHistAT w{\ts}$ from those parts of $\mfAT i{\ts},\aoHistAT it$
(i.e, of the local model) not contained in $\dsetAT i{\ts+1}$.

Comparing \dfn\ref{dfn:d-separating-set--IS} with the earlier \dfn\ref{dfn:d-separating-set-NO-IS},
we see they are pleasingly similar; all that changed is that $u$'s
have been replaced with $w$'s to now take into account the possibility
of indirect sources.

\subsubsection{Definition of Influence under ISDs}

With this as background, we are now in position to define the concept
of influence in all its generality:

\begin{definition}[Experienced Influence under ISDs]

\label{dfn:influence} The \emph{influence experienced by agent~$i$
at stage $\ts+1$} is a conditional probability distribution over
the direct influence sources:
\begin{multline}
\iffunc(\ifsAT i{\ts+1}|\dsetAT i{\ts+1},\mfAT v{\ts},\aAT i{\ts},\mfAT v{\ts+1})\defas\Pr(\left\langle \nmfAT u{\ts},\jaGT ut,\nmfAT u{\ts+1}\right\rangle |\dsetAT i{\ts+1},\mfAT v{\ts},\aAT i{\ts},\mfAT v{\ts+1},\bO,\jpolG{\excl i})\\
=\sum_{\left\langle \nmfAT v{\ts},\jaGT vt,\nmfAT v{\ts+1}\right\rangle }\Pr(\nmfAT u{\ts+1},\nmfAT v{\ts+1}|\mfAT v{\ts},\nmfAT v{\ts},\aAT i{\ts},\jaGT vt,\mfAT v{\ts+1})\sum_{\aoHistAT w{\ts}}\jpolG w(\jaGT wt|\aoHistAT w{\ts})\Pr(\nmfAT wt,\aoHistAT w{\ts}|\dsetAT i{\ts+1},\bO,\jpolG{\excl i})\label{eq:definition-of-influence}
\end{multline}
where
\begin{itemize}
\item $u$ denote (direct) influence sources;
\item $v$ denote the (indirect) `second order' sources;
\item $w$ (as above) denotes the union of $u$ and $v$;
\item $\Pr(\nmfAT u{\ts+1},\nmfAT v{\ts+1}|\mfAT v{\ts},\nmfAT v{\ts},\aAT i{\ts},\jaGT vt,\mfAT v{\ts+1})$
is the term necessary to predict the intra-stage sources. It is a
term that consists of the product of CPTs;
\item $\jpolG w(\jaGT wt|\aoHistAT w{\ts})=\prod_{i\in w}\polA i(\aAT i{\ts}|\aoHistAT i{\ts})=\jpolG u(\jaGT ut|\aoHistAT u{\ts})\jpolG v(\jaGT vt|\aoHistAT v{\ts})$
is the product of action probabilities according to the policies of
the other agents that are relevant directly (the $u$) or indirectly
for intra-stage sources (the $v$); and
\item $\Pr(\nmfAT wt,\aoHistAT w{\ts}|\dsetAT i{\ts+1},\bO,\jpolG{\excl i})=\Pr(\nmfAT ut,\nmfAT vt,\aoHistAT u{\ts},\aoHistAT v{\ts}|\dsetAT i{\ts+1},\bO,\jpolG{\excl i})$
predicts the non-modeled factors that are relevant directly (the $u$)
or indirectly for intra-stage sources (the $v$), as well as the histories
for the relevant agents.
\end{itemize}
Tying back to the example of \fig\ref{fig:intra-stage-housesearch_influence-sources},
\eqref{eq:definition-of-influence} reduces to
\[
\iffunc(\argsAT l12|\dsetAT22)=\sum_{\argsAT l11,\aAT11}\Pr(\argsAT l12|\argsAT l11,\aAT11)\sum_{\aoHistAT11}\jpolG1(\jaGT11|\aoHistAT11)\Pr(\argsAT l11,\aoHistAT11|\dsetAT12,\bO,\jpolG1).
\]

We use $\ifpiAT i{\ts+1}(\jpolG{\excl i})$ to denote the conditional
distribution $\iffunc(\cdot|\dsetAT i{\ts+1},\mfAT i{\ts},\aAT i{\ts},\mfAT i{\ts+1})$.
\end{definition}

We make a few observations:
\begin{itemize}
\item The term $\Pr(\nmfAT u{\ts+1},\nmfAT v{\ts+1}|\mfAT v{\ts},\nmfAT v{\ts},\aAT i{\ts},\jaGT vt,\mfAT v{\ts+1})$
can be simplified as given by \eqref{eq:P(intra-stage-source)}, but
it is important to keep in mind that this resulting term requires
actual inference and is not the product of CPTs anymore.
\item Note that, in many cases, we will consider other agents that use deterministic
policies, however, we chose to give the more general description that
also allows for stochastic policies. In case of deterministic policies,
the summation over $\jaGT vt$ can be omitted, $\jaGT{v/w}t$ can
be replaced by $\jpolGT{v/w}{\ts}(\oHistAT{v/w}{\ts})$, and $\aoHistAT w{\ts}$
becomes $\oHistAT w{\ts}$~\citep{Oliehoek12RLBook}. 
\item The dependence of $\iffunc(\ifsAT i{\ts+1}|\dsetAT i{\ts+1},\mfAT v{\ts},\aAT i{\ts},\mfAT v{\ts+1})$
on $\aAT i{\ts}$ is only needed when $\aAT i{\ts}$ is an indirect
source (i.e., it is an ancestor of $\nmfAT u{\ts+1}\text{ or }\nmfAT v{\ts+1}$).
\end{itemize}

\subsubsection{Exerted vs. Experienced Influence}

Here we make a reinterpretation of the experienced influence at stage
$t+1$ as the result of the influence exerted at stage $\ts$ plus
the effect of the intra-stage effects. While this does not fundamentally
change anything about the definition of influence per \dfn\ref{dfn:influence},
it may provide some insight on the nature with which influence manifests
itself in settings with intra-stage connections, and provide guidance
for possible implementations.

In particular, it is possible to define a distribution, only in terms
of variables at stage $\ts$, which acts as a sufficient statistic
to predict the intra-stage source. The intuition is that the \emph{experienced
influence,} can be thought of as being induced by the \emph{exerted
influence:}\textbf{ }
\begin{itemize}
\item \textbf{Exerted Influence (at stage~$t$):
\begin{multline}
\Pr(\nmfAT wt,\jaGT w{\ts}|\dsetAT i{\ts+1},\bO,\jpolG{\excl i})=\Pr(\nmfAT u{\ts},\nmfAT vt,\jaGT u{\ts},\jaGT v{\ts}|\dsetAT i{\ts+1},\bO,\jpolG{\excl i})\\
=\sum_{\aoHistAT w{\ts}}\jpolG w(\jaGT w{\ts}|\aoHistAT w{\ts})\Pr(\nmfAT wt,\aoHistAT w{\ts}|\dsetAT i{\ts+1},\bO,\jpolG{\excl i}).\label{eq:exerted_influence}
\end{multline}
}
\item \textbf{Experienced Influence (at $t+1$):
\begin{multline}
\iffunc(\ifsAT i{\ts+1}|\dsetAT i{\ts+1},\mfAT v{\ts},\aAT i{\ts},\mfAT v{\ts+1})=\Pr(\nmfAT u{\ts},\jaGT u{\ts},\nmfAT u{\ts+1}|\dsetAT i{\ts+1},\mfAT v{\ts},\mfAT v{\ts+1},\bO,\jpolG{\excl i})\\
=\sum_{\left\langle \nmfAT v{\ts},\jaGT vt,\nmfAT v{\ts+1}\right\rangle }\Pr(\nmfAT u{\ts+1},\nmfAT v{\ts+1}|\mfAT v{\ts},\nmfAT v{\ts},\aAT i{\ts},\jaGT vt,\mfAT v{\ts+1})\Pr(\nmfAT wt,\jaGT w{\ts}|\dsetAT i{\ts+1},\bO,\jpolG{\excl i}).\label{eq:experienced_influence}
\end{multline}
}
\end{itemize}
This last equation \eqref{eq:experienced_influence} clearly demonstrates
how the experienced influence is induced by the exerted influence.
The notion of exerted influence \eqref{eq:exerted_influence} lays
a clear link to IBA in settings without ISDs (cf. Equation \ref{eq:IncomingInfluence__non-isd})
and is conceptually useful since it isolates which information needs
to be retained for each stage $\ts$. As such, we expect that any
practical implementations for computing the influence by means of
filtering (belief tracking)~\citep{RussellNorvig09Book3rdEd,Thrun05BookPR}
would use this as the primary quantity of interest.

\subsection{Influence-Augmented Local Model (IALM)}

Here we define the influence-augmented local model under intra-stage
connections. Looking at \dfn\ref{dnf:IALM}, we can conclude that
the only changes that we need to make involve the transition function
\eqref{eq:IALM-T}.

In particular, we need to deal with the fact our definition of influence
\eqref{eq:definition-of-influence} can now be of the more complex
form $\iffunc(\ifsAT i{\ts+1}|\dsetAT i{\ts+1},\mfAT v{\ts},\aAT i{\ts},\mfAT v{\ts+1})$,
as given by \eqref{eq:experienced_influence}. This means that the
NLAF probability $\Pr(\mfnAT i{\ts+1}|\langle\mfAT i{\ts},\dsetAT i{\ts+1}\rangle,\aAT i{\ts},\ifpiAT i{\ts+1})$
given by \eqref{eq:P(NLAFS)__general} must be updated to deal with
this new form, and this in turn implies that the definition of $\bar{T}_{i}(\sAugAT i{t+1}|\sAugAT it,\aAT i{\ts})$
per \eqref{eq:IALM-T} needs to be updated too.

Let us start with the former. Like \eqref{eq:LFM:T:NLAFs}, this can
now depend on ISDs from OLAFs $\mflAT i{\ts+1}$

\noindent 
\begin{multline}
\Pr(\mfnAT i{\ts+1}|\langle\mfAT i{\ts},\dsetAT i{\ts+1}\rangle,\mflAT i{\ts+1},\aAT i{\ts},\ifpiAT i{\ts+1})\defas\\
\sum_{\ifsAT i{\ts+1}=\left\langle \nmfAT u{\ts},\jaGT u\ts,\nmfAT u{\ts+1}\right\rangle }\iffunc(\ifsAT i{\ts+1}|\dsetAT i{\ts+1},\mfAT v{\ts},\aAT i{\ts},\mfAT v{\ts+1})\Pr(\mfnAT i{\ts+1}|\mfAT i{\ts},\mflAT i{\ts+1},\aAT i{\ts},\ifsAT i{\ts+1}),\label{eq:P(NLAFS)__general-IS}
\end{multline}
with $\Pr(\mfnAT i{\ts+1}|\mfAT i{\ts},\mflAT i{\ts+1},\aAT i{\ts},\ifsAT i{\ts+1})$
simply the product of CPTs of the NLAFs, as given by \eqref{eq:LFM:T:NLAFs},
but now restricted to only $\nmfAT i{\ts},\nmfAT i{\ts+1},\jaGT{\excl i}\ts$
that are influence sources.

We are now in the position to define the IALM under intra-stage dependencies:

\begin{definition}[IALM]

\label{dnf:IALM-IS} Given an LFM with intra-stage dependencies, $\mathcal{M}^{LFM}$,
and profile of policies for other agents~$\jpolG{\excl i}$, an \emph{Influence-Augmented
Local Model (IALM) }for agent~$i$ is a POMDP $\mathcal{M}_{i}^{IALM}(\mathcal{M}^{LFM},\jpolG{\excl i})=\left\langle \bar{\sS},\aAS{i},\bar{T}_{i},\bar{R}_{i},\oAS{i},\bar{O}_{i},\hor,\lbAT i0\right\rangle $,
where
\begin{itemize}
\item $\bar{\sS},\aAS{i},\bar{R}_{i},\oAS{i},\bar{O}_{i},\hor,\lbAT i0$
are identical to those in \dfn\ref{dnf:IALM},
\item $\bar{T}_{i}$ is the transition function is defined as:
\end{itemize}
\begin{multline}
\bar{T}_{i}(\sAugAT i{t+1}|\sAugAT it,\aAT i{\ts})\defas\Pr(\mfAT i{\ts+1}|\langle\mfAT i{\ts},\dsetAT i{\ts+1}\rangle,\aAT i{\ts},\ifpiAT i{\ts+1})\KroD{\dsetAT i{\ts+2}}{d(\mfAT i{\ts},\aAT i{\ts},\mfAT i{\ts+1},\dsetAT i{\ts+1})}\\
=\Pr(\mfnAT i{\ts+1}|\langle\mfAT i{\ts},\dsetAT i{\ts+1}\rangle,\mflAT i{\ts+1},\aAT i{\ts},\ifpiAT i{\ts+1})\Pr(\mflAT i{\ts+1}|\mfAT i{\ts},\mfnAT i{\ts+1},\aAT i{\ts})\KroD{\dsetAT i{\ts+2}}{d(\mfAT i{\ts},\aAT i{\ts},\mfAT i{\ts+1},\dsetAT i{\ts+1})},\label{eq:IALM-T-IS}
\end{multline}
with the first term is given by \eqref{eq:P(NLAFS)__general-IS} and
the second term is given by \eqref{eq:LFM:T:OLAFs}.

\end{definition}

\subsection{Planning in an IALM with ISDs}

\label{sec:ISD:planning}

Since the only modifications that we needed to make to incorporate
ISDs were in the transition function, the conclusions about how to
plan in IALM made in \sect\ref{sec:IBA:planning} remain valid. In
particular, the IALM is still a POMDP, with a well-defined belief-update
function, and value functions. The solution of the IALM still gives
the influence-based best-response value, defined in \eqref{eq:V(b0)_inf}
as the value of the initial local-form belief: $\VA i(\ifpiA i(\polA{\excl i}))\defas\VAT i0(\lbAT i0).$

\section{Sufficiency of Influence-Based Abstraction}

\label{sec:Sufficiency}

In this section, we will show that influence-based abstraction is
\emph{completely lossless. }By that we mean that an IALM constructed
according to \dfn\ref{dnf:IALM-IS} can be used to accurately predict
rewards and observations, and thus to compute an exact, optimal (best-response)
value.

The latter is our main result, \thm\ref{thm:V_IALM=00003DV_GFBRM},
which shows that the optimal values for the GFBRM and the IALM are
equal, thus establishing that one can use the IALM to plan (or learn)
without any loss in value. In other words, it proves that the definition
of influence constitutes a sufficient statistic for predicting the
optimal value, and thus that the resulting IALM achieves a best-response
against the policy $\polA{\excl i}$ that generated the influence
$\ifpiA{i}(\polA{\excl i})$.

\begin{theorem}\label{thm:V_IALM=00003DV_GFBRM}For a finite-horizon
POSG, the solution of the IALM for the incoming influence point $\ifpiA{i}(\polA{\excl i})$
associated with any $\polA{\excl i}$ achieves the same value $V_{i}(\ifpiA{i}(\polA{\excl i}))$,
given by \eqref{eq:V(b0)_inf}, as the best-response value $V_{i}(\polA{\excl i})$,
given by \eqref{eq:V(b0)_JESP}, computed against $\polA{\excl i}$
directly:
\begin{equation}
\forall_{\polA{\excl i}}\quad V_{i}(\ifpiA{i}(\polA{\excl i}))=V_{i}(\polA{\excl i}).\label{eq:thm:sufficiency}
\end{equation}
 \end{theorem}

We note that this results also holds in the presence of intra-stage
connections. To prove the result (in \sect\ref{sec:suff:proof-of-theorem})
we will show that the immediate reward terms and observation probabilities
are equal (\sect\ref{sec:suff:rewards-and-obs}). In turn, to show
this, we will need to show that transition probabilities are the same
given a local-form belief and a global-form belief, which means that
the local-form belief is a sufficient statistic to predict the next
local state (\sect\ref{sec:suff:local-transitions}). In order to
allow the rewriting to take place, we first show how the global-form
belief can be factorized.

We believe that this proof by itself is useful: it isolates the core
technical property that needs to hold for sufficiency in \lem\ref{lem:localTransitionSufficiency_Pxmxm}
in \sect\ref{sec:suff:local-transitions}. In this way it 1) conveys
insight into the nature of how abstraction of latent state factors
affects value, 2) provides a derivation that can be used to obtain
simplifications of the definition of influence (\dfn\ref{dfn:influence})
in simpler cases, and 3) provides a recipe of how to prove similar
results in problems which add even more complexities.

\subsection{Factorization of the Global-Form Belief}

\label{sec:Suffciency:factorization-of-global-form-belief}

In order to prove the equivalence of the GFBRM and the IALM, we will
show that their value functions are the same. In order to do that,
it will be necessary to decompose the global-form belief~$\gbA i$
in components.

To do that, we make use of the insight that, for any $\dsetAT i{\ts+1}$,
the law of total probability allows us to write
\begin{equation}
\gbA i(\sT{\ts},\aoHistAT{\excl i}{\ts})=\sum_{\dsetAT i{\ts+1}}\bA i(\left\langle \mfAT i{\ts},\nmfAT it\right\rangle ,\aoHistAT{\excl i}{\ts},\dsetAT i{\ts+1})=\sum_{\dsetAT i{\ts+1}}\bA i(\mfAT i{\ts},\dsetAT i{\ts+1})\bA i(\nmfAT it,\aoHistAT{\excl i}{\ts}|\mfAT i{\ts},\dsetAT i{\ts+1}).\label{eq:gfb_factored_Dt}
\end{equation}
(We drop the superscript `g' because we are rewriting to something
that we do not call global-form belief anymore.)

\noindent Also, it is important to remember that the belief is \emph{defined}
as 
\[
\gbA i(\sT{\ts},\aoHistAT{\excl i}{\ts})\defas\Pr(\sT{\ts},\aoHistAT{\excl i}{\ts}|\aoHistAT it,\bO,\jpolG{\excl i}),
\]
 which means that in \eqref{eq:gfb_factored_Dt}, the definitions
of the components are 
\begin{eqnarray}
\bA i(\mfAT i{\ts},\dsetAT i{\ts+1}) & \defas & \Pr(\mfAT i{\ts},\dsetAT i{\ts+1}|\aoHistAT it,\bO,\jpolG{\excl i}),\label{eq:b(fi,mi,Di)}\\
\bA i(\nmfAT it,\aoHistAT{\excl i}{\ts}|\mfAT i{\ts},\dsetAT i{\ts+1}) & \defas & \Pr(\nmfAT it,\aoHistAT{\excl i}{\ts}|\mfAT i{\ts},\dsetAT i{\ts+1},\aoHistAT it,\bO,\jpolG{\excl i}).\label{eq:b(fj,oj)}
\end{eqnarray}

These equations further clarify how to think about inclusion of actions
$\aA i$ and observations $\oA i$ inside the d-separating set $\dsetAT i{\ts+1}$:
the belief \emph{per definition} conditions on the history of actions
and observations, as such these can be included in $\dsetAT i{\ts+1}$
without further problems. In particular, suppose that $a_{i}^{k}$
is part of d-separating set $\dsetAT i{\ts+1}$, then this will lead
to $\Pr(\mfAT i{\ts},\left\langle \dots a_{i}^{k}\dots\right\rangle |\left\langle \dots a_{i}^{k}\dots\right\rangle ,\bO,\jpolG{\excl i})$
in \eqref{eq:b(fj,oj)}. However, the interpretation is simply that
this does not influence the probabilities, since $P(x|x)=1$. Similarly,
it would lead to a term $\Pr(\nmfAT it,\aoHistAT{\excl i}{\ts}|\mfAT i{\ts},\left\langle \dots a_{i}^{k}\dots\right\rangle ,\left\langle \dots a_{i}^{k}\dots\right\rangle ,\bO,\jpolG{\excl i})$
in \eqref{eq:b(fj,oj)}. Again, this poses no problem, since $\Pr(y|x,x)=\Pr(y|x)$.
However, let us repeat that we do need all observation relevant state
factors in the local state: otherwise we cannot define the local observation
model $\bar{O}_{i}$ and track the local-form belief $\bA i(\mfAT i{\ts},\dsetAT i{\ts+1})$
(cf. \dfn\ref{dfn:lfm} and \dfn\ref{dnf:IALM}). 

\subsection{Sufficiency for Prediction of Local State Transitions}

\label{sec:suff:local-transitions}

In this section, we show that the influence together with the local-form
belief is sufficient to predict local state transitions. We first
prove the following lemma, that shows that pairwise marginal distributions
over states are the same in the IALM and the GFBRM. This will then
be used in other proofs.

\begin{lemma}

\label{lem:localTransitionSufficiency_Pxmxm} 

The joint distribution over current local state and next local state
induced by a local-form belief is identical to that of the global-form
belief:

\begin{equation}
\forall_{\aoHistAT i{\ts}}\forall_{\mfAT i{\ts},\mfAT i{\ts+1}}\qquad\Pr(\mfAT i{\ts},\mfAT i{\ts+1}|\gbA i,\aAT it,\jpolG{\excl i})=\Pr(\mfAT i{\ts},\mfAT i{\ts+1}|\lbA i,\aAT it,\ifpiAT i{\ts+1}),
\end{equation}
where $\lbA i,\,\gbA i$ denote the for the local-form and global-form
beliefs induced by $\aoHistAT i{\ts}$.

\end{lemma} \begin{proof} To improve readability we will omit some
time indices that do not cause confusion. We assume arbitrary $\aoHistAT i{\ts},\mfAT i{\ts},\mfAT i{\ts+1}$,
and start with the left-hand side, which is given by \eqref{eq:P_xmxm_gfb}:
{\scriptsize
\begin{align}
 & \sum_{\nmfAT i{\ts}}\sum_{\jaG{\excl i}}\Pr(\mfAT i{\ts+1}|\sT{\ts},\aA i,\jaG{\excl i})\sum_{\aoHistAT{\excl i}{\ts}}\Pr(\jaG{\excl i}|\aoHistAT{\excl i}{\ts},\jpolG{\excl i})\gbA i(\sT{\ts},\aoHistAT{\excl i}{\ts})\nonumber \\
= & \text{\{via \eqref{eq:gfb_factored_Dt}\}}\nonumber \\
 & \sum_{\nmfAT i{\ts}}\sum_{\jaG{\excl i}}\Pr(\mfAT i{\ts+1}|\sT{\ts},\aA i,\jaG{\excl i})\sum_{\aoHistAT{\excl i}{\ts}}\Pr(\jaG{\excl i}|\aoHistAT{\excl i}{\ts},\jpolG{\excl i})\left[\sum_{\dsetAT i{\ts+1}}\bA i(\mfAT i{\ts},\dsetAT i{\ts+1})\bA i(\nmfAT it,\aoHistAT{\excl i}{\ts}|\mfAT i{\ts},\dsetAT i{\ts+1})\right]\\
= & \text{\{via \eqref{eq:P_xm_sa}\}}\nonumber \\
 & \sum_{\nmfAT i{\ts}}\sum_{\jaG{\excl i}}\left[\sum_{\nmfAT i{\ts+1}}\Pr(\nmfAT i{\ts+1},\mfAT i{\ts+1}|\sT{\ts},\aA i,\jaG{\excl i})\right]\sum_{\aoHistAT{\excl i}{\ts}}\Pr(\jaG{\excl i}|\aoHistAT{\excl i}{\ts},\jpolG{\excl i})\sum_{\dsetAT i{\ts+1}}\bA i(\mfAT i{\ts},\dsetAT i{\ts+1})\bA i(\nmfAT it,\aoHistAT{\excl i}{\ts}|\mfAT i{\ts},\dsetAT i{\ts+1})\\
= & \text{\{via \eqref{eq:P_xxm_sa__LFM_transitionProbs}\}}\nonumber \\
 & \sum_{\nmfAT i{\ts}}\sum_{\jaG{\excl i}}\sum_{\nmfAT i{\ts+1}}\Pr(\mflAT i{\ts+1}|\mfAT i{\ts},\aA i,\mfnAT i{\ts+1})\Pr(\mfnAT i{\ts+1}|\mfAT i{\ts},\mflAT i{\ts+1},\aA i,\nmfAT u{\ts},\nmfAT u{\ts+1},\jaG u)\Pr(\nmfAT i{\ts+1}|\mfAT i{\ts},\nmfAT i{\ts},\aA i,\jaG{\excl i},\mfAT i{\ts+1})\nonumber \\
 & \sum_{\aoHistAT{\excl i}{\ts}}\Pr(\jaG{\excl i}|\aoHistAT{\excl i}{\ts},\jpolG{\excl i})\sum_{\dsetAT i{\ts+1}}\bA i(\mfAT i{\ts},\dsetAT i{\ts+1})\bA i(\nmfAT it,\aoHistAT{\excl i}{\ts}|\mfAT i{\ts},\dsetAT i{\ts+1})\\
= & \text{\{reordering terms\}}\nonumber \\
 & \sum_{\dsetAT i{\ts+1}}\Pr(\mflAT i{\ts+1}|\mfAT i{\ts},\aA i,\mfnAT i{\ts+1})\bA i(\mfAT i{\ts},\dsetAT i{\ts+1})\nonumber \\
 & \left[\sum_{\jaG{\excl i}}\sum_{\aoHistAT{\excl i}{\ts}}\sum_{\nmfAT i{\ts}}\sum_{\nmfAT i{\ts+1}}\Pr(\mfnAT i{\ts+1}|\mfAT i{\ts},\mflAT i{\ts+1},\aA i,\nmfAT u{\ts},\nmfAT u{\ts+1},\jaG u)\Pr(\nmfAT i{\ts+1}|\mfAT i{\ts},\nmfAT i{\ts},\aA i,\jaG{\excl i},\mfAT i{\ts+1})\Pr(\jaG{\excl i}|\aoHistAT{\excl i}{\ts},\jpolG{\excl i})\bA i(\nmfAT it,\aoHistAT{\excl i}{\ts}|\mfAT i{\ts},\dsetAT i{\ts+1})\right]\label{eq:eq-with-bracketed-part}
\end{align}
}

This equation has grouped together all the probabilities that are
affected by the non-local part of the problem in between the brackets.
The terms before do not depend on the external part at all. We will
now further investigate the externally influenced (bracketed) part:

{\scriptsize

\begin{align}
 & \sum_{\jaG{\excl i}}\sum_{\aoHistAT{\excl i}{\ts}}\sum_{\nmfAT i{\ts}}\sum_{\nmfAT i{\ts+1}}\Pr(\mfnAT i{\ts+1}|\mfAT i{\ts},\mflAT i{\ts+1},\aA i,\ifsAT i{\ts+1})\Pr(\nmfAT i{\ts+1}|\mfAT i{\ts},\nmfAT i{\ts},\aA i,\jaG{\excl i},\mfAT i{\ts+1})\jpolG{\excl i}(\jaG{\excl i}|\aoHistAT{\excl i}{\ts})\bA i(\nmfAT it,\aoHistAT{\excl i}{\ts}|\mfAT i{\ts},\dsetAT i{\ts+1})\\
= & \sum_{\jaG{\excl i}}\sum_{\nmfAT i{\ts}}\sum_{\nmfAT i{\ts+1}}\Pr(\mfnAT i{\ts+1}|\mfAT i{\ts},\mflAT i{\ts+1},\aA i,\ifsAT i{\ts+1})\Pr(\nmfAT i{\ts+1}|\mfAT i{\ts},\nmfAT i{\ts},\aA i,\jaG{\excl i},\mfAT i{\ts+1})\sum_{\aoHistAT{\excl i}{\ts}}\jpolG{\excl i}(\jaG{\excl i}|\aoHistAT{\excl i}{\ts})\bA i(\nmfAT it,\aoHistAT{\excl i}{\ts}|\mfAT i{\ts},\dsetAT i{\ts+1})
\end{align}
}In this equation, not all non-modeled factors $\nmfAT i{\ts+1}$
are relevant: we can restrict to the intra-stage sources $\nmfAT u{\ts+1}$
and their intra-stage ancestors $\nmfAT v{\ts+1}$, other factor's
probabilities just sum to 1. This yields: {\scriptsize

\begin{align}
 & \sum_{\jaG{\excl i}}\sum_{\nmfAT i{\ts}}\sum_{\nmfAT u{\ts+1}}\Pr(\mfnAT i{\ts+1}|\mfAT i{\ts},\mflAT i{\ts+1},\aA i,\ifsAT i{\ts+1})\sum_{\nmfAT v{\ts+1}}\Pr(\nmfAT u{\ts+1},\nmfAT v{\ts+1}|\mfAT i{\ts},\nmfAT i{\ts},\aA i,\jaG{\excl i},\mfAT i{\ts+1})\sum_{\aoHistAT{\excl i}{\ts}}\jpolG{\excl i}(\jaG{\excl i}|\aoHistAT{\excl i}{\ts})\bA i(\nmfAT it,\aoHistAT{\excl i}{\ts}|\mfAT i{\ts},\dsetAT i{\ts+1})\\
= & \text{\{restricting to \ensuremath{\jaG v,\mfAT v{\ts+1}} that actually influence \ensuremath{\ensuremath{\nmfAT u{\ts+1}}}. I.e.,\,\ensuremath{v}\, denotes other `second order' sources\}}\nonumber \\
 & \sum_{\jaG{\excl i}}\sum_{\nmfAT i{\ts}}\sum_{\nmfAT u{\ts+1}}\Pr(\mfnAT i{\ts+1}|\mfAT i{\ts},\mflAT i{\ts+1},\aA i,\ifsAT i{\ts+1})\sum_{\nmfAT v{\ts+1}}\Pr(\nmfAT u{\ts+1},\nmfAT v{\ts+1}|\mfAT v{\ts},\nmfAT v{\ts},\aA i,\jaG v,\mfAT v{\ts+1})\sum_{\aoHistAT{\excl i}{\ts}}\jpolG{\excl i}(\jaG{\excl i}|\aoHistAT{\excl i}{\ts})\bA i(\nmfAT it,\aoHistAT{\excl i}{\ts}|\mfAT i{\ts},\dsetAT i{\ts+1})\nonumber \\
= & \text{\{pushing in summations, recall \ensuremath{\ifsAT i{\ts+1}=\left\langle \nmfAT u{\ts},\jaG u,\nmfAT u{\ts+1}\right\rangle }\}}\nonumber \\
 & \sum_{\ifsAT i{\ts+1}}\Pr(\mfnAT i{\ts+1}|\mfAT i{\ts},\mflAT i{\ts+1},\aA i,\ifsAT i{\ts+1})\sum_{\jaG v}\sum_{\nmfAT v{\ts}}\sum_{\nmfAT v{\ts+1}}\Pr(\nmfAT u{\ts+1},\nmfAT v{\ts+1}|\mfAT v{\ts},\nmfAT v{\ts},\aA i,\jaG v,\mfAT v{\ts+1})\sum_{\aoHistAT{\excl i}{\ts}}\jpolG{\excl i}(\jaG{\excl i}|\aoHistAT{\excl i}{\ts})\bA i(\nmfAT it,\aoHistAT{\excl i}{\ts}|\mfAT i{\ts},\dsetAT i{\ts+1})\\
= & \text{\{marginalize out non-relevant terms\}}\nonumber \\
 & \sum_{\ifsAT i{\ts+1}}\Pr(\mfnAT i{\ts+1}|\mfAT i{\ts},\mflAT i{\ts+1},\aA i,\ifsAT i{\ts+1})\sum_{\jaG v}\sum_{\nmfAT v{\ts}}\sum_{\nmfAT v{\ts+1}}\Pr(\nmfAT u{\ts+1},\nmfAT v{\ts+1}|\mfAT v{\ts},\nmfAT v{\ts},\aA i,\jaG v,\mfAT v{\ts+1})\nonumber \\
 & \sum_{\aoHistAT u{\ts}}\sum_{\aoHistAT v{\ts}}\jpolG u(\jaGT ut|\aoHistAT u{\ts})\jpolG v(\jaGT vt|\aoHistAT v{\ts})\bA i(\nmfAT ut,\nmfAT vt,\aoHistAT u{\ts},\aoHistAT v{\ts}|\mfAT i{\ts},\dsetAT i{\ts+1})\\
= & \text{\{let \ensuremath{w=u\cup v} denote the union of direct and indirect sources\}}\nonumber \\
 & \sum_{\ifsAT i{\ts+1}}\Pr(\mfnAT i{\ts+1}|\mfAT i{\ts},\mflAT i{\ts+1},\aA i,\ifsAT i{\ts+1})\sum_{\jaG v}\sum_{\nmfAT v{\ts}}\sum_{\nmfAT v{\ts+1}}\Pr(\nmfAT u{\ts+1},\nmfAT v{\ts+1}|\mfAT v{\ts},\nmfAT v{\ts},\aA i,\jaG v,\mfAT v{\ts+1})\sum_{\aoHistAT w{\ts}}\jpolG w(\jaG w|\aoHistAT w{\ts},\jpolG w)\bA i(\nmfAT wt,\aoHistAT w{\ts}|\mfAT i{\ts},\dsetAT i{\ts+1})\\
= & \text{\{since \ensuremath{\bA i(\nmfAT wt,\aoHistAT w{\ts}|\mfAT i{\ts},\dsetAT i{\ts+1})\defas\Pr(\nmfAT wt,\aoHistAT w{\ts}|\mfAT i{\ts},\dsetAT i{\ts+1},\aoHistAT it,\bO,\jpolG{\excl i})\overset{\text{\{def. of d-set \eqref{eq:dset-def--intrastage}}\}}{=}\Pr(\nmfAT wt,\aoHistAT w{\ts}|\dsetAT i{\ts+1},\bO,\jpolG{\excl i})} \}}\nonumber \\
\nonumber \\
 & \sum_{\ifsAT i{\ts+1}}\Pr(\mfnAT i{\ts+1}|\mfAT i{\ts},\mflAT i{\ts+1},\aA i,\ifsAT i{\ts+1})\sum_{\left\langle \nmfAT v{\ts},\jaG v,\nmfAT v{\ts+1}\right\rangle }\Pr(\nmfAT u{\ts+1},\nmfAT v{\ts+1}|\mfAT v{\ts},\nmfAT v{\ts},\aA i,\jaG v,\mfAT v{\ts+1})\sum_{\aoHistAT w{\ts}}\jpolG w(\jaG w|\aoHistAT w{\ts})\Pr(\nmfAT wt,\aoHistAT w{\ts}|\dsetAT i{\ts+1},\bO,\jpolG{\excl i}).\label{eq:ExternalPart_preInfluence}
\end{align}
}%end small

We can now apply the definition of influence (\dfn\ref{dfn:influence}
\vpageref{dfn:influence}) to \eqref{eq:ExternalPart_preInfluence},
which yields
\begin{equation}
=\sum_{\ifsAT i{\ts+1}=\left\langle \nmfAT u{\ts},\jaG u,\nmfAT u{\ts+1}\right\rangle }\Pr(\mfnAT i{\ts+1}|\mfAT i{\ts},\mflAT i{\ts+1},\aA i,\nmfAT u{\ts},\jaG u,\nmfAT u{\ts+1})\iffunc(\ifsAT i{\ts+1}|\dsetAT i{\ts+1},\mfAT v{\ts},\aA i,\mfAT v{\ts+1}),\label{eq:ExternalPart_postInfluence}
\end{equation}
which is the definition \eqref{eq:P(NLAFS)__general-IS} of $\Pr(\mfnAT i{\ts+1}|\langle\mfAT i{\ts},\dsetAT i{\ts+1}\rangle,\mflAT i{\ts+1},\aA i,\ifpiAT i{\ts+1})$.
\textbf{}

Substituting \eqref{eq:ExternalPart_postInfluence} back in \eqref{eq:eq-with-bracketed-part}
we get

\begin{align}
 & \sum_{\dsetAT i{\ts+1}}\Pr(\mflAT i{\ts+1}|\mfAT i{\ts},\mfAT i{\ts+1},\mfnAT i{\ts+1},\aA i)\bA i(\mfAT i{\ts},\dsetAT i{\ts+1})\left[\Pr(\mfnAT i{\ts+1}|\langle\mfAT i{\ts},\dsetAT i{\ts+1}\rangle,\mflAT i{\ts+1},\aA i,\ifpiAT i{\ts+1})\right]\nonumber \\
= & \text{\{via \ref{eq:IALM-T-IS} \}}\nonumber \\
 & \sum_{\dsetAT i{\ts+1}}\Pr(\mfAT i{\ts+1}|\mfAT i{\ts},\dsetAT i{\ts+1},\aA i,\ifpiAT i{\ts+1})\bA i(\mfAT i{\ts},\dsetAT i{\ts+1})\overset{\text{\{via \eqref{eq:P_xmxm_lfb}\}}}{=}\Pr(\mfAT i{\ts},\mfAT i{\ts+1}|\lbA i,\aAT it,\ifpiAT i{\ts+1}),
\end{align}
which concludes the proof.\end{proof}

\begin{lemma} 

\label{lem:localTransitionSufficiency_Pxm}

A local-form belief is a sufficient statistic for predicting the next
local state. That is, when $\lbA i,\,\gbA i$ denote the for the local-form
and global-form beliefs induced by the same action-observation history
$\aoHistAT i{\ts}$, we have that:

\begin{equation}
\forall_{\aoHistAT i{\ts}}\forall_{\mfAT i{\ts+1}}\qquad\Pr(\mfAT i{\ts+1}|\gbA i,\aAT i\ts,\jpolG{\excl i})=\Pr(\mfAT i{\ts+1}|\lbA i,\aAT i\ts,\ifpiAT i{\ts+1}).\label{eq:lem:local_state_prediction_sufficiency}
\end{equation}
\end{lemma}\begin{proof} This follows directly from \lem\eqref{lem:localTransitionSufficiency_Pxmxm}:

\begin{multline*}
\Pr(\mfAT i{\ts+1}|\lbA i,\aAT it,\ifpiAT i{\ts+1})=\sum_{\mfAT i{\ts}}\Pr(\mfAT i{\ts},\mfAT i{\ts+1}|\lbA i,\aAT it,\ifpiAT i{\ts+1})\\
=\sum_{\mfAT i{\ts}}\Pr(\mfAT i{\ts},\mfAT i{\ts+1}|\gbA i,\aAT it,\jpolG{\excl i})=\Pr(\mfAT i{\ts+1}|\gbA i,\aAT it,\jpolG{\excl i}).\qedhere
\end{multline*}
\end{proof}

\subsection{Sufficiency for Predicting Rewards and Observations}

\label{sec:suff:rewards-and-obs}

Given that we established that local-form beliefs in an IALM are sufficient
to predict local-state transitions, we can now also establish their
sufficiency for predicting rewards and observations.

\begin{lemma}\label{lem:RewardSufficiency} The local-form belief
is a sufficient statistic to predict the immediate reward. That is

\begin{equation}
\forall_{\aoHistA i}\forall_{\aAT i\ts}\qquad\RA i(\gbA i,\aAT i\ts)=\RA i(\lbA i,\aAT i\ts)\label{eq:RewardEquality}
\end{equation}
where $\lbA i,\,\gbA i$ denote the for the local-form and global-form
beliefs induced by $\aoHistA i$. \end{lemma} \begin{proof} \sloppy
When we compare equations \eqref{eq:R(gfb,a)__LFM} and \eqref{eq:R(lfb,a)},
we see that this holds if $\Pr(\mfAT i{\ts},\mfAT i{\ts+1}|\gbA i,\aAT it,\jpolG{\excl i})=\Pr(\mfAT i{\ts},\mfAT i{\ts+1}|\lbA i,\aAT it,\ifpiAT i{\ts+1})$.
This is precisely what \lem\ref{lem:localTransitionSufficiency_Pxmxm}
shows.\end{proof}

\begin{lemma}\label{lem:observation_suficiency} The local-form belief
is a sufficient statistic for predicting the observation. That is:
\begin{equation}
\forall_{\aoHistA i}\forall_{\aAT i\ts,\oAT i{\ts+1}}\qquad\Pr(\oAT i{\ts+1}|\gbA i,\aAT i\ts)=\Pr(\oAT i{\ts+1}|\lbA i,\aAT i\ts),\label{eq:observationEquality}
\end{equation}
where $\lbA i,\,\gbA i$ denote the for the local-form and global-form
beliefs induced by $\aoHistA i$.\end{lemma}\begin{proof} \sloppy Comparing
equations \eqref{eq:P(o|gfb,a)__LFM} and \eqref{eq:P(o|lfb,a)},
we see that equality holds if $\Pr(\mfAT i{\ts+1}|\gbA i,\aA i,\jpolG{\excl i})=\Pr(\mfAT i{\ts+1}|\lbA i,\aA i,\ifpiAT i{\ts+1})$;
this is exactly what \lem\ref{lem:localTransitionSufficiency_Pxm}
shows.\end{proof}

\subsection{Proof of Theorem~\ref{thm:V_IALM=00003DV_GFBRM}: Sufficiency for
Predicting Optimal Value}

\label{sec:suff:proof-of-theorem}

Finally, we can prove that our definition of influence is sufficient
to predicting the optimal best-response value. The values in \eqref{eq:thm:sufficiency}
are defined as the value of the initial beliefs, cf. equations \eqref{eq:V(b0)_inf}
and \eqref{eq:V(b0)_JESP}. Putting this all together, we need to
show that 
\begin{equation}
\forall_{\polA{\excl i}}\quad\VA i(\ifpiA{i}(\polA{\excl i}))\defas\VAT i0(\lbAT i0)=\VAT i0(\gbAT i0)\defas\VA i(\polA{\excl i}).
\end{equation}
 The proof is by induction over the horizon, where the base case is
given by the last stage.

\paragraph*{Base Case.}

Assume an arbitrary last-stage AOH, $\aoHistAT i{\h-1}$, and let
$\lbA i,\,\gbA i$ denote the for the local-form and global-form beliefs
induced by it. Their respective values are given by 
\[
\VAT i\ts(\gbA i)=\max_{\aA i}\RA i(\gbA i,\aA i),
\]
\[
\VAT i\ts(\lbA i)=\max_{\aA i}\RA i(\lbA i,\aA i).
\]
So we need to show that the predicted immediate rewards are equal.
This, however, is exactly what Lemma~\ref{lem:RewardSufficiency}
shows.

\paragraph*{Induction Step.}

The induction hypothesis is that, for stage $\ts+1$, 
\begin{equation}
\forall_{\aoHistAT i{\ts+1}}\qquad\V{}_{i}^{\ts+1}(\lbAT i{\ts+1})=\V{}_{i}^{\ts+1}(\gbAT i{\ts+1}),
\end{equation}
where we write $\lbAT i{\ts+1},\,\gbAT i{\ts+1}$ are the local-form
and global-form beliefs induced by $\aoHistAT i{\ts+1}$.

Now we need to prove that $\V{}_{i}^{\ts}(\lbA i)=\V{}_{i}^{\ts}(\gbA i)$,
for all $\aoHistAT i{\ts}$. Since, per definition,
\[
\VAT i\ts(\lbA i)=\max_{\aA i}\QAT i\ts(\lbA i,\aA i),
\]
\[
\VAT i\ts(\gbA i)=\max_{\aA i}\QAT i\ts(\gbA i,\aA i),
\]
we will show this by proving that the Q-values are equal. Assume an
arbitrary $\aoHistAT i{\ts}$. Its Q-values, for all $\aA i$, are
given by \eqref{eq:Q(b,a)__GFBRM}:

\begin{equation}
\QAT i\ts(\gbA i,\aA i)=\RA i(\gbA i,\aA i)+\discount\sum_{\oA i}\Pr(\oA i|\gbA i,\aA i)\VAT{i}{\ts+1}(BU(\gbA i,\aA i,\oA i))
\end{equation}
By the induction hypothesis, we get
\begin{equation}
\QAT i\ts(\gbA i,\aA i)=\RA i(\gbA i,\aA i)+\sum_{\oA i}\Pr(\oA i|\gbA i,\aA i)\VAT{i}{\ts+1}(BU(\lbA i,\aA i,\oA i)).\label{eq:thm-proof:inductionstep:Q(bg)}
\end{equation}
Note that $BU(\gbA i,\aA i,\oA i)$ and $BU(\lbA i,\aA i,\oA i)$
are the local-form and global-form beliefs induced by the same next-stage
history $\aoHistAT i{\ts+1}=(\aoHistAT i{\ts},\aA i,\oA i)$, and
hence the induction hypothesis applies: 
\[
\VAT{i}{\ts+1}(BU(\gbA i,\aA i,\oA i))=\VAT{i}{\ts+1}(BU(\lbA i,\aA i,\oA i)).
\]

So, in order to show that \eqref{eq:thm-proof:inductionstep:Q(bg)}
is equal to 
\begin{equation}
\QAT i\ts(\lbA i,\aA i)=\RA i(\lbA i,\aA i)+\sum_{\oA i}\Pr(\oA i|\lbA i,\aA i)\VAT{i}{\ts+1}(BU(\lbA i,\aA i,\oA i))
\end{equation}
we need to show equality for both the immediate rewards, $\RA i(\gbA i,\aA i)=\RA i(\lbA i,\aA i)$,
and the observation probabilities, $\Pr(\oA i|\gbA i,\aA i)=\Pr(\oA i|\lbA i,\aA i)$.
The former was shown in \lem\ref{lem:RewardSufficiency} and the
latter was shown in \lem\ref{lem:observation_suficiency}. Hence,
the Q-values are the same, hence the values are the same, which concludes
the induction step. \qed

\section{Tractable Influence Representations}

\label{sec:sub-classes-with-Compact-representations}

There are a number of important problem classes and associated models
developed in previous work that emphasize weakly coupled problem structure
in more restrictive settings. We now reformulate these classes in
the context of IBA, thus demonstrating how the theory presented in
this paper unifies such previous work in a coherent graphical framework.
All of the models that we review below are specialized instances of
the factored Dec-POMDP (fDec-POMDP) model and since an fDec-POMDP
is an fPOSG, our definition of influence is applicable to all of these
models.

However, as we illuminate below, some sub-classes allow for particularly
compact influence specifications that can be computed efficiently.
Similar to the examples in \sect\ref{sec:IBA-by-Example}, this makes
clear how it is possible to compute best responses more effectively,
and provides some intuition about how influence search approaches
can enable speed-ups in these sub-classes.

We will also see how our unified perspective allows us to make novel
observations about these previously defined classes that can lead
to improvements and extensions. For instance, we will see that we
can derive more compact forms of influence for the so-called EDI-Dec-MDP
framework. In general we expect that the more compact the representation,
the more efficiently these sub-classes can be solved. However, we
note that, unlike (most of) the papers that introduced these sub-classes,
in this paper we are not proposing an influence\textendash search
technique to solve the optimization for all agents. This is left for
future work.

\subsection{TD-POMDP}

An earlier embodiment of influence abstraction ~\citep{Witwicki10ICAPS,Witwicki11PhD,Witwicki12AAMAS}
sought to exploit cooperative agents' weak coupling, showing that
searching in the space of joint influences can provide significant
speed-ups over searching the space of joint policies for a restrictive
sub-class of fDec-POMDPs. The so-called \emph{Transition-Decoupled
POMDP (TD-POMDP) }\citep{Witwicki11PhD} describes a local state for
each agent that resembles our local form models. However, it also
distinguishes so-called \emph{mutually-modeled factors (MMFs)} common
to more than one agent's local state. These MMFs have the same role
as our non-locally affected factors (NLAFs), but impose additional
restrictions \citep[Section 3.4.3]{Witwicki11PhD}. Specifically,
there are two important differences that make the TD-POMDP more restrictive
than our local-form model:
\begin{enumerate}
\item The TD-POMDP does not allow intra-stage dependencies between private
state variables and MMFs. 
\item In a TD-POMDP each state factor can only be directly affected by (have
an incoming edge from) the action (or private state variable) of just
one agent. 
\end{enumerate}
These constraints effectively limit the representational power of
the TD-POMDP to \textit{non-concurrent} interactions. As an example,
the \problemName{planetary exploration} domain from \fig\ref{fig:dset-rover}
can be directly modeled as a TD-POMDP by making $pl$ the single MMF
in the model. In contrast, the \<housesearch> problem from \fig\ref{fig:house-dbn}
cannot be modeled in the same way: the TD-POMDP version of this problem
requires separating the `found' variable into two MMF variables:
`found by agent~1' and `found by agent~2' thus increasing the
size of the local problems \citep{Witwicki12AAMAS}.

\begin{observation} The TD-POMDP model is a special case of the LFM,
imposing restrictions that limit its modeling capabilities to a subset
of those interactions representable as local-form POSGs: $\text{TD-POMDP}\subset\text{LFM}$.
\end{observation}

The TD-POMDP's formalization is less flexible than that proposed
in this paper. In particular, it seems difficult to extend the TD-POMDP
to deal with intra-stage connections, which we have argued in Sections~\ref{sec:IBA:def-of-influence:links-sources-destinations}
and \ref{sec:IBA-with-IS-deps} is important for expressiveness.

However, the authors derive that this representational restriction
affords the TD-POMDP a particular form of influence, since the history
of mutually-modeled factors is guaranteed to d-separate an agent\textquoteright s
observations from all external factors (i.e., those outside of its
local state). The form of influence that \citeauthor{Witwicki10ICAPS}
propose for TD-POMDPs actually corresponds to our notion of `induced
CPT' (cf. \sect\ref{sec:inducedCPT}) or the marginal of their product
\eqref{eq:P(NLAFS)__general}. In many cases this allows for compact
representations of the influence. Compact influence representations
in turn appear to provide traction when it comes to computing solutions,
as evidenced by the efficiency and scalability gains of influence-space
search for TD-POMDPs~\citep{Witwicki11PhD,Witwicki12AAMAS}.

\subsection{TI-Dec-MDP}

Another model, the Transition-Independent Dec-MDP \citep{Becker03AAMAS},
imposes other more stringent restrictions on the dependencies between
agents' local models. In particular, an agent fully observes its
private factors and there are no paths of dependence in the DBN connecting
one agent\textquoteright s private factors, actions, and observation
to those of another. This implies that the agents are \textit{transition
and observation independent}. Agents' local models are instead coupled
through their rewards, which can depend on the \textit{events} $e_{i}=\langle s_{i}^{t},a_{i}^{t},s_{i}^{t+1}\rangle$
that occur (at most once) within another agents' state space.

\begin{figure}[tb]
\centering\input{figs/frag_cf_TOI.tex} \includegraphics[scale=0.39]{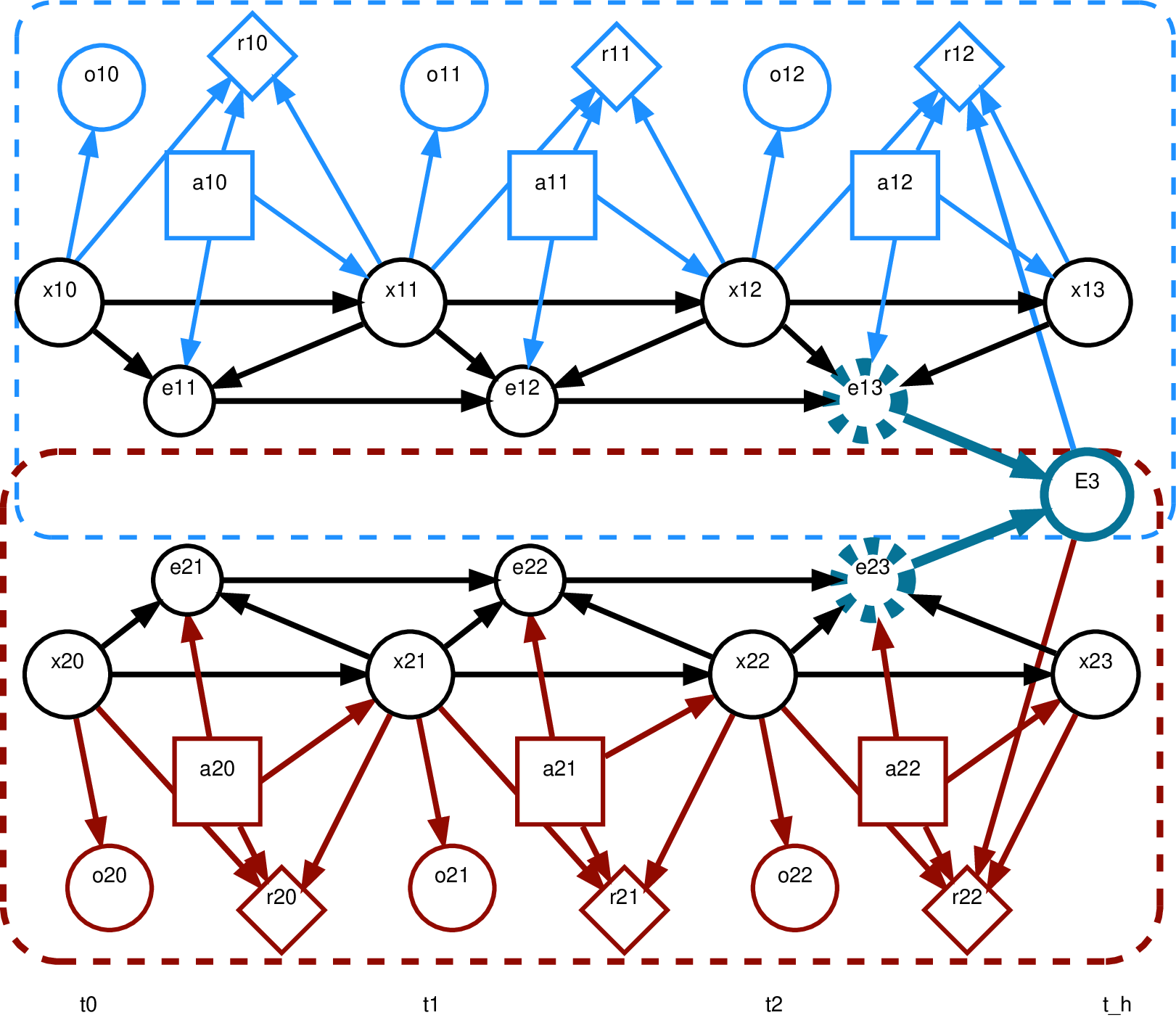} 

\caption{Local-form representation of TI-Dec-MDP: the LFMs of both agents (indicated
by different color bounding boxes) are tied together by a influence
source $E$ that indicates if the joint event happened. To be able
to predict the value of this influence source, the agent $i$ will
need to condition on their $e_{i}$ variable.}

\label{fig:subclasses:TOIDecMDP}
\end{figure}

This class of problems includes, for instance, missions executed by
Mars rovers during which they need to collect samples at various sites.
In such settings, it is reasonable to assume that the rovers have
their own routes and therefore will not affect each other's transitions.
However, the utility of rover 1 taking a soil sample at a particular
site might depend on what samples are taken by rover 2 at nearby sites.

For instance, imagine that the rovers pass at different sides of a
canyon, taking a picture of this canyon provides some utility, but
if both rovers take a picture from their side of the canyon (corresponding
to the individual events $e_{i}$), this may enable a better 3D reconstruction,
providing more value than just the sum of two individual pictures.

This can be easily captured in a factored representation as shown
in \fig\ref{fig:subclasses:TOIDecMDP}. It shows that the combined
occurrence of both agents\textquoteright{} events (as represented
by Boolean variable $E$) leads to a change in the reward (split between
the agents as soon as the event occurs). When the discount factor
is 1 (as \citealt{Becker03AAMAS} assume), the reward may as well
be affected at the last time step as we have indicated. This leads
to a very simple form of influence $\ifpiA i(e_{j}^{h-1})$ corresponding
to the probability of $e_{j}^{h-1}$ being true. This corresponds
exactly to the characterization of `parameter space' presented by
Becker et al.\ in their development of the \textit{coverage set algorithm~(CSA)}.

Our characterization of the TI-Dec-MDP immediately leads to some new
insights.

\begin{observation} While the TI-Dec-MDP framework is arguably more
restrictive than the TD-POMDP, the graphical structure in \fig\ref{fig:subclasses:TOIDecMDP}
makes clear that a TI-Dec-MDP is not a TD-POMDP: $E^{3}$ is affected
from both sub-problems simultaneously.\end{observation}

\begin{observation} The properties that 1) events cannot occur more
than once; and 2) events are unobserved, allow for history-independent
influence encoding in TI-Dec-MDPs. \end{observation}

\begin{observation} CSA and closely-related TI-Dec-MDP algorithms
\citep{Petrik09JAIR} exploit structure that is also present in more
general contexts, such as TI-Dec-POMDPs with partial observability
of private factors. \end{observation}

That is, we make the observation that that CSA and its successors
can actually be extended to more general problem whose joint value
function is piecewise linear and convex in the influence parameters,
such as settings where agents receive only partial observations of
their local states. 

\subsection{Event-Driven Interactions}

The TI-Dec-MDPs assumes that transitions are independent, but interactions
are present in rewards. However, in many problems it may be the other
way around: for instance, the rewards that a vacuum cleaner robot
generates only depends on the amount of dirt it cleans up, but it
cannot enter a dirty room until a general purpose house-hold robot
opens the door.

\begin{figure}
\centering\input{figs/frag_cf_TOI.tex}\small \includegraphics[scale=0.39]{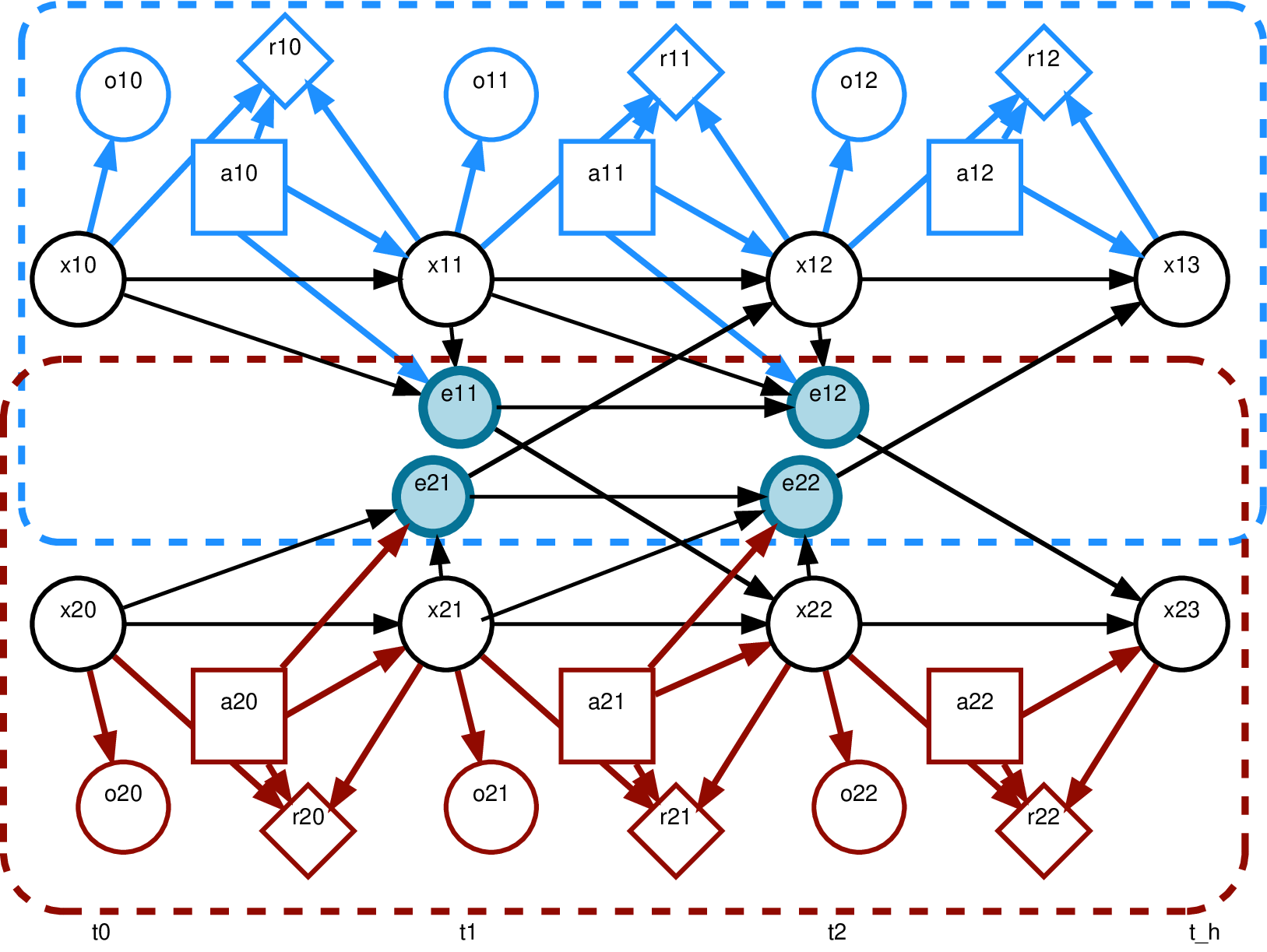}\caption{Local-form representation of the ED-Dec-MDP. The events $e_{1}$ act
as influence destination for agent 2 and vice versa. To avoid clutter
we do not indicate influence sources. The history of all $e_{i}$
serves as a d-separating set for both agents.}

\label{fig:subclasses:ED}

\end{figure}

To deal with such problems \citet{Becker04AAMAS} proposed the \emph{Dec-MDP
with Event-Driven Interaction (EDI-Dec-MDP)}, which provides an explicit
representation for structured transition dependencies between two
agents. Again, we will interpret this model in the IBA framework,
as illustrated in \fig\ref{fig:subclasses:ED}. This figure shows
that, agent $1$'s transition probabilities may be affected by the
prior occurrence of agent~2\textquoteright s event $e_{2}$ (and
vice versa). In this case, the history of these event features are
sufficient for d-separation, i.e., $\dsetA1=\{e_{1},e_{2}\}$. This
leads us to develop a new influence specification for this sub-class
of problems:

\begin{observation} The (induced-CPT form of) influence on EDI-Dec-MDP
agent~$i$, $\ifpiAT i{t+1}(\polA j)$ can be defined as $I(\sAT j\ts,\aAT j{\ts},\sAT j{\ts+1}|\vec{e}_{i}^{\,t},\vec{e}_{j}^{\,t})$.
Moreover, similar to what we saw in \sect\ref{sec:IBA-by-Example},
the history $\vec{e}_{i}^{\,t},\vec{e}_{j}^{\,t}$ can be represented
compactly since events can only switch to true. \end{observation}

The marginal of product of induced CPTs $Pr(e_{1}^{t+1},e_{2}^{t+1}|\vec{e}_{1}^{\,t},\vec{e}_{2}^{\,t})$
is similar to the parameters used by Becker \textit{et al.\ }(in
their application of CSA), but is slightly more compact, since it
does not depend on private factors $s_{i}^{t}$, which our theory
suggests to be unnecessary.  

Having derived a more compact parameter form, we anticipate that this
will translate directly into a more efficient application of CSA.
We note that our reformulation of the TI-Dec-MDP and EDI-Dec-MDP also
serve as influence specifications for the EDI-CR model \citep{Mostafa09WIIAT},
developed to include both event-driven interactions and reward dependencies
(as in the TI-Dec-MDP).

The \emph{distributed POMDP with coordination locales (DPCL) }model~\citep{Varakantham09ICAPS}
can also be reinterpreted using \fig\ref{fig:subclasses:ED}. This
model assumes all agents' observations are conditionally independent
given the state, but that in some specific states, agents can affect
each other's transitions or rewards. Looking at \fig\ref{fig:subclasses:ED},
the events $e_{i}$ precisely can model what \citeauthor{Varakantham09ICAPS}\
refer to as \emph{future-time coordination locales }(``situations
where actions of one agent impact actions of others in the future'').
\citeauthor{Varakantham09ICAPS} also consider \emph{same-time} coordination
locales, which can model simultaneous effects such as robots failing
to move when both try to move to the same grid cell. In \fig\ref{fig:subclasses:ED}
this would be captured by adding arrows from $\vec{e}_{i}^{\,t}$
to $\sAT i{\ts+1}$ (or alternatively by introducing joint events
$E$, as in \fig\ref{fig:subclasses:TOIDecMDP}, at every time step).
While these same-time coordination locales overcome the modeling requirements
of non-concurrency as observed in TD-POMDPs and ED-MDPs, the solution
method proposed by \citeauthor{Varakantham09ICAPS} is heuristic.
In fact, it is precisely our definition of influence presented in
this paper that explains how to deal with such concurrent interaction
in a principled fashion.

\subsection{ND-POMDP}

\begin{figure}[tb]
\begin{centering}
\input{figs/frag_cf_TOI.tex} \includegraphics[scale=0.38]{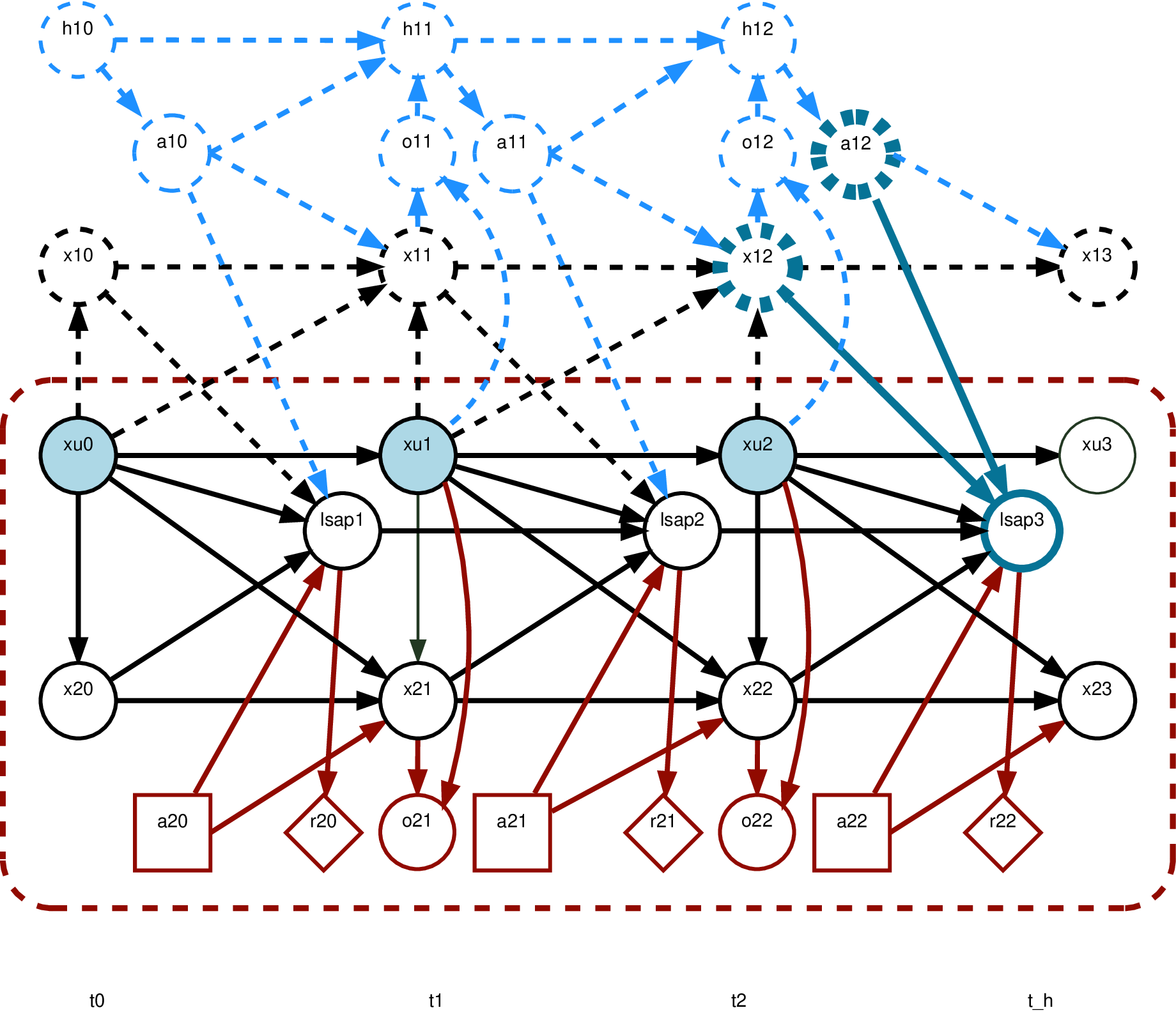}
\par\end{centering}
\centering{}\caption{Local-form representation of agent 2 in a two-agent ND-POMDP. Highlighted
are the influence sources for the ``neighborhood state-action variable''
$\argsT z3$. (Time indices are omitted in the figure to avoid clutter,
but can be inferred from the stages indicated at the bottom.) Note
that because $z$ itself nor any of its descendants are observable
(rewards are not observed in Dec-POMDPs), it does not open a path
of influence to $\sA2$. Therefore only the history of $\sA0$ needs
to be encoded in the d-separating set.}
\label{fig:subclasses:NDPOMDP}
\end{figure}

The \emph{Network Distributed POMDP (ND-POMDP) }introduced by \citet{Nair05AAAI}
is another transition and observation independent model whose structure
can easily be represented in our framework. It was motivated by problems
like sensor networks for intrusion detection, where the sensors need
to select actions to scan their local surroundings. Such actions do
not affect the local state of other sensors, but combinations of actions
of neighboring sensor nodes can lead to higher rewards (e.g., if two
adjacent sensors scan the same area where an intruder is, there might
be a higher detection probability).

The formalization is depicted in \fig\ref{fig:subclasses:NDPOMDP}.
Each agent\textquoteright s observation can be affected by an unaffectable,
mutually-modeled factor $s_{0}$ (e.g., the location of an intruder).
The reward dependencies involving joint actions are captured with
an unobservable variable $z$ encoding the local state-action pair
that in much the same way as did the TI-Dec-MDP's events. The difference
is that these joint actions are not constrained to occur only once,
and may affect the rewards at any time.

In general, an ND-POMDP can consist of multiple local neighborhoods,
which can be modeled using a coordination (hyper-)graph~\citep{Guestrin01NIPS,Nair05AAAI,Kok06JMLR}.
Agents correspond to nodes, while $\edS$ is a set of (hyper-)edges
corresponding to subsets, $\ed,$ of agents. To encode the interactions
between different subsets of agents $\ed\in\edS$, one can introduce
different variables $z_{\ed}$. Our reformulation presented here immediately
leads to the first specification of influence that we are aware of
for this problem class. Let us write $N(i)$ to denote the neighbors
of agent~$i$ excluding agent~$i$ itself: $N(i)=\left\{ j\in\agentS|\exists\ed\in\edS\;i,j\in\ed\;\wedge\;i\neq j\right\} $.

\begin{observation} The influence on ND-POMDP agent~$i$, $\ifpiAT i{t+1}(\polA{\excl i})$
can be defined as 
\begin{equation}
\ifpiAT i{t+1}(\sA{N(i)}^{t},\aA{N(i)}^{t}|\vec{s}_{0}^{~t})=\prod_{j\in N(i)}\ifpoAT{j}{\ts}(\sAT{j}\ts,\aAT j{\ts}|\vec{s}_{0}^{~t}),\label{eq:NDPOMDP_influence}
\end{equation}
with $\ifpoAT{j}{\ts}(\sAT{j}\ts,\aAT j{\ts}|\vec{s}_{0}^{~t})=\Pr(\sAT{j}\ts,\aAT j{\ts}|\vec{s}_{0}^{~t})$
the outgoing influence of agent~$j$.\end{observation}

Like the other mentioned sub-classes, the ND-POMDP affords a compact
influence encoding, suggesting that influence-based planning methods
could gain traction if applied here. Existing forms of influence search
exploit the fact that one can enumerate the \emph{joint influences,
}which describe how agents influence each other \citep{Witwicki10ICAPS,Witwicki12AAMAS}.
Due to the factorization of \eqref{eq:NDPOMDP_influence}, the joint
influence space here is a product space, which is easier to generate
and search through. Moreover, it would be possible to exploit the
graph structure of the ND-POMDP, similar to the approach by \citet[Section 6.6]{Witwicki11PhD}:
the space of joint influences can be decomposed as a factor graph
over which one can optimize more effectively.

\section{Related Work}

\label{sec:Related-Work}

Apart from the models and approaches reviewed in \sect\ref{sec:sub-classes-with-Compact-representations},
there are important connections to be drawn with a large body of other
work. Given the generality of the intuitive notion of `influence'
this should come as no surprise.  Here we describe those relations
to previous work and discuss further insights that our results provide.
We sub-divide these related works in:
\begin{itemize}
\item work on locality of interaction and value factorization in multiagent
systems,
\item other decomposition-like approaches in multiagent systems, and
\item more general forms of abstraction.
\end{itemize}

\subsection{Locality of Interaction and Value Factorization}

Past studies of factored Dec-POMDPs with factored value functions
\citep{Nair05AAAI,Varakantham07AAMAS,Kumar11IJCAI,Witwicki11AAMAS}
have shown that gains in computational efficiency are possible when
the value function can be expressed as the sum of a number of local
components, each of which is specified over subsets of agents and
state factors. In particular, the value in general Dec-POMDPs can
be expressed as a function $\V_{\jpol}(\sT{\ts},\oHistT\ts)$ \citep[e.g., see][chap. 3]{Oliehoek16Book}
of states and joint observation histories. Such a value function is
said to be a\emph{ factored value function }if there is a set of components
$\ed\in\edS$ such that
\begin{equation}
\V_{\jpol}(\sT{\ts},\oHistT\ts)=\sum_{\ed\in\edS}\V_{\jpolG{\ed}}(\sAT\ed\ts,\oHistGT\ed\ts),
\end{equation}
with $\jpolG{\ed}$ and $\oHistGT\ed\ts$ the policies respectively
observation histories of the agents that participate in component
$\ed$, and $\sAT\ed\ts$ the value of the state factors relevant
for $\ed$. For such problems, it is easy to show that they possess
\emph{locality of interaction} \citep{Nair05AAAI}: one can define
a local neighborhood for each agent such that its actions will not
impact the value beyond that neighborhood. This property allows one
to reduce the problem to a form of (distributed) constraint optimization
problem \citep[e.g., see][chap. 8]{Oliehoek16Book}.

However, for general factored Dec-POMDPs, the components $\ed$ involve
all agents and factors. I.e., they are \emph{not} local \citep{Oliehoek08AAMAS,Oliehoek10PhD}.
This paper shows that even in the most general case, \emph{it actually
is possible to find local (i.e., restricted scope) components}, although
this may be at the cost of introducing a dependence on the history
of a subset of the local state factors (the d-separating set $\dsetAT i{}$).
This means that it may be possible to extend the planning-as-inference
method of \citep{Kumar11IJCAI} to exploit structure in general fDec-POMDPs.\footnote{Note that, in general, IBA draws close connections to the paradigm
of planning as inference \citep{Toussaint09KI}; it performs inference
to compute a compact local model; subsequently, inference (among other
choices of solution methods) could be used to solve the IALM.} Researchers in the field of (deep) multiagent reinforcement learning,
have tried to exploit such factorized structure approximately \citep{Guestrin02ICML,Kok06JMLR,Kuyer08ECML,VanDerPol16LICMAS,Sunehag18AAMAS,Rashid18arxivQMIX,Castellini19AAMAS,Bohmer19arxiv,Son19ICML,Wang19arxiv_nearDecomposableVs},
and our work brings deeper understanding of those approaches.

For instance, \citet{Sunehag18AAMAS} proposed a form of factored
value functions~\citep{Guestrin01NIPS} making use of neural networks
that can be understood better using the theory developed in this paper.
Specifically they propose \emph{value-decomposition networks}, a variant
of \emph{deep Q-networks (DQN) }introduced by \citet{Mnih15Nature},
that uses a Q-function 
\begin{equation}
\tilde{Q}(\aoHist,\ja)=\sum_{\agentI i\in\agentS}\argsA{\tilde{Q}}{i}(\aoHistA i,\aA i),\label{eq:VDN-assumption}
\end{equation}
which is implemented in a single neural network with a linear layer
at the end that performs this summation. They state that
\begin{quote}
``the main assumption we make and exploit is that the joint action-value
function for the system can be additively decomposed into value functions
across agents''
\end{quote}
and this assumption has been pointed out as a limitation in subsequent
work \citep{Rashid18arxivQMIX,Bohmer19arxiv}. This paper, however,
demonstrates that \emph{there is a very large class of problems for
which this assumption (approximately) holds}. In particular, we show
that for any factored Dec-POMDP for which we can create a set of local-form
models (cf. \dfn\ref{dfn:lfm}), we have that:
\begin{equation}
V_{\jpol}(\aoHist)=\sum_{\agentI i\in\agentS}V_{i}(\lbA i)=\sum_{\agentI i\in\agentS}\max_{\aA i}Q_{i}(\lbA i,\aA i),\label{eq:IBA-V-decomposition}
\end{equation}
where $\lbA i$ is the local-form belief induced by $\aoHistA i$
and the policies of the other agents $\jpolG{\excl i}$. We also discuss
that, by introducing dummy variables as required (cf.\ the end of \sect\ref{sec:IBA:def-of-influence:links-sources-destinations}),
any factored Dec-POMDP can be re-coded as such set of local-from models.\footnote{Of course, depending on the problem, these components themselves might
be small (need to involve only few state variables) or large. We cannot
claim anything about the size of these components in general problems.
We merely reason that they can in principle be constructed, which
is sufficient to support our argument here.} As such, there is a very large class of problems for which ``the
system can be additively decomposed into value functions across agents''.
However, the devil is the details, we write ``(approximately)''
since the statement by \citet{Sunehag18AAMAS} is about $Q$ not $V$.
In particular, we have that 
\begin{equation}
Q_{\jpol}(\aoHist,\ja)\neq\sum_{\agentI i\in\agentS}Q_{i}(\lbA i,\aA i)
\end{equation}
since each term $Q_{i}(\lbA i,\aA i)$ assumes that the other agents
act according to $\jpolG{\excl i}$, not according to $\ja$. This
can explain the empirical improvements of methods that consider `higher
order approximations' with Q-components that involve subsets of agents~\citep{Oliehoek13AAMAS,Castellini19AAMAS,Bohmer19arxiv}.

We point out that this does not mean that an approximation $Q_{\jpol}(\aoHist,\ja)\approx\sum_{\agentI i\in\agentS}Q_{i}(\lbA i,\aA i)$
is senseless: in fact, we know that for the modified joint policy
$\jpol'$, which is like $\jpol$ but does $\ja$ instead of $\jpol(\aoHist)$,
the decomposition of $V_{\jpol'}$ according to \eqref{eq:IBA-V-decomposition}
also holds. As such, the question ``how good of an approximation
can we get with the individually factored Q-functions from \eqref{eq:VDN-assumption}?''
can be reinterpreted as a question of how the prediction of the components
$Q_{i}(\lbA i,\aA i)$ (which assume the others follow $\jpolG{\excl i}$)
generalize to ``first action modified policies'' $\jpol'$. In other
words, if for all such one-joint-action-modifications $\jpol'$ and
their induced local form beliefs $\argsAT{b}{i}{l\prime}$ we have
that $Q_{i}(\lbA i,\aA i)\approx Q_{i}(\argsAT{b}{i}{l\prime},\aA i)$,
then we expect this approximation to work well. Further formalizing
the impact of such first-action-modifications may be a promising direction
of research, and could lead to a novel notion of \emph{influence strength}
\citep{Allen09IAT,Oliehoek15arxiv_UBs} in multiagent domains.

We also remark that this analysis shows that, at least in cases that
do not exhibit strict locality of interaction, value factorization
inherently depends on the current policies of other agents, and hence
implies an `on-policy characteristic': when learning such factored
value components we can only learn about the $Q_{i}(\lbA i,\aA i)$
that are induced by those $\jpolG{\excl i}$ that are currently being
followed by the other agents. Of course, agent~$i$ itself can still
try to learn its approximation $\argsA{\tilde{Q}}{i}(\aoHistA i,\aA i)$
with off-policy methods, but sudden large changes to the own policy
may affect the ability of other agents to learn their local approximation.
We speculate that in more tightly coupled problems on-policy methods
with factorization may outperform off-policy ones.

\subsection{More General Forms of Decomposition in MASs}

In multiagent decision-making, there is a rich history of trying to
leverage structured interactions. For instance, our approach resembles
the distributed approximate planning method by \citet{Guestrin02UAI}
in that both methods decompose an agent\textquoteright s decision
model into internal and external parts. Our proposed abstraction,
in addition to being sufficient for \emph{optimal} decision-making,
is more general in that it can deal with partial observability.

\citet{Allen07AAAI,Allen09IAT} proposed a different formalization
of `influence' by building upon information-theoretical concepts
(mutual information between individual actions and joint states/observations/rewards).
They show how their notion of influence and influence gap (which measures
differences between the influencing power of agents) can predict the
difficulty of solving a problem. While conceptually closely related
to our work, their proposed notion of influence does not seem to support
doing abstraction in any non-trivial manner, and thus should be seen
as a very different type of object than our influence point.

The work by \citet{Chitnis20CORL} is close in spirit to IBA: they
propose to form a local abstraction of a factored MDP that approximates
the original model well. Their approach is to abstract away a subset
of \emph{exogenous variables} \citep{Boutilier99JAIR} and they propose
a method to select this subset. However, exogenous variables are defined
as variables that can influence our local model, but that \emph{cannot
be influenced by }the local model. This stands in stark contrast to
the non-modeled variables in IBA which can be affected by the local
model.

Another class of related work is that focusing on \emph{anonymous
interactions }such as mean field games, D-SPAIT, and Collective Dec-POMDPs
\citep{Jovanovic88anonymous,Kizilkale12TAC,Varakantham14AAAI,Robbel16AAAI,Nguyen17AAAI,Subramanian19AAMAS}.
These models assume that the interactions between a large set of agents
are governed by low dimensional statistics that capture how the rest
of the population influences each individual. For instance, in disease
propagation, only the \emph{number} (not the identity) of people that
are infected in one's neighborhood might matter~\citep{Robbel16AAAI}.
As argued in the beginning of \sect\ref{sec:IBA-with-IS-deps}, the
ability to include intra-stage connections into the IBA framework
can enable us to model most (if not all) such problems in the IBA
framework. So far, however, we have not yet identified how this can
lead to compact influence representations (as described in \sect\ref{sec:sub-classes-with-Compact-representations})
or more efficient approaches to solving these games.

\citet{Bazinin18GCAI} investigate exploiting a heuristic form of
influence in deterministic multiagent planning problems, formalized
in the qualitative Dec-POMDP~\citep{Brafman13AAAI} framework. Their
approach plans ``per agent'': first each agent computes a plan assuming
the other agents execute the actions that are most beneficial for
it. This creates constraints (influences) for the other agents. Then
the next agent gets to plan, subject to these constraints, and the
process iterates.

In this work, we show how we can define a local-form model based upon
a factored POSG model and a specification of a local state function
$\LSF$. We have not touched the question of how to define this local
state function. E.g., in a multi-robot cleaning task, each agent potentially
could clean every location, leading to local models as large as the
original problem. To counter this, one can apply \emph{organizations}\textbf{
}\citep{Carley99computationalOrg,Ferber04organizations,Vazquez05organizing}
which effectively constrain which agents can address what parts of
the problems. \citet{Sleight12AAMAS,Sleight15AAMAS} investigate such
organizations in a decision-theoretic context, also taking into account
simpler, approximate, forms of influence. Our definition of influence
is different as it critically depends on the d-separating set, which
is not considered by Sleight and Durfee. \citet{Claes17AAMAS} use
heuristics from multi-robot task allocation \citep{Gerkey03MRTA}.

Other models \citep{Spaan08AAMAS,Varakantham09ICAPS,Melo10AAMAS,Melo11AI}
have allowed for approximate decoupled local planning by leveraging
a form of context-specific independence, where agents only influence
each other in certain states. An important direction of research is
to also exploit this type of independence in LFMs. Similar ideas have
been considered in the multiagent RL setting too \citep{Melo09AAMAS,Hauwere10AAMAS}.
Structured interactions between agents are starting to be used for
examining concepts like understanding by agents \citep{Corona19NeurIPS}.

Approaches that take the perspective of a protagonist agent, like
the recursive modeling method \citep{Gmytrasiewicz95ICMAS,Gmytrasiewicz00JAAMAS},
I-POMDPs \citep{Gmytrasiewicz05JAIR}, and work on ad-hoc teams~\citep{Stone10AAAI,Albrecht13AAMAS}
 inherently provide a subjective perspective, which can include modeling
other agents recursively, that is conceptually related to our notion
of the local model. Although the formal definition of these models
is different from the fPOSG, our definition of influence (and thus
IBA) is readily applicable to factored-state versions of these models,
and therefore IBA can be extended to such models. Specifically, influence-based
abstraction is conceptually similar to existing approaches that exploit
\textit{behavioral equivalence} \citep{Pynadath07AAAI,Rathnasabapathy06AAMAS},
but these approaches abstract classes of behaviors down to policies,
whereas we abstract policies down to even more abstract influences.
These relations are illustrated in \fig\ref{fig:behavioral-equivalence}.

\begin{figure}
\begin{centering}
\includegraphics[width=0.8\columnwidth]{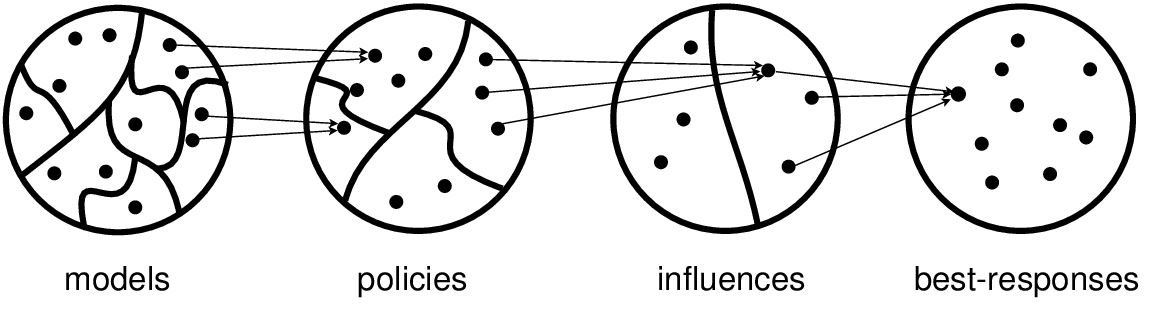}
\par\end{centering}
\caption{Many (e.g., I-POMDP) models for agent~$j$ may be \emph{behaviorally
equivalent, }i.e., map to the same policy $\pi_{j}$. In turn, many
policies $\polA j$ can lead to the same influence $\ifpiA i$ on
agent~$i$. Finally, many influences may map to the same best-response
$\polA i$. (Note that only a small part of the space of $\polA i$
may be a best-response to some influence/policy/model.)}

\label{fig:behavioral-equivalence}
\end{figure}

Also in multiagent learning, the idea that abstract representations
of influence exist and can help in learning are starting to be considered
\citep{HernandezLeal17arxiv}. For instance, \citet{Claes15AAMAS}
investigated an approximate form of influence of team mates in collaborative
spatial task allocation problems. \citet{Foerster17ICML} propose
`fingerprints' (episode indices) for when data was collected to
capture non-stationary due to the changing opponent strategy. \citet{Hong18AAMAS}
propose to augment DQN~\citep{Mnih15Nature} with a module to learn
`policy features' based on observations of the actions of other
agents. \citet{Jaques19ICML} propose to use a mutual information-based
version of influence (similar to \citealp{Allen07AAAI}, discussed
above) as an auxiliary reward, and \citet{Wang20ICLR} extended this
to direct exploration in multiagent reinforcement learning. To some
extent, all forms of agent modeling \citep[e.g.][]{HernandezLeal17arxiv,Hernandez19AIIDE,Tacchetti19ICLR}
or tracking \citep{Sunberg17ACC} can be seen as a some form of influence
prediction, since one can think of the action of another agent as
an influence source. However, few of these approaches further formalize
the structure of this interaction, which means that they have not
exploited the insight that one only may need to remember a subset
of variables, even though this can lead to significant improvements
\citep{Suau19arxiv}.

\subsection{Other Forms of Abstraction}

\label{sec:Other-Forms-of-abstraction}

Influence-based abstraction is a form of state abstraction, which
has a long tradition in AI planning and learning \citep[e.g.,][]{Sacerdoti74AIJ,Knoblock93Book,McCallum93ML,Dearden97AIJ,Dean97AAAI,Hoey99UAI,Givan00AIJ,Boutilier00AI,Ravindran03IJCAI,Jong05IJCAI,Konidaris09IJCAI,Kaelbling12TRHPN,Hostetler14AAAI,Anand16ICAPS,Bai16IJCAI,Abel19AAAI}.
Other types of abstraction \citep{Mahadevan10AAAI} are temporal abstractions,
such as options and macro-actions \citep{Sutton99AIJ,Theocharous03NIPS16,Amato19JAIR,Machado17ICML},
and functional abstraction, which tries to identify appropriate basis
functions \citep{Keller06ICML,Parr07ICML,Mahadevan07JMLR,Petrik07IJCAI_laplacian},
including the huge body of recent work on deep RL \phantom{\c{c}}\citep{Schmidhuber91NIPS,Mnih15Nature,FrancoisLavet18FaT}.
We will focus on related work on state abstraction.

Different manners of performing state abstraction in MDPs exist, such
as state aggregation methods which cluster similar states together,
or starting with one abstract state and subsequently splitting \citep{Givan03AIJ},
or removing state factors with no impact on the policy or rewards~\citep{Jong05IJCAI,Dearden97AIJ}.
At a technical level these approaches are based on the idea that the
original MDP and the abstraction are bisimilar. MDP homomorphisms
\citep{Ravindran02SARA,Ravindran03IJCAI} generalize the idea of bisimilarity
to also consider similarity of different actions. These ideas can
also be used as metrics~\citep{Ferns04UAI,Ferns14UAI}, and many
of these ideas lie at the core of recent model-based (deep) RL approaches
\citep{Corneil18ICML,Gelada19ICML,Biza19AAMAS,VanDerPol20AAMAS}.

Other methods implement abstraction as part of the solution method~\citep{Hoey99UAI,Boutilier00AI,St-Aubin01NIPS13}.
Different notions of which states to group together exist. \citet{Li06ISAIM}
present a unifying framework that discriminates a number of types
of exact state abstraction, and some of these were recently extended
to approximate state abstractions~\citep{Abel16ICML}. The introduced
notions of model irrelevance/similarity are particularly relevant:
they group together states that behave (approximately) identical in
terms of rewards and transitions, which is also what IBA achieves
in its influence-augmented local model.

However, there is one big difference between all these methods and
the influence-based abstraction: in order to achieve a good approximation,
all the previous notions can only group states together that have
very similar (usually measured in L1 norm of) transition probabilities,
which severely limits their applicability. Existing methods can generally
not abstract away an entire state variable that is an influence source
and still provide guarantees of near optimality. In contrast, IBA
does enable abstracting away such influence sources, and thus groups
together states that can have very different transition probabilities.
IBA corrects for this by incorporating the influence in the IALM,
by means of the dependence on the d-separating set $\dsetAT i{}$.

Another body of work casts abstracted, non-Markovian, models as models
with imprecise probabilities~\citep{Givan00AIJ,Iyengar05MOR,Sanner10AAMAS,Delgado11AIJ,Delgado11IJAR,Petrik14NIPS,Delgado16AIJ}.
These typically place intervals on the transition probabilities and
compute `robust' policies that give the optimal worst case (with
respect to the realized transition probabilities) payoff. Essentially
these models are equivalent to a two-player zero-sum game where the
agent faces an adversarial environment that chooses the transition
probabilities to sabotage the agent~\citep{Iyengar05MOR}. A disadvantage
of such approaches is that they are only useful if the uncertainty
intervals are sufficiently small and, as above, this is very hard
to guarantee when abstracting away entire state variables. As such,
the contribution of IBA is complimentary: it shows that it is possible
to create abstract models which have no uncertainty interval at all.

IBA also bears some similarity to the framework of mixed-observability
MDPs~\citep{Ong09RSS,Ong10IJRR}, which splits the state $s=\left\langle o,l\right\rangle $
into observable state factors $o$ and hidden ones $l$. IBA, however,
splits $s=\left\langle \mfA i,\nmfA i\right\rangle $ into modeled
$\mfA i$ and non-modeled factors $\nmfA i$. As such, the frameworks
are complimentary: the local state space of an agent after performing
IBA can have mixed observability\footnote{While the non-modeled factors $\nmfA i$ are hidden, (some of) the
modeled state factors $\mfA i$ can be fully observed: in our formalism
such observability of a factor $\mfI k$ would be modeled by introducing
an observation factor that has $\mfI k$ as its only parent and has
the identity function as its conditional probability table.} and the hidden part $l$ of a mixed-observability MDPs can be abstracted
by using IBA.

Finally, abstractions have also been investigated as the basis for
robotic decision making \citep{Konidaris18JAIR} and multi-robot decision
making \citep{Le18JAIR,Amato19JAIR}. These methods typically combine
temporal and state abstraction. Specifically, \citet{Konidaris18JAIR}
focus on learning abstract state representation that support open
loop planning using a given set of `skills' (also called `options',
\citealt{Sutton99AIJ}, or `macro actions', \citealt{Amato19JAIR})
and demonstrate this on a robot. While the formalization allows for
probabilistic effects (they can reason about probability that the
plan is executable), they assume that the skills are such that the
effect of a skill $\sigma$ does not depend on the previous state,
such that $\Pr(s'|\sigma)$ is well defined. In practice, the approach
typically requires small sets of states $s'$ with positive support,
or the probability of executability drops. As such, the framework
is less suited for highly stochastic environments, such as those affected
by other agents or other type of exogenous events~\citep{Boutilier99JAIR}.
Thus, again, our work here is complementary, since it shows what parts
of history may need to be retained to decrease this stochasticity.
\citet{Le18JAIR} focus on multi-robot motion planning. The difficulty
here is to reason both about detailed motions, as well as the presence
of multiple robots. To deal with this they propose to reason about
the interaction (making use of multiagent path planning) in an abstract
representation, this high-level plan is then used as a heuristic for
the low-level motion planning. \citet{Amato19JAIR} formalize hierarchical
Dec-POMDPs, called Mac-Dec-POMDP (for `macro-action') where multiple
agents act using options. The focus of this work lies on how to plan
with options in the Dec-POMDP setting, but the abstractions at higher
levels are assumed to be given.

\section{Conclusion, Discussion and Future Work}

\label{sec:conclusions}

This paper makes a theoretical contribution to the field of decision
making in factored multiagent settings by giving a rigorous definition
of \emph{influence-based abstraction (IBA) }in such settings. It defines
a notion of `influence' that enables an agent in a POMDP to perform
a lossless abstraction of the decision making problem it faces. That
is, we prove that, for a given abstraction in terms of a \emph{local-form
model}, an \emph{influence point} is a sufficient statistic for the
part of the problem that is abstracted away. The local-form model
and influence point together induce what we call an \emph{influence-augmented
local model~(IALM)}:\emph{ }a local model that is sufficient to compute
an exact best response.

The proof of sufficiency also serves a practical purpose: it isolates
the core technical property (in \sect\ref{sec:suff:local-transitions})
that needs to hold for sufficiency. In this way it conveys insight
into the nature of \emph{how} abstraction of latent state factors
affects value, provides a derivation that can be used to obtain simplifications
of the definition of influence in simpler cases, and provides a recipe
of how to prove similar results in more general cases.

At a higher level, IBA is important for the following reasons:
\begin{enumerate}
\item The theory presented in this paper presents a new perspective on abstraction
in structured settings: it shows that such abstractions can be seen
as special cases of POMDPs, where one only needs to remember about
a subset of variables. Effectively, this can create a problem class
in between MDPs and POMDPs: In this class, in order to predict the
local dynamics, we will need to use memory, but this memory only needs
to store information about the history of a subset of state variables.
\item It can enable more efficient best-response computation in fPOSGs,
as well as providing a very natural form of approximation via approximate
inference.
\item It provides a better understanding of previously identified sub-classes
of fPOSGs \citep{Becker03AAMAS,Becker04AAMAS,Nair05AAAI,Petrik09JAIR,Oliehoek10PhD,Kumar11IJCAI}
and how they relate to each other. The insightful connections that
we have drawn promote extensions of specialized methods beyond their
respective sub-classes as well as comparisons with one another in
more general contexts. For instance, our work has identified a compact
representation of influences in ND-POMDPs (where none was known) and
identified a more compact representation for EDI-Dec-MDPs. 
\item It demonstrates how the value function for essentially \emph{any}
factored Dec-POMDP can be decomposed into the sum of a number of \emph{local}
value functions. As such, IBA demonstrates that all such problems
satisfy a weak form of locality of interaction (also `value factorization')\textemdash a
property that is exploited in several Dec-POMDP solution methods \citep{Nair05AAAI,Oliehoek10PhD,Kumar11IJCAI}
and multiagent RL papers \citep{Guestrin02ICML,Kok06JMLR,Kuyer08ECML,VanDerPol16LICMAS,Sunehag18AAMAS,Rashid18arxivQMIX,Castellini19AAMAS,Bohmer19arxiv,Son19ICML,Wang19arxiv_nearDecomposableVs}.
\item Influences can provide a more compact, yet sufficient statistic for
the behavior of other agents in a MAS. We expect this to be important
in multiagent reinforcement learning, since it is often easier to
learn a compact statistic from the same amount of data.
\end{enumerate}
We emphasize that this definition of influence is not a magic bullet:
while the influence-augmented local model is sufficient to compute
a best-response locally, the computation of the required influence
point itself is an intractable inference problem in general. However,
in certain cases where this problem \emph{is} feasible it can enable
faster best-response computations and search for multiagent plans
via \emph{influence search }\citep{Witwicki10ICAPS,Witwicki12AAMAS}.
As such, an important direction of future work would investigate how
the definition of influence presented in this paper can support influence
search in more general settings.

Moreover, even in cases where influences are intractable to compute,
the concept forms the basis for principled approximations. For instance,
by being \emph{optimistic} with respect to the influence sources,
one is able to compute upper bounds on the optimal value of Dec-POMDPs
with hundreds of agents, thus leading to firm guarantees on the quality
of heuristic solutions \citep{Oliehoek15IJCAI}. Furthermore, there
is evidence, in the context of deep reinforcement learning, that such
approximate versions of influence may in some problems improve learning,
both in terms of speed as well as performance \citep{Suau19ALA}.
As such, a fruitful direction of research is to better understand
such approximate characterization of influence~\citep{Congeduti20arxiv}.
This article has provided the foundations for such an exploration.

An important direction of of future research would explore the applications
of (approximate) forms of influence. For instance, it is possible
that these can make a huge impact on human-robot interactions \citep{Shah11HRI,Nikolaidis15HRI}.
We note that even though the discussion in this paper was based on
the more general case of multiagent systems, there is nothing that
stops us from applying IBA in complex systems with just a single agent.
As a case in point, \citet{Suau19ALA} show improvements of learning
on a single traffic intersection and on Atari games.

Finally, in this paper the discussion is limited to settings where
structure is known. \emph{If} we have structure, this can be exploited
to define influence, possibly leading to more efficient best responses,
or other benefits. Certainly, in many cases knowledge of the structure
is not available. Future work could try to build off the advances
in structure learning algorithms~\citep{Murphy02PHD,KollerFriedman09,Doshi11ICML,Murphy12Book}
and their integration with decision making problems \citep{Degris06ICML,Strehl07AAAI,Walsh10AAAI,Doshi09NIPS,Littman12ICGI,Katt19AAMAS};
as long as it is possible to learn a model our methods would apply.
In particular, even though such an estimated model might be inaccurate,
its reduction to an IALM would add no further estimation error. This
observation may open up a new research direction in sequential decision
making that forsakes approximate solution methods (e.g., Monte Carlo
tree search,~\citealt{Browne12CAAIG}, RL techniques like DQN,~\citealt{Mnih15Nature},
or other forms of approximate dynamic programming, \citealt{Bertsekas05DPBook_vol1,Bertsekas07DPBook_vol2,Powell12AnnalsOR})
in favor of learning useful approximate models that give structured
representations of interactions. Arguably, such approaches that make
exact use of assumed (but approximate) models lie at the basis of
many, if not most, engineering disciplines and thus served human intelligence
well in the past.

\section*{Acknowledgments}

We would like to thank Elena Congeduti, Rolf Starre and Miguel Suau,
and Mikko Lauri for their helpful comments. Major parts of this work
where performed while F.~A.~Oliehoek was affiliated with MIT, University
of Amsterdam, and University of Liverpool, and while S. Witwicki was
affiliated with EPFL. This paper is the result of research that received
funding from various funding agencies, including AFOSR (MURI), NWO
(VENI), and the European Research Council (ERC) under the European
Union\textquoteright s Horizon 2020 research and innovation programme
(grant agreement No.~758824\textemdash INFLUENCE).
\begin{center}
\includegraphics[width=4cm]{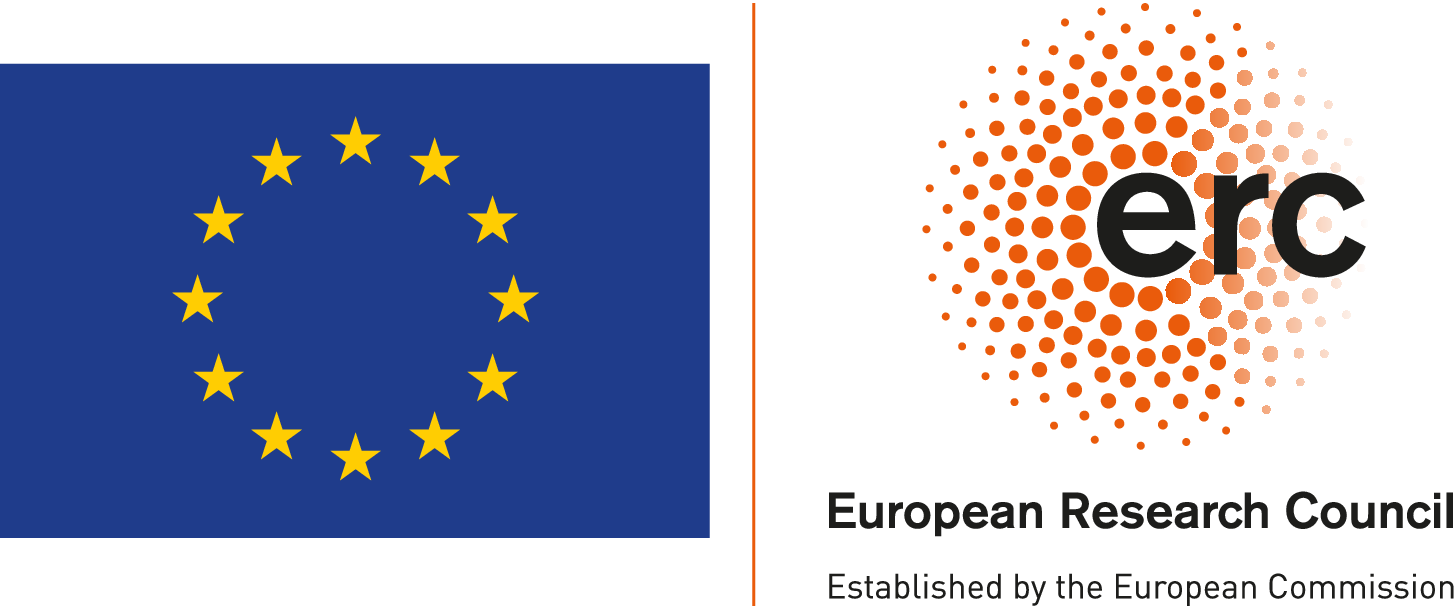}
\par\end{center}

\newpage{}

\appendix

\section{Proofs and Derivations}

Here we give proofs and derivations of a number of results. These
are referred from the main text, and will be stated here without further
explanation.

\subsection{GFBRMs}

\subsubsection{Expected Reward }

\label{app:GFBRM_reward}

\begin{eqnarray*}
R_{i}(\gbA i,\aAT it) & = & \E_{\sAugAT it\sim\gbA i,\sAugAT i{t+1}\sim\Aug T(\sAugAT it,\aAT it,\cdot)}\left[\Aug{R_{i}}(\sAugAT it,\aAT it,\sAugAT i{t+1})\right]\\
 & = & \sum_{\sAugAT it}\gbA i(\sAugAT it)\sum_{\sAugAT i{t+1}}\Aug T(\sAugAT i{t+1}|\sAugAT it,\aAT it)\Aug{R_{i}}(\sAugAT it,\aAT it,\sAugAT i{t+1})\\
 & = & \sum_{\langle\sT{\ts},\aoHistAT{\excl i}{\ts}\rangle}\gbA i(\langle\sT{\ts},\aoHistAT{\excl i}{\ts}\rangle)\sum_{\langle\sT{\ts+1},\aoHistAT{\excl i}{\ts+1}\rangle}\Pr(\langle\sT{\ts+1},\aoHistAT{\excl i}{\ts+1}\rangle|\langle\sT{\ts},\aoHistAT{\excl i}{\ts}\rangle,\aA i)\Aug{R_{i}}(\langle\sT{\ts},\aoHistAT{\excl i}{\ts}\rangle,\aA i,\langle\sT{\ts+1},\aoHistAT{\excl i}{\ts+1}\rangle)\\
 & = & \sum_{\langle\sT{\ts},\aoHistAT{\excl i}{\ts}\rangle}\gbA i(\langle\sT{\ts},\aoHistAT{\excl i}{\ts}\rangle)\sum_{\sT{\ts+1}}\sum_{\jaG{\excl i}}\sum_{\joGT{\excl i}{\ts+1}}\Pr(\sT{\ts+1},\jaG{\excl i},\joGT{\excl i}{\ts+1}|\langle\sT{\ts},\aoHistAT{\excl i}{\ts}\rangle,\aA i)R_{i}(\sT{\ts},\aA i,\jaG{\excl i},\sT{\ts+1})\\
 & = & \sum_{\langle\sT{\ts},\aoHistAT{\excl i}{\ts}\rangle}\gbA i(\langle\sT{\ts},\aoHistAT{\excl i}{\ts}\rangle)\sum_{\sT{\ts+1}}\sum_{\jaG{\excl i}}\Pr(\sT{\ts+1},\jaG{\excl i}|\langle\sT{\ts},\aoHistAT{\excl i}{\ts}\rangle,\aA i)R_{i}(\sT{\ts},\aA i,\jaG{\excl i},\sT{\ts+1})\\
 & = & \sum_{\sT{\ts}}\sum_{\sT{\ts+1}}\sum_{\jaG{\excl i}}\Pr(\sT{\ts+1}|\sT{\ts},\ja)R_{i}(\sT{\ts},\ja,\sT{\ts+1})\sum_{\aoHistAT{\excl i}{\ts}}\Pr(\jaG{\excl i}|\aoHistAT{\excl i}{\ts})\gbA i(\sT{\ts},\aoHistAT{\excl i}{\ts})
\end{eqnarray*}

\subsubsection{Expected Observation Probability}

\label{app:GFBRM_observ}

$\Pr(\oAT i{\ts+1}|\gbA i,\aAT it)$
\begin{eqnarray}
 & = & \E_{\sAugAT it\sim\gbA i,\sAugAT i{t+1}\sim\Aug T(\sAugAT it,\aAT it,\cdot)}\left[\Aug O(\oAT i{\ts+1}|\aAT it,\sAugAT i{t+1})\right]\nonumber \\
 & = & \sum_{\sAugAT it}\gbA i(\sAugAT it)\sum_{\sAugAT i{t+1}}\Aug T(\sAugAT i{t+1}|\sAugAT it,\aAT it)\Aug O(\oAT i{\ts+1}|\aAT it,\sAugAT i{t+1})\nonumber \\
 & = & \sum_{\langle\sT{\ts},\aoHistAT{\excl i}{\ts}\rangle}\gbA i(\langle\sT{\ts},\aoHistAT{\excl i}{\ts}\rangle)\sum_{\langle\sT{\ts+1},\aoHistAT{\excl i}{\ts+1}\rangle}\Pr(\langle\sT{\ts+1},\aoHistAT{\excl i}{\ts+1}\rangle|\langle\sT{\ts},\aoHistAT{\excl i}{\ts}\rangle,\aA i)\Pr(\oA i|\aA i,\langle\sT{\ts+1},\aoHistAT{\excl i}{\ts+1}\rangle)\nonumber \\
 & = & \sum_{\langle\sT{\ts},\aoHistAT{\excl i}{\ts}\rangle}\gbA i(\langle\sT{\ts},\aoHistAT{\excl i}{\ts}\rangle)\sum_{\sT{\ts+1}}\sum_{\jaG{\excl i}}\sum_{\joGT{\excl i}{\ts+1}}\Pr(\sT{\ts+1},\jaG{\excl i},\joGT{\excl i}{\ts+1}|\langle\sT{\ts},\aoHistAT{\excl i}{\ts}\rangle,\aA i)\Pr(\oA i|\aA i,\jaG{\excl i},\sT{\ts+1},\joGT{\excl i}{\ts+1})\nonumber \\
 & = & \sum_{\langle\sT{\ts},\aoHistAT{\excl i}{\ts}\rangle}\gbA i(\langle\sT{\ts},\aoHistAT{\excl i}{\ts}\rangle)\sum_{\sT{\ts+1}}\sum_{\jaG{\excl i}}\sum_{\joGT{\excl i}{\ts+1}}\Pr(\sT{\ts+1}|\sT{\ts},\jaG{\excl i},\aA i)\Pr(\jaG{\excl i}|\aoHistAT{\excl i}{\ts},\jpolG{\excl i})\Pr(\joGT{\excl i}{\ts+1}|\aA i,\jaG{\excl i},\sT{\ts+1})\nonumber \\
 &  & \quad\Pr(\oA i|\aA i,\jaG{\excl i},\sT{\ts+1},\joGT{\excl i}{\ts+1})
\end{eqnarray}
\begin{eqnarray}
 & = & \sum_{\langle\sT{\ts},\aoHistAT{\excl i}{\ts}\rangle}\gbA i(\langle\sT{\ts},\aoHistAT{\excl i}{\ts}\rangle)\sum_{\sT{\ts+1}}\sum_{\jaG{\excl i}}\sum_{\joGT{\excl i}{\ts+1}}\Pr(\sT{\ts+1}|\sT{\ts},\jaG{\excl i},\aA i)\Pr(\jaG{\excl i}|\aoHistAT{\excl i}{\ts},\jpolG{\excl i})\Pr(\joGT{\excl i}{\ts+1}|\aA i,\jaG{\excl i},\sT{\ts+1})\nonumber \\
 &  & \quad\frac{\Pr(\oA i,\joGT{\excl i}{\ts+1}|\aA i,\jaG{\excl i},\sT{\ts+1})}{\Pr(\joGT{\excl i}{\ts+1}|\aA i,\jaG{\excl i},\sT{\ts+1})}\nonumber \\
 & = & \sum_{\langle\sT{\ts},\aoHistAT{\excl i}{\ts}\rangle}\gbA i(\langle\sT{\ts},\aoHistAT{\excl i}{\ts}\rangle)\sum_{\sT{\ts+1}}\sum_{\jaG{\excl i}}\sum_{\joGT{\excl i}{\ts+1}}\Pr(\sT{\ts+1}|\sT{\ts},\ja)\Pr(\jaG{\excl i}|\aoHistAT{\excl i}{\ts},\jpolG{\excl i})\Pr(\oA i,\joGT{\excl i}{\ts+1}|\aA i,\jaG{\excl i},\sT{\ts+1})\nonumber \\
 & = & \sum_{\sT{\ts}}\sum_{\sT{\ts+1}}\sum_{\jaG{\excl i}}\sum_{\joGT{\excl i}{\ts+1}}\Pr(\sT{\ts+1}|\sT{\ts},\ja)\Pr(\joT{\ts+1}|\ja,\sT{\ts+1})\sum_{\aoHistAT{\excl i}{\ts}}\Pr(\jaG{\excl i}|\aoHistAT{\excl i}{\ts},\jpolG{\excl i})\gbA i(\sT{\ts},\aoHistAT{\excl i}{\ts})\label{eq:P(o|gfb,a)__GFM-1}
\end{eqnarray}

\subsection{LFMs}

\subsubsection{Expected Reward}

\label{app:LFM:exp_reward}

Starting with \eqref{eq:R(gfb,a)__GFM}, we have that $R_{i}(\gbA i,\aA i)$
\begin{eqnarray*}
 & = & \sum_{\sT{\ts}}\sum_{\sT{\ts+1}}\sum_{\jaG{\excl i}}\Pr(\sT{\ts+1}|\sT{\ts},\ja)R_{i}(\sT{\ts},\ja,\sT{\ts+1})\sum_{\aoHistAT{\excl i}{\ts}}\Pr(\jaG{\excl i}|\aoHistAT{\excl i}{\ts},\jpolG{\excl i})\gbA i(\sT{\ts},\aoHistAT{\excl i}{\ts})\\
 & = & \sum_{\sT{\ts}}\sum_{\sT{\ts+1}}\sum_{\jaG{\excl i}}\Pr(\sT{\ts+1}|\sT{\ts},\aA i,\jaG{\excl i})R_{i}(\sT{\ts},\ja,\sT{\ts+1})\sum_{\aoHistAT{\excl i}{\ts}}\Pr(\jaG{\excl i}|\aoHistAT{\excl i}{\ts},\jpolG{\excl i})\gbA i(\sT{\ts},\aoHistAT{\excl i}{\ts})\\
 & = & \text{\{restrict to actual dependencies of \ensuremath{R_{i}}\}}\\
 &  & \sum_{\sT{\ts}}\sum_{\sT{\ts+1}}\sum_{\jaG{\excl i}}\Pr(\sT{\ts+1}|\sT{\ts},\aA i,\jaG{\excl i})R_{i}(\mfAT i{\ts},\aA i,\mfAT i{\ts+1})\sum_{\aoHistAT{\excl i}{\ts}}\Pr(\jaG{\excl i}|\aoHistAT{\excl i}{\ts},\jpolG{\excl i})\gbA i(\sT{\ts},\aoHistAT{\excl i}{\ts})\\
 & = & \sum_{\sT{\ts}}\sum_{\jaG{\excl i}}\sum_{\mfAT i{\ts+1},\nmfAT i{\ts+1}}\Pr(\mfAT i{\ts+1},\nmfAT i{\ts+1}|\sT{\ts},\aA i,\jaG{\excl i})R_{i}(\mfAT i{\ts},\aA i,\mfAT i{\ts+1})\sum_{\aoHistAT{\excl i}{\ts}}\Pr(\jaG{\excl i}|\aoHistAT{\excl i}{\ts},\jpolG{\excl i})\gbA i(\sT{\ts},\aoHistAT{\excl i}{\ts})
\end{eqnarray*}

\begin{eqnarray}
 & = & \text{\{via \eqref{eq:P_xm_sa}\}}\nonumber \\
 &  & \sum_{\mfAT i{\ts},\nmfAT i{\ts}}\sum_{\jaG{\excl i}}\sum_{\mfAT i{\ts+1}}\Pr(\mfAT i{\ts+1}|\sT{\ts},\aA i,\jaG{\excl i})R_{i}(\mfAT i{\ts},\aA i,\mfAT i{\ts+1})\sum_{\aoHistAT{\excl i}{\ts}}\Pr(\jaG{\excl i}|\aoHistAT{\excl i}{\ts},\jpolG{\excl i})\gbA i(\sT{\ts},\aoHistAT{\excl i}{\ts})\nonumber \\
 & = & \sum_{\mfAT i{\ts}}\sum_{\mfAT i{\ts+1}}R_{i}(\mfAT i{\ts},\aA i,\mfAT i{\ts+1})\sum_{\nmfAT i{\ts}}\sum_{\jaG{\excl i}}\Pr(\mfAT i{\ts+1}|\sT{\ts},\aA i,\jaG{\excl i})\sum_{\aoHistAT{\excl i}{\ts}}\Pr(\jaG{\excl i}|\aoHistAT{\excl i}{\ts},\jpolG{\excl i})\gbA i(\sT{\ts},\aoHistAT{\excl i}{\ts})\nonumber \\
 & = & \sum_{\mfAT i{\ts}}\sum_{\mfAT i{\ts+1}}R_{i}(\mfAT i{\ts},\aA i,\mfAT i{\ts+1})\Pr(\mfAT i{\ts},\mfAT i{\ts+1}|\gbA i,\aAT it,\jpolG{\excl i}),\label{eq:R(gfb,a)__LFM-1}
\end{eqnarray}
where we implicitly defined (remember $\sT{\ts}=\langle\mfAT i{\ts},\nmfAT i{\ts}\rangle$)

\begin{equation}
\Pr(\mfAT i{\ts},\mfAT i{\ts+1}|\gbA i,\aAT it,\jpolG{\excl i})\defas\sum_{\nmfAT i{\ts}}\sum_{\jaG{\excl i}}\Pr(\mfAT i{\ts+1}|\sT{\ts},\aA i,\jaG{\excl i})\sum_{\aoHistAT{\excl i}{\ts}}\Pr(\jaG{\excl i}|\aoHistAT{\excl i}{\ts},\jpolG{\excl i})\gbA i(\sT{\ts},\aoHistAT{\excl i}{\ts})\label{eq:P_xmxm_gfb-1}
\end{equation}

\subsubsection{Expected Observation Probability}

\label{app:LFM:exp_obs}In this case, the expected observation probability
$\Pr(\oAT i{\ts+1}|\gbA i,\aA i)$ equals

\begin{eqnarray}
 & = & \sum_{\sT{\ts}}\sum_{\sT{\ts+1}}\sum_{\jaG{\excl i}}\sum_{\joGT{\excl i}{\ts+1}}\Pr(\sT{\ts+1}|\sT{\ts},\ja)\Pr(\joT\ts|\ja,\sT{\ts+1})\sum_{\aoHistAT{\excl i}{\ts}}\Pr(\jaG{\excl i}|\aoHistAT{\excl i}{\ts},\jpolG{\excl i})\gbA i(\sT{\ts},\aoHistAT{\excl i}{\ts})\nonumber \\
 & = & \sum_{\sT{\ts+1}}\sum_{\sT{\ts}}\sum_{\jaG{\excl i}}\sum_{\joGT{\excl i}{\ts+1}}\Pr(\oAT i{\ts+1},\joGT{\excl i}{\ts+1}|\aA i,\jaG{\excl i},\sT{\ts+1})\Pr(\sT{\ts+1}|\sT{\ts},\aA i,\jaG{\excl i})\nonumber \\
 &  & \sum_{\aoHistAT{\excl i}{\ts}}\Pr(\jaG{\excl i}|\aoHistAT{\excl i}{\ts},\jpolG{\excl i})\gbA i(\sT{\ts},\aoHistAT{\excl i}{\ts})\nonumber \\
 & = & \text{\{marginalize\}}\nonumber \\
 &  & \sum_{\sT{\ts+1}}\sum_{\sT{\ts}}\sum_{\jaG{\excl i}}\Pr(\oAT i{\ts+1}|\aA i,\jaG{\excl i},\sT{\ts+1})\Pr(\sT{\ts+1}|\sT{\ts},\aA i,\jaG{\excl i})\sum_{\aoHistAT{\excl i}{\ts}}\Pr(\jaG{\excl i}|\aoHistAT{\excl i}{\ts},\jpolG{\excl i})\gbA i(\sT{\ts},\aoHistAT{\excl i}{\ts})\nonumber \\
 & = & \text{\{restrict to actual dependencies\}}\nonumber \\
 &  & \sum_{\sT{\ts+1}}\sum_{\sT{\ts}}\sum_{\jaG{\excl i}}\Pr(\oAT i{\ts+1}|\aA i,\mfAT i{\ts+1})\Pr(\sT{\ts+1}|\sT{\ts},\aA i,\jaG{\excl i})\sum_{\aoHistAT{\excl i}{\ts}}\Pr(\jaG{\excl i}|\aoHistAT{\excl i}{\ts},\jpolG{\excl i})\gbA i(\sT{\ts},\aoHistAT{\excl i}{\ts})\nonumber \\
 & = & \sum_{\mfAT i{\ts+1},\nmfAT i{\ts+1}}\sum_{\sT{\ts}}\sum_{\jaG{\excl i}}\Pr(\oAT i{\ts+1}|\aA i,\mfAT i{\ts+1})\Pr(\mfAT i{\ts+1},\nmfAT i{\ts+1}|\sT{\ts},\aA i,\jaG{\excl i})\sum_{\aoHistAT{\excl i}{\ts}}\Pr(\jaG{\excl i}|\aoHistAT{\excl i}{\ts},\jpolG{\excl i})\gbA i(\sT{\ts},\aoHistAT{\excl i}{\ts})\nonumber \\
 & = & \sum_{\mfAT i{\ts+1}}\Pr(\oAT i{\ts+1}|\aA i,\mfAT i{\ts+1})\sum_{\sT{\ts}}\sum_{\nmfAT i{\ts+1}}\sum_{\jaG{\excl i}}\Pr(\mfAT i{\ts+1},\nmfAT i{\ts+1}|\sT{\ts},\aA i,\jaG{\excl i})\sum_{\aoHistAT{\excl i}{\ts}}\Pr(\jaG{\excl i}|\aoHistAT{\excl i}{\ts},\jpolG{\excl i})\gbA i(\sT{\ts},\aoHistAT{\excl i}{\ts})\nonumber \\
 & = & \sum_{\mfAT i{\ts+1}}\Pr(\oAT i{\ts+1}|\aA i,\mfAT i{\ts+1})\Pr(\mfAT i{\ts+1}|\gbA i,\aA i,\jpolG{\excl i})\label{eq:P(o|gfb,a)__LFM-1}
\end{eqnarray}
where we implicitly defined
\begin{equation}
\Pr(\mfAT i{\ts+1}|\gbA i,\aA i)\defas\sum_{\sT{\ts}}\sum_{\jaG{\excl i}}\Pr(\mfAT i{\ts+1}|\sT{\ts},\aA i,\jaG{\excl i})\sum_{\aoHistAT{\excl i}{\ts}}\Pr(\jaG{\excl i}|\aoHistAT{\excl i}{\ts}\jpolG{\excl i})\gbA i(\sT{\ts},\aoHistAT{\excl i}{\ts}).\label{eq:P_fm__gb-1}
\end{equation}

\subsection{IALMs}

\subsubsection{Expected Observation Probability}

\label{app:IALM:exp_obs}

\begin{align*}
\Pr(\oAT i{\ts+1}|\lbA i,\aAT i\ts) & =\E_{\sAugAT i\ts\sim\lbA i,\sAugAT i{t+1}\sim\Aug T(\sAugAT i\ts,\aAT i\ts,\cdot)}\left[\Aug O(\oAT i{\ts+1}|\aAT it,\sAugAT i{t+1})\right]\\
 & =\sum_{\sAugAT i\ts}\lbA i(\sAugAT it)\sum_{\sAugAT i{t+1}}\Aug T(\sAugAT i{t+1}|\sAugAT i\ts,\aAT i\ts)\Aug O(\oAT i{\ts+1}|\aAT it,\sAugAT i{t+1})
\end{align*}

\begin{eqnarray}
 & = & \sum_{\mfAT i{\ts},\dsetAT i{\ts+1}}\lbA i(\mfAT i{\ts},\dsetAT i{\ts+1})\sum_{\mfAT i{\ts+1},\dsetAT i{\ts+2}}\Pr(\mfAT i{\ts+1},\dsetAT i{\ts+2}|\mfAT i{\ts},\dsetAT i{\ts+1},\aAT it,\ifpiAT i{\ts+1})\Pr(\oAT i{\ts+1}|\aAT it,\mfAT i{\ts+1})\nonumber \\
 & = & \sum_{\mfAT i{\ts},\dsetAT i{\ts+1}}\lbA i(\mfAT i{\ts},\dsetAT i{\ts+1})\sum_{\mfAT i{\ts+1}}\Pr(\mfAT i{\ts+1}|\mfAT i{\ts},\dsetAT i{\ts+1},\aAT it,\ifpiAT i{\ts+1})\Pr(\oAT i{\ts+1}|\aAT it,\mfAT i{\ts+1})\nonumber \\
 & = & \sum_{\mfAT i{\ts+1}}\Pr(\oAT i{\ts+1}|\aAT it,\mfAT i{\ts+1})\left[\sum_{\mfAT i{\ts},\dsetAT i{\ts+1}}\Pr(\mfAT i{\ts+1}|\mfAT i{\ts},\dsetAT i{\ts+1},\aAT it,\ifpiAT i{\ts+1})\lbA i(\mfAT i{\ts},\dsetAT i{\ts+1})\right]\nonumber \\
 & = & \sum_{\mfAT i{\ts+1}}\Pr(\oAT i{\ts+1}|\aAT it,\mfAT i{\ts+1})\Pr(\mfAT i{\ts+1}|\lbA i,\aAT it,\ifpiAT i{\ts+1}),\label{eq:P(o|lfb,a)-1}
\end{eqnarray}
where we implicitly defined

\begin{equation}
\Pr(\mfAT i{\ts+1}|\lbA i,\aA i,\ifpiAT i{\ts+1})\defas\sum_{\mfAT i{\ts},\dsetAT i{\ts+1}}\Pr(\mfAT i{\ts+1}|\mfAT i{\ts},\dsetAT i{\ts+1},\aAT it,\ifpiAT i{\ts+1})\lbA i(\mfAT i{\ts},\dsetAT i{\ts+1}).
\end{equation}
(consistent with equation \ref{eq:P_fm__lb}).

\subsubsection{Expected Reward}

\label{app:IALM:exp_reward}
\begin{eqnarray*}
\RA i(\lbA i,\aAT it) & = & \E_{\sAugAT it\sim\lbA i,\sAugAT i{t+1}\sim\Aug T(\sAugAT it,\aAT it,\cdot)}\left[\Aug{R_{i}}(\sAugAT it,\aAT it,\sAugAT i{t+1})\right]\\
 & = & \sum_{\sAugAT it}\lbA i(\sAugAT it)\sum_{\sAugAT i{t+1}}\Aug T(\sAugAT i{t+1}|\sAugAT it,\aAT it)\Aug{R_{i}}(\sAugAT it,\aAT it,\sAugAT i{t+1})\\
 & = & \sum_{\mfAT i{\ts},\dsetAT i{\ts+1}}\lbA i(\mfAT i{\ts},\dsetAT i{\ts+1})\sum_{\mfAT i{\ts+1},\dsetAT i{\ts+2}}\Pr(\mfAT i{\ts+1},\dsetAT i{\ts+2}|\mfAT i{\ts},\dsetAT i{\ts+1},\aAT it,\ifpiAT i{\ts+1})\RA i(\mfAT i{\ts},\aAT it,\mfAT i{\ts+1})\\
 & = & \sum_{\mfAT i{\ts},\dsetAT i{\ts+1}}\lbA i(\mfAT i{\ts},\dsetAT i{\ts+1})\sum_{\mfAT i{\ts+1}}\Pr(\mfAT i{\ts+1}|\mfAT i{\ts},\dsetAT i{\ts+1},\aAT it,\ifpiAT i{\ts+1})\RA i(\mfAT i{\ts},\aAT it,\mfAT i{\ts+1})
\end{eqnarray*}

\begin{eqnarray}
 & = & \sum_{\mfAT i{\ts}}\sum_{\mfAT i{\ts+1}}\RA i(\mfAT i{\ts},\aAT it,\mfAT i{\ts+1})\left[\sum_{\dsetAT i{\ts+1}}\Pr(\mfAT i{\ts+1}|\mfAT i{\ts},\dsetAT i{\ts+1},\aAT it,\ifpiAT i{\ts+1})\lbA i(\mfAT i{\ts},\dsetAT i{\ts+1})\right]\nonumber \\
 & = & \sum_{\mfAT i{\ts}}\sum_{\mfAT i{\ts+1}}\RA i(\mfAT i{\ts},\aAT it,\mfAT i{\ts+1})\Pr(\mfAT i{\ts},\mfAT i{\ts+1}|\lbA i,\aAT it,\ifpiAT i{\ts+1})\label{eq:R(lfb,a)-1}
\end{eqnarray}
where we implicitly defined
\begin{equation}
\Pr(\mfAT i{\ts},\mfAT i{\ts+1}|\lbA i,\aAT it,\ifpiAT i{\ts+1})\defas\sum_{\dsetAT i{\ts+1}}\Pr(\mfAT i{\ts+1}|\mfAT i{\ts},\dsetAT i{\ts+1},\aAT it,\ifpiAT i{\ts+1})\lbA i(\mfAT i{\ts},\dsetAT i{\ts+1})
\end{equation}
(consistent with equation \ref{eq:P_xmxm_lfb}).

\newpage{}

\section{List of Acronyms}

\label{sec:List-of-Acronyms}

\begin{tabular}{>{\raggedright}p{0.25\columnwidth}>{\raggedright}p{0.73\columnwidth}}
\toprule 
Acronym & description\tabularnewline
\midrule
2DBN & 2-stage dynamic Bayesian network\tabularnewline
AOH & action-observation history\tabularnewline
CPT & conditional probability table\tabularnewline
DBN & dynamic Bayesian network\tabularnewline
Dec-MDP & decentralized Markov decision process\tabularnewline
Dec-POMDP & decentralized partially observable Markov decision process\tabularnewline
EDI-Dec-MDP & Dec-MDP with event-driven interactions\tabularnewline
fDec-POMDP & factored Dec-POMDP\tabularnewline
fPOSG & factored POSG\tabularnewline
GFBRMs & global-form best-response model\tabularnewline
IALM & Influence-augmented local model\tabularnewline
IBA & influence-based abstraction\tabularnewline
ISDs & intra-stage dependencies\tabularnewline
LFM & local-form model\tabularnewline
MDP & Markov decision process\tabularnewline
ND-POMDP & network-distributed POMDP\tabularnewline
NLAF & non-locally affected factor\tabularnewline
NMF & non-modeled factor\tabularnewline
OLAF & only-locally affected factor\tabularnewline
POMDP & partially observable Markov decision process\tabularnewline
POSG & partially observable stochastic game\tabularnewline
RL & reinforcement learning\tabularnewline
TD-POMDP & transition-decoupled POMDP\tabularnewline
TI-Dec-MDP & transition-independent Dec-MDP\tabularnewline
\bottomrule
\end{tabular}

\newpage{}

\section{List of Notation}

\label{sec:List-of-Notation}

\begin{longtable}[c]{>{\raggedright}p{0.14\columnwidth}>{\raggedright}p{0.78\columnwidth}}
\toprule 
symbol & description\tabularnewline
\hline
\endhead
 & \tabularnewline
General & \tabularnewline
$\E\left[\cdot\right]$ & expectation\tabularnewline
$\PrS(\cdot)$ & set of probability distributions over $\cdot$\tabularnewline
$\KroD{\cdot}{\cdot}$ & denotes the Kronecker delta function\tabularnewline
$\argsA{\left(\cdot\right)}{i}$ & a variable of interest $\left(\cdot\right)$ associated with agent
$i$\tabularnewline
$\argsA{\left(\cdot\right)}{\excl i}$ & a tuple of variables associated with all agents except $i$\tabularnewline
$\argsAT{\left(\cdot\right)}{i}{\ts}$ & a variable of interest $\left(\cdot\right)$ associated with agent
$i$ at time step $\ts$\tabularnewline
$\argsT{\left(\cdot\right)}{k:\ts}$ & partial history of values of $\left(\cdot\right)$ (e.g., $l_{tgt}^{k:t}$
is the history of target locations)\tabularnewline
$\vec{\left(\cdot\right)}^{\ts}$ & history of values of $\left(\cdot\right)$. I.e., $\vec{\left(\cdot\right)}^{\ts}=\argsT{\left(\cdot\right)}{0:\ts}$\tabularnewline
 & \tabularnewline
Models & \tabularnewline
$\mathcal{M}^{POSG}$ & A partially observable stochastic game (POSG)\tabularnewline
$\mathcal{M}^{LFM}$ & A local-form model: includes local state definitions for each agent\tabularnewline
$\mathcal{M}_{i}^{GFBR}$ & A global-form best-response model (GFBRM) for agent $i$\tabularnewline
$\mathcal{M}_{i}^{IALM}$ & An influence-augmented local model can be computed from an LFM when
fixing other policies: $\mathcal{M}_{i}^{IALM}(\mathcal{M}^{LFM},\jpolG{\excl i})$\tabularnewline
 & \tabularnewline
\multicolumn{2}{l}{Model components}\tabularnewline
$\agentS$ & the set of agents or (d)ecision makers\tabularnewline
$\sS$ & set of (global) states\tabularnewline
$\s$ & a global (i.e., Markov) state\tabularnewline
$\jaS$ & set of (joint) actions $\jaS=\aAS1\times\dots\aAS\nrA$\tabularnewline
$\ja$ & a (joint) action $\ja=\left\langle \aA1,\dots,\aA{\nrA}\right\rangle $\tabularnewline
$\Tfunc$ & transition function specifies $\Pr(\sT{\ts+1}|\sT{\ts},\jaT{\ts})$\tabularnewline
$\set{\REWF}$ & set of reward functions\tabularnewline
$\RA i$ & reward function of agent $i$\tabularnewline
$\joS$ & set of (joint) observations $\joS=\oAS1\times\dots\times\oAS\nrA$\tabularnewline
$\jo$ & a (joint) observation $\jo=\left\langle \oA1,\dots,\oA{\nrA}\right\rangle $\tabularnewline
$\Ofunc$ & observation function specifies $\Pr(\jo|\ja,\s')$\tabularnewline
$\gamma$ & the discount factor\tabularnewline
$\hor$ & horizon of the problem\tabularnewline
$\bO$ & initial state distribution: $\bO\in\PrS(\sS)$\tabularnewline
$\bar{\sAS i}$ & set of \emph{augmented }states. E.g., in a GFBRM\tabularnewline
$\bar{T}_{i},\bar{R}_{i}$,etc. & transitions, rewards, etc. over augmented states\tabularnewline
 & \tabularnewline
\multicolumn{2}{l}{histories and beliefs}\tabularnewline
$\aoHistAT i{\ts}$ & the action-observation history (AOH) of agent~$i$ at stage $\ts$\tabularnewline
$\aoHistATS i\ts$ & the set of AOHs of agent~$i$ at stage $\ts$\tabularnewline
$\bA{}$ & belief of a single POMDP agent $\bA{}(\s)\defas\Pr(\s|\bO,\aoHistAT{}{\ts})$\tabularnewline
$BU()$ & The belief update $\bA{}'=BU(\bA{},\aA{},\oA{})$\tabularnewline
$\gbA i$ & global-form belief of agent~$i$\tabularnewline
$\lbA i$ & local-form belief of agent~$i$\tabularnewline
$\polA i$ & policy of agent~$i$\tabularnewline
 & \tabularnewline
\multicolumn{2}{l}{Value functions}\tabularnewline
$\V^{{\ts}}$ & The optimal value function at stage $\ts$ with $\hor-\ts$ stages-to-go\tabularnewline
$\QT{\ts}$ & The optimal action-value function, or `Q-function'\tabularnewline
 & \tabularnewline
\multicolumn{2}{l}{Factored States}\tabularnewline
$\sfacS$ & the set of state factors $\sfacS=\left\{ \sfacI1,\dots,\sfacI\nrSF\right\} $
in a factored model, we have that $\sS=\sfacvIS1\times\dots\times\sfacvIS\nrSF$\tabularnewline
$\sfacI k$ & the $k$-th state factor\tabularnewline
$\sfacvIS k$ & the set of values $\sfacvI k\in\sfacvIS k$ that $\sfacI k$ can take
\tabularnewline
$\sfacvI k$ & a value of state factor $k$\tabularnewline
$\text{ORel}{}_{i}(\sfac)$ & observation relevant factor of agent~$i$\tabularnewline
$\text{RRel}{}_{i}(\sfac)$ & reward relevant factor of agent~$i$\tabularnewline
 & \tabularnewline
\multicolumn{2}{l}{Local states of agent $i$}\tabularnewline
$\sAS i$ & state space of agent~$i$ (general term: also outside LFMs)\tabularnewline
$\sA i$ & state for agent~$i$ (general term: also outside LFMs)\tabularnewline
$\LSF(i)$ & the local state function of an LFM for agent $i$: partitions $\sfacS$
into modeled state factors $\mfI k$ and non-modeled ones $\nmfI k$\tabularnewline
$\mfI k$ & $k$-th modeled factor\tabularnewline
$\mflI k$ & $k$-th  only-locally-affected factor (OLAF): a modeled factor that
is not an influence destination\tabularnewline
$\mfnI k$ & $k$-th a non-locally-affected factor (NLAF): a modeled factor that
is an influence destination\tabularnewline
$\slfmAS i$ & local state space (of \emph{modeled }factors) in an LFM\tabularnewline
$\mfA i$ & local state of agent~$i$ in an LFM\tabularnewline
$\nmfI k$ & a non-modeled factor\tabularnewline
$\nmfA i$ & instantiation of all non-modeled factors. I.e., $\s=\left\langle \mfA i,\nmfA i\right\rangle $\tabularnewline
$OLAF(i)$ & the set of OLAFs\tabularnewline
$NLAF(i)$ & the set of NLAFs\tabularnewline
 & \tabularnewline
\multicolumn{2}{l}{Influence notation}\tabularnewline
$\ifsAT i{\ts}$ & instantiation of all direct influence sources for stage $\ts$: $\ifsAT i{\ts}=\langle\nmfAT u{\ts-1},\jaGT u{\ts-1},\nmfAT u{\ts}\rangle$\tabularnewline
$\nmfAT u{\ts-1}$ & the (non-modeled) state factors that are direct influence sources\tabularnewline
$\nmfAT u{\ts}$ & the (non-modeled) state factors that are direct intra-stage influence
sources\tabularnewline
$\jaGT u{\ts-1}$ & the actions (of some subset of agents) that are direct influence sources\tabularnewline
$\aoHistGT u{\ts-1}$ & the AOHs of those other agents whose action is an influence source
(i.e., $\aoHistGT u{\ts-1}$ involves the same agents as $\jaGT u{\ts-1}$)\tabularnewline
$v$ & indirect sources: $\nmfAT v{\ts-1}$, $\jaGT u{\ts-1}$, $\nmfAT v{\ts}$
can effect the direct sources $\ifsAT i{\ts}$\tabularnewline
$w$ & union of direct and indirect sources: $w=u\cup v$; e.g., $\jpolG w$
is the joint policy of those agents whose action is either a direct
or an indirect influence source\tabularnewline
$\dsetAT i{\ts}$ & a d-separating set for agent~$i$'s influence at stage $t$\tabularnewline
$\dsetUF$ & the d-set update function: $\dsetAT i{\ts+2}=\dsetUF(\mfAT i{\ts},\aAT i{\ts},\mfAT i{\ts+1},\dsetAT i{\ts+1})$\tabularnewline
$\dsetCompF$ & the d-set compression function $\dsetCompF(\dsetAT i{\ts+1})$ that
computes a sufficient statistic for $\dsetAT i{\ts+1}$\tabularnewline
$\ifpiA{i}(\jpolG{\excl i})$ & $\ifpiA{i}(\jpolG{\excl i})=\left(\ifpiAT i1(\jpolG{\excl i}),\dots,\ifpiAT i{\h}(\jpolG{\excl i})\right)$
is an incoming influence point $\ifpiA{i}(\jpolG{\excl i})$\tabularnewline
$\ifpiAT i{\ts}(\jpolG{\excl i})$ & The incoming influence at stage~$\ts$: a conditional probability
distribution over values of the influence sources\tabularnewline
$\iffunc(\ifsAT i{\ts}|\dsetAT i{\ts})$ & shorthand for $\ifpiAT i{\ts}(\ifsAT i{\ts}|\dsetAT i{\ts},\bO,\jpolG{\excl i})=\ifpiAT i{\ts}(\jpolG{\excl i})(\ifsAT i{\ts}|\dsetAT i{\ts},\bO)$\tabularnewline
$p_{\ifpiAT i{\ts+1}}$ & influence-induced CPT that specifies $p_{\ifpiAT i{\ts+1}}(\mfnT{\ts+1}|\mfAT i{\ts},\dsetAT i{\ts+1},\aA i)$\tabularnewline
\bottomrule
\end{longtable}

~

\bibliographystyle{theapa}%UNCOMMENT

\bibliography{bib}

\end{document}